%% file: arxiv_draft.tex
\documentclass[nohyperref]{article}

\usepackage[accepted]{icml2022_arxiv}

\newif\ifmakefigs 
\makefigstrue    

\usepackage{titlesec}
\titlespacing\section{0pt}{0pt plus 2pt minus 0pt}{-2pt plus 2pt minus 0pt}
\titlespacing\subsection{0pt}{0pt plus 2pt minus 0pt}{-2pt plus 2pt minus 0pt}

\usepackage{preamble_general}
\usepackage{preamble_for_specific_document}
\usepackage{preamble_from_icml}

\renewcommand{\red}[1]{}
\renewcommand{\redx}{\color{red} \xmark \color{black}}

\begin{document}

\addtocontents{toc}{\protect\setcounter{tocdepth}{0}}

\twocolumn[
\icmltitle{Easy Variational Inference for Categorical Models via an Independent Binary Approximation}

\icmltitlerunning{Easy Variational Inference for Categorical Models via an Independent Binary Approximation}



\icmlsetsymbol{equal}{*}

\begin{icmlauthorlist}
\icmlauthor{Michael T. Wojnowicz}{tufts,disc}
\icmlauthor{Shuchin Aeron}{tufts,ece}
\icmlauthor{Eric L. Miller}{tufts,ece}
\icmlauthor{Michael C. Hughes}{tufts,cs}
\end{icmlauthorlist}

\icmlaffiliation{disc}{Data Intensive Studies Center}
\icmlaffiliation{ece}{Dept. of Electrical and Computer Engineering}
\icmlaffiliation{cs}{Dept. of Computer Science}
\icmlaffiliation{tufts}{Tufts University, Medford, MA, USA}
\icmlcorrespondingauthor{Email:}{michael.wojnowicz@tufts.edu}

\icmlkeywords{Machine Learning, ICML}

\vskip 0.3in
]

\printAffiliationsAndNotice{}

\setlength{\floatsep}{8pt plus 4pt minus 4pt}
\setlength{\textfloatsep}{8pt plus 4pt minus 3pt}

%
%
\setlength{\abovedisplayskip}{2pt plus 3pt}
\setlength{\belowdisplayskip}{2pt plus 3pt}

\begin{abstract}
We pursue tractable Bayesian analysis of generalized linear models (GLMs) for categorical data.
Thus far, GLMs are difficult to scale to more than a few dozen categories due to non-conjugacy or strong posterior dependencies when using conjugate auxiliary variable methods.
We define a new class of GLMs for categorical data called \textit{categorical-from-binary} (\CB) models. Each  \CB~model has a likelihood that is bounded by the product of binary likelihoods, suggesting a natural posterior approximation.  This approximation makes inference  straightforward and fast; using well-known auxiliary variables for probit or logistic regression, the product of binary models admits conjugate closed-form variational inference that is embarrassingly parallel across categories and invariant to category ordering. Moreover, an independent binary model simultaneously approximates   \textit{multiple} \CB~models. Bayesian model averaging over these can improve the quality of the approximation for any given dataset. We show that our approach scales to thousands of categories, outperforming posterior estimation competitors like Automatic Differentiation Variational Inference (ADVI) and No U-Turn Sampling (NUTS) in the time required to achieve fixed prediction quality.
\end{abstract}

\setlength{\tabcolsep}{0.1cm}
\begin{table*}[!t]
\caption{Assessment of categorical regression models in terms of the presence (\greencheck) or absence (\redx) of desirable features for fast, scalable Bayesian inference. \textit{Rows:} 
PGA refers to \pga \citep{polson2013bayesian}. SB-Softmax  refers to softmax regression with a stick-breaking link function \citep{linderman2015dependent}.
MNP+ACA stands for multinomial probit with \citet{albert1993bayesian} augmentation.
\IB~refers to our proposed independent binary regression (Sec.~\ref{sec:independent_binary_model}).
The first two rows are categorical models, the next three rows are categorical models with augmentation, and the last two rows are not categorical models, but in this paper we show how (and justify why) they can be used for approximate inference.
See Sec.~\ref{sec:extended_caption_for_feature_table} for an extended version of this table caption.}
\label{tab:comparison_of_Bayesian_categorical_models}
\begin{tabular}{l|ccccccc}
\textbf{Model} & \multicolumn{7}{c}{\textbf{Inference Feature}} \\
& Invariance to & Latent  &  Auxiliary & Closed-form & Conditional & Closed-form  & Embarassingly \\
 & category & linear & variable & likelihood &  conjugacy & variational & parallel across \\
 & ordering & regression &  independence &  &  & inference & categories \\
\toprule 
Softmax & \greencheck & \redx  & \greencheck & \greencheck & \redx & \redx & \redx \\
MNP & \greencheck & \redx  & \greencheck & \redx & \redx & \redx & \redx \\
\midrule 
Softmax+PGA  & \greencheck &  \greencheck & \greencheck & \greencheck & \greencheck & \redx  & \redx \\
SB-Softmax+PGA & \redx & \greencheck  & \greencheck &\greencheck & \greencheck & \greencheck & \redx \\
MNP+ACA  & \greencheck  & \greencheck & \redx & \redx & \greencheck & \greencheck & \redx \\
\midrule
\IB-Probit+ACA & \greencheck &   \greencheck &  \greencheck &  \greencheck &  \greencheck &  \greencheck & \greencheck \\
\IB-Logit+PGA & \greencheck &   \greencheck &  \greencheck &  \greencheck &  \greencheck &  \greencheck & \greencheck 
\end{tabular}
\end{table*}

\section{Introduction} \label{sec:intro}

We consider the problem of modeling categorical data informed by covariates using the machinery of generalized linear models (GLMs). 
Because our intended big data applications may involve rare events or little data for some quantities of interest, %
we take a Bayesian approach in order to estimate \emph{distributions} over unknown parameters given available data, and then average over these distributions when making predictions.
While many generalized linear models for categorical data have been proposed, Bayesian analysis of these models remains difficult with substantial active research due to the need for methods that are simultaneously accurate, tractable, and scalable.


The most common modeling choice for categorical data is  multi-class logistic regression, which uses a softmax (a.k.a. multi-logit) function to produce category probabilities.
The softmax likelihood is not conjugate to any standard prior over weight parameters (such as Gaussian), so estimating posteriors over weights requires expensive sampling methods \citep{hoffman2014no} or non-conjugate variational optimization methods \citep{wang2013variational,braun2010variational,kucukelbir2017automatic}.
Recent auxiliary variable methods~\citep{polson2013bayesian} have yielded conjugate conditionals amenable to Gibbs sampling, but closed-form variational updates for multiple categories require \emph{stick-breaking}~\citep{linderman2015dependent}.
Stick-breaking imposes an asymmetric order over categories, yet in many cases it is unnatural to view category selection as a sequential process.
In practice, this asymmetry complicates prior specification and inference quality~\citep{zhang2017permuted}. 

An alternative model is multi-class probit regression, whose link function is the cumulative distribution function of the Normal distribution.
The probit admits conjugate inference under a well-known auxiliary variable representation~\citep{albert1993bayesian,held2006bayesian}.
However, multi-class probit models encode strong posterior dependence among entries of the auxiliary parameter vectors. This dependence requires one-entry-at-a-time sampling instead of joint sampling~\citep{johndrow2013diagonal}, yielding poor mixing performance as the number of categories grows.
Furthermore, implementations often require picking a ``base category''; this choice can impact the practical results of inference \citep{burgette2021symmetric}.
Finally, the multinomial probit lacks closed-form category probabilities~\citep{johndrow2013diagonal},  which has prevented adoption within more complicated models~\citep{holsclaw2017bayesian}.

Motivated by difficulties that arise from these previous efforts (summarized in Table 1), we present a new class of categorical models -- \emph{categorical-from-binary} (CB) models\footnote{Code: \href{\codeURL}{\codeURLShort}} -- whose defining feature is that each one's likelihood can be lower-bounded by the likelihood of an independent binary model.
To perform approximate posterior estimation for such models, we fit the independent binary model via coordinate-ascent variational methods, taking advantage of well-known closed-form updates for binary logit or probit models.
This approach is scalable to thousands of categories, even more so because it is \emph{embarrassingly parallel} across categories, meaning we can fit a separate model for each category with no inter-worker communication overhead~\citep{fosterSecParallelAlgorithm1995}.
Even without parallelization, we demonstrate heldout predictions of comparable prediction quality to other categorical GLMs in far less time (see Fig.~\ref{fig:holdout_perf_over_time}), with competitive likelihoods only slightly below the expensive gold standards.
Our accurate predictions are possible via a \emph{Bayesian model average} (BMA) over members of our CB model class which we can deploy cheaply using only one posterior fit for the surrogate model.
Our experiments reveal that our proposed methods offer a promising first-line approach for fast Bayesian analysis of big categorical data, especially when the number of categories is large.




\subsection{Problem formulation} \label{sec:problem_formulation}

\begin{subequations}
Consider a given training set of $N$ 
paired observations, $\{(\+x_i, y_i)\}_{i=1}^N$, where each observation (indexed by $i$) consists of $\+x_i \in \R^M$, a (fixed) vector of covariates, and integer $y_i \in \set{1,...,K}$, indicating which of the $K$ categories $i$ belongs to. 
We treat $y_i$ as a random variable generated as:
\begin{align}
y_i &\sim \text{Cat}(\+s_i), \quad  \+s_i = (s_{i1}, \ldots s_{iK})^T \in \Delta_{K-1}, \\
\+s_i &= f(\+\eta_i), \quad \+\eta_i= \+B^T \+x_i.
\end{align}
Here, $\+B \in \R^{M \times K}$ are unobserved regression weights, whose matrix-vector product with covariates $\+x_i$ yields the so-called linear predictor $\+\eta_i$.
The function $f : \R^K \to \Delta_{K-1}$ maps the real-valued vector $\+\eta_i$ to a vector $\+s_i$ of $K$ non-negative values that sum to one.    We refer to this model as a \textit{categorical regression} or a generalized linear model (GLM) for categorical data.  Note that $f$ need not be invertible, so the model need not be identifiable (i.e., there may exist $\+B_1 \neq \+B_2$ which yield identical distributions over $y_i$). 
\label{eqn:categorical_regression}
\end{subequations}

We wish to pursue Bayesian inference, treating the parameter $\+B$ as a random variable with prior $\pi(\+B)$.
We use a Gaussian prior in practice.
Our primary interest is the \emph{prediction task}: given $N$ training pairs $(\+x_i, y_i)$ and a new covariate vector $\+x_*$, we wish to make probabilistic prediction of the new category label $y_*$ via the posterior predictive $p( y_* | \{y_i\}_{i=1}^N ) = \int p( y_* | \+B) p( \+B | \{y_i\}_{i=1}^N ) d \+B$. 
This prediction requires the completion of a \emph{posterior estimation} task: given a fixed training set of size $N$, estimate $p( \+B | \{y_i \}_{i=1}^N)$.
To keep the formal statements of both tasks simple, we treat covariates as fixed knowns and suppress conditioning on $\+x_i$ in notation.
We stress that our focus is on the posterior predictive, as category outcomes $y_*$ are relevant to applications while the weights $\+B$ are intermediate quantities whose non-identifiability can make assessment challenging; large differences in parameter space may not imply notable changes in prediction quality.

 
\textbf{Contributions.}
Our contribution is to define a class of categorical models (choices of the function $f$) that we call \emph{categorical-from-binary} models.
Using this class, we show that a well-justified approximation is possible such that posterior estimation enjoys all the beneficial properties in Table~\ref{tab:comparison_of_Bayesian_categorical_models}.
To our knowledge, out of all previous models listed in Table~\ref{tab:comparison_of_Bayesian_categorical_models}, only our approach yields tractable category probabilities, provides scalable yet closed-form variational optimization, is invariant to category order, and can be integrated into more complex graphical models.
We further provide a prediction method that averages across models to obtain accurate categorical predictions from our approximate posterior.

\section{Models}

\subsection{Overview} \label{sec:models_overview}
Bayesian inference for \textit{categorical} regressions (Eq.~(\ref{eqn:categorical_regression})) is difficult, in the sense that no current approach has all the features given in Table~\ref{tab:comparison_of_Bayesian_categorical_models}. In contrast, Bayesian inference for \textit{binary} regression is far more straightforward.  Using \pga \citep{polson2013bayesian} for  logistic regression, or the Normal augmentation of  \citet{albert1993bayesian} for probit regression, one may obtain conditionally conjugate models and closed-form variational inference \citep{durante2019conditionally, consonni2007mean,armagan2011note,fasano2019scalable}.
These properties extend to regression models for $K$-bit binary vectors that treat each bit independently (see Sec.~\ref{sec:independent_binary_model}).  In fact, if each bit were to indicate the presence of a particular category, then every desirable inferential feature listed in Table~\ref{tab:comparison_of_Bayesian_categorical_models} would be present. We would like to exploit this collection of features for \emph{categorical} modeling.  The problem is that independent binary regression allows for multiple non-zero bits while categorical models require exactly one non-zero bit.
 To address this problem, we construct categorical models around independent binary models (Sec.~\ref{sec:categorical_from_binary_models}), which enables the efficient posterior estimation (Sec.~\ref{sec:posterior_estimation}).


\subsection{A model for independent binary vectors } \label{sec:independent_binary_model}

Consider a general univariate binary regression likelihood \citep{albert1993bayesian} of the form
\begin{align*}
\BINARYy_i \cond \BINARYvecbeta &\indsim \Bernoulli \big( H(\BINARYeta_i) \big),  \quad i=1,...,n 
\labelit \label{eqn:general_binary_regression}
\end{align*}
where $\BINARYy_i \in \{0, 1\}$ are binary response random variables, the linear predictor $\BINARYeta_i = \+x_i^T \BINARYvecbeta$ is formed from known covariates $\+x_i \in \R^M$ and unknown parameters $\BINARYvecbeta \in \R^M$, and $H$ is an arbitrary cdf that is referred to as an inverse link function. Logistic regression sets $H$ to be the standard logistic cdf and probit regression sets $H$ to the standard Gaussian cdf.
We use the breve notation to distinguish random variables here from those in later categorical models.
For additional intuition about Eq.~\eqref{eqn:general_binary_regression}, see Sec.~\ref{sec:additional_interpretation_for_the_general_binary_regression_model}.   


Now let us consider modeling binary \textit{vectors}: $\IBvecy_i = (\IBy_{i1}, ..., \IBy_{iK}) \in \set{0,1}^K, i=1,2, \hdots, N$. Crucially, each $\IBvecy_i$ is a K-bit vector, and \emph{not} a one-hot vector: any number of entries could be 1 or 0.   Taking the product of binary regression likelihoods of the form in Eq.~\eqref{eqn:general_binary_regression}, we obtain a model which we call \emph{independent binary} (\IB) \emph{regression},
\begin{align*}
\IBy_{ik} \cond \IBvecbeta_k &\indsim \Bernoulli \big( H(\IBeta_{ik} ) \big)
\labelit \label{eqn:independent_binary_model}
\end{align*}
independently across each $k=1,2,\hdots,K$, with each linear predictor $\IBeta_{ik} = \+x_i^T \IBvecbeta_k$ formed from  known covariates $\+x_i \in \R^M$ and unknown parameters $\IBvecbeta_k \in \R^M$. 
 The likelihood for an observation under a K-bit \IB~model is
\begin{align}
p_\IB(\IBvecy_i  \cond \IBmatrixB ) {=}  \ds\prod_{k=1}^K H( \IBeta_{ik} )^{\IBy_{ik}} \; \big( 1-H(\IBeta_{ik}) \big)^{1-\IBy_{ik}},
\label{eqn:IB_likelihood}
\end{align}
where $\IBmatrixB = (\wh{\+\beta}_1, ..., \wh{\+\beta}_K) \in \R^{M \times K}$ is a matrix of weights for each combination of covariate and category.  
\IB~is a \textit{class} of models, each member defined by a chosen $H$.
When $H$ is the standard logistic or standard Gaussian cdf, we respectively obtain the \IB-Logit or \IB-Probit models.   

\subsection{Categorical-from-binary models} \label{sec:categorical_from_binary_models}


Suppose now that we are interested in regression models for categorical (one-of-K) data, i.e. 
$y_i \sim \text{Cat}( s_{i1}, \ldots s_{iK} )$ where $y_i \in \set{1,...,K}$ and $\+s_i \in \Delta_{K-1}$.  We restrict our focus to categorical models which are related to  \IB~models (Eq.~\eqref{eqn:independent_binary_model}) in the following manner:

\begin{definition}
A \textbf{categorical-from-binary} (\CB) model is a GLM for categorical data $y_i \in \set{1,..,K}$ which always assigns a higher likelihood to a category $k$ than an IB model does to the corresponding one-hot vector. \CB~models obey the likelihood bound
\begin{align*}
p_\CB(y_i \cond \+B) >  p_\IB(\IBvecy_i = \+e_{y_i} \cond \IBmatrixB = \+B)
\labelit \label{eqn:CFB_likelihood_larger_than_IB}
\end{align*}
for all observations $y_i \in \set{1,...,K}$, covariates $\+x_i \in \R^M$, and weights $\+B \in \R^{M \times K}$, and where $\+e_{y_i}$ is the one-hot indicator vector with value of 1 only at entry $y_i$.
\label{def:categorical_from_binary_model}
\end{definition}




To construct a \CB~model, we must construct a function $f$ for the relation $p_\CB(y_i \cond \+B) = f(\+\eta_i$),  where $\+\eta_i = \+B^T \+x_i$, such that the bound in Eq.~\eqref{eqn:CFB_likelihood_larger_than_IB} is satisfied.   We begin by choosing a cdf $H$ (e.g. standard Gaussian or standard Logistic) to specify a concrete \IB~model (e.g. \IB-Probit or \IB-Logit). 
We refer to the chosen \IB-model as the \textbf{\base} of a \CB~model.     We then define $h : \R^K \to \R^K$ such that $h(\+\eta_i) = \big(H(\eta_{i1}), \hdots, H(\eta_{iK})\big)^T$.
\CB~models construct $f$ via the composition $f=g \circ h$ for some function $g$ defined below.
This composition means that \CB~category probabilities are determined by the vector output of $h$, whose entries define the probabilities of ``success" at each of the $K$ bits of the \IB~model:
$H(\eta_{ik})=p_\IB(\widehat{y}_{ik} = 1 \cond \+\beta_k)$ for all $k$.
\red{CONFIRM: The term "base" is used enough to warrant the definition.}



\red{TODO: confirm the transpose in $h$.}


\red{TODO: Look for corresponding model and remove or define it.}

\subsection{Concrete categorical-from-binary likelihoods} \label{sec:two_classes_of_CB_models}

After selecting a specific cdf $H$, fully specifying a concrete \CB~model for categorical data requires identifying the transformation $g$ which maps the \IB~probabilities of success $h(\+\eta_i)$ into the simplex $\Delta_{K-1}$ in a way that satisfies the bound in Eq.~\eqref{eqn:CFB_likelihood_larger_than_IB}.  We now provide two such specifications.
First, the \emph{marginalization} construction assumes the probability of the $k$th category is proportional to the probability of success in the $k$-th bit of the \IB~model.
Second, the \emph{conditioning} construction requires the \IB~model to assign non-zero probability only to valid one-hot vectors, and then assumes that the probability of the $k$th category equals the probability of its one-hot representation.

\textbf{\CB~construction via marginalization.} A \emph{categorical-from-binary-via-marginalization} (\CBM) model produces category probabilities by normalizing the marginal probabilities of success $\set{H(\eta_{ik})}_{k=1}^K$ from an \IB~model:
\begin{align}
p_\CBM(y_i = k \cond \+B) =  \frac{H( \eta_{ik})}{ \sum_{\ell=1}^K H( \eta_{i\ell})}.
	\labelit \label{eqn:CBM_category_probabilities}
\end{align}
for all categories $k \in \set{1, \ldots, K}$.

\textbf{\CB~construction via conditioning.} A \emph{categorical-from-binary-via-conditioning} (\CBC) model produces category probabilities from an IB model by conditioning on the event that the vector has exactly one positive entry:
\begin{align}
p_\CBC(y_i {=} k \cond \+B) = 
\df{H( \eta_{ik}) \ds\prod_{j \neq k} (1-H(\eta_{ij}))}{\ds\sum_{\ell=1}^K  H(\eta_{i\ell}) \ds\prod_{j \neq \ell} (1-H(\eta_{ij}))    }
\label{eqn:CBC_category_probabilities}
\end{align} 
for categories $k=1,..., K$.
A \CBC~model is an \IB~model truncated to the space of one-hot vectors. A \CBC~model can also be expressed as a \textit{normalized odds} model (see Sec. \ref{sec:CBC_as_normalized_odds}). 


\begin{proposition}
\CBC~and \CBM~models are categorical-from-binary (\CB)~models satisfying Definition~\ref{def:categorical_from_binary_model}. \red{Should this be: ``\CBC~and \CBM~models are categorical-from-binary (\CB)~models, and the choice of $H$ determines a corresponding \IB~model."?}
\label{prop:CBT_and_CBM_are_CB_models}
\end{proposition}

\begin{proof}
Deferred to Appendix Sec.~\ref{sec:CBT_and_CBM_are_CB_models_appendix} to save space.
\end{proof}

\begin{example} To illustrate the strategies taken by the \CBM~and \CBC~models in forming a categorical regression from an \IB~regression, contrast how they assign probability to the first of $K=3$ categories.  
{\tiny
\begin{align*}
p_\CBM (y = 1 \cond \+B)
& \propto  p_\IB \bigg(\IBvecy \in \big\{ (1,0,0), (1,1,0), (1,0,1), 
(1,1,1)  \big\} \cond \+B \bigg) \\
p_\CBC (y = 1 \cond \+B) 
&= p_\IB \bigg(\IBvecy = (1,0,0) \, \bigg| \, \IBvecy \in \big\{ (1,0,0), (0,1,0), (0,0,1)  \big\}, \+B \bigg) 
\end{align*}
}The distinction can be understood through a voting metaphor. The conditioning strategy of \CBC~models forces the $K$ independent binary models to agree on a single ``vote", the marginalization strategy of \CBM~models allows multiple votes across the $K$ independent bits.
\end{example}

\textbf{From model classes to models.} Both \CBC~and \CBM~are   model \textit{classes}, generating different models as $H$ varies.  For instance, by taking $H$ to be the standard Gaussian cdf  we can generate the \CBC-Probit and \CBM-Probit models (for which \IB-Probit is the base), and by taking $H$ to be the standard Logistic cdf, we can generate the \CBC-Logit and  \CBM-Logit models (for which \IB-Logit is the base).  



\input{related_work}

\section{Posterior estimation}\label{sec:posterior_estimation}



We now develop our methodology for approximating a \CB~model's posterior over weights, $p(\+B | \{y_i \}_{i=1}^N)$.  The key insight is this:  we can provably optimize a lower bound on the likelihood of a \CB~model by instead performing traditional variational inference for the \IB~model.  
First, we establish that after integrating away the weights $\+B$, the marginal likelihood of a \CB~model is lower-bounded by the marginal likelihood of an \IB~model.
Second, we argue that a mean-field variational posterior $q(\+B)$ estimated to approximate the \IB~model is also a suitable approximation for a \CB~model.
This suggests a straightforward coordinate ascent variational inference algorithm (IB-CAVI), which uses the efficient conjugate updates for logit or probit \textit{binary} models discussed in 
Sec.~\ref{sec:models_overview}.

\red{TODO: Content has moved all around; reabsorb all of this. In this section, we show how to use an \IB~approximation within the variational framework to produce  approximate Bayesian inference for categorical GLMs that has closed-form optimization updates.  The key insight is that we can provably optimize a lower bound on a \CB~model by instead maximizing the traditional ELBO of the independent binary model.  The optimization problem is not specific to the choice of underlying \CB~model, so a single optimization run is sufficient for approximate inference on \textit{multiple} \CB~models. In Sec.~\ref{sec:determining_a_good_target_for_an_IB_approximation}, we show how to exploit this feature by applying Bayesian model averaging over approximation targets. This improves approximation quality with no additional training cost. }

\subsection{Marginal likelihood bounds} 
\begin{subequations}
For any dataset $\+y = \{ y_i \}_{i=1}^N$, any \CB~likelihood and any choice of prior with density $\pi(\+B)$, we obtain 
 \begin{align}
 p_\CB(\+y) &= \ds\int 
 	\pi(\+B) \,
	\prod_{i=1}^N p_\CB(y_i \cond \+B)
	\wrt{\+B}
	\label{eqn:marginal_likelihood_of_CB_model}\\
 &>	\ds\int 
 	\pi(\+B) \,
	\prod_{i=1}^N 
	p_\IB( \IBvecy_i = \+e_{y_i}  \cond \+B)
	\wrt{\+B} 
	\label{eqn:marginal_likelihood_of_IB_model} \\
 & = p_\IB(\IBmatrixY = \+E(\+y ) ),
 \end{align}
which follows from the likelihood bound relating CB to IB (Eq.~\eqref{eqn:CFB_likelihood_larger_than_IB}) and monotonicity of the integral. 
Here, $\IBmatrixY = (\IBvecy_i)_{i=1}^N$, and $\+E(\+y)$ represents a one-hot representation of the categorical training data $\+y$, stacking all one-hot vectors $\{\+e_{y_i} \}_{i=1}^N$.
\label{eqn:IB_lower_bounds_the_marginal_likelihood_of_a_corresponding_CB_model}
\end{subequations}



\red{TODO: In the VI section, Eric had said that the following should be made clearer much earlier: "we can provably optimize a lower bound on a truly categorical model by instead maximizing the traditional ELBO of the independent binary model."  But here I show the same thing with MCMC.  So can I state a broader claim that handles both VI aand MCMC, and lift it up to higher in the paper (like the intro and summary?)}


\subsection{Variational surrogate bounds}\label{sec:variational_surrogate_bounds}

Variational inference~\citep{wainwright2008graphical, blei2017variational} deterministically approximates a posterior distribution by finding the member $Q \in \Q$  of a tractable family of distributions which maximizes a lower bound on the   logarithm of the \textit{evidence} (the marginal likelihood of the data).  This lower bound is known as the evidence lower bound or ``ELBO''.


For categorical-from-binary (\CB) models, the evidence of interest is $p_{\CB}( \+y ) = \int  p_{\CB}( \+y | \+B)\, \pi( \+B)\,d\+B$, where $\pi$ denotes the prior density on $\+B$. This quantity is intractable for both \CBC~and \CBM~models (as defined in Sec.~\ref{sec:two_classes_of_CB_models}) because they lack a conjugate prior.
For instance, a Gaussian prior is not conjugate, since the logarithm of the joint density $p_c (y_i \cond \+B) \pi(\+B)$  does not yield a quadratic in $\+B$, where $c$ is in $\set{\CBC,\CBM}$ and $\pi$ is a Gaussian density. 

\red{TODO: What I said is true if $\+\beta$ is a vector, is that still technically correct if $\+B$ is a matrix ?}
 
Bayesian inference for a model with an intractable marginal likelihood can sometimes be provided through a conventional variational approach. If we select $Q$ as any distribution over $\+B \in \R^{M \times K}$, and let $q(\cdot)$ be the density of $Q$, then the traditional lower bound of the log of the evidence, which we denote $\ELBO_\CB \leq \log p_{\CB}( \+y)$, follows from Jensen's inequality: 
\begin{align*}
\ELBO_\CB(q)  = 
 \E_q [\log p_\CB(\+y \cond \+B)] - D_\text{KL}(q(\+B) \parallel \pi(\+B)),
\labelit \label{eqn:ELBO_for_CB}	
\end{align*}
where $D_\text{KL}$ is the Kullback-Leibler divergence. 
Unfortunately, this lower bound is still intractable to compute. The energy term $\E_q [\log p_\CB(\+y \cond \+B)]$ contains $N$ expectations that lack closed-form expression (expected log sums of $K$ nonlinear quantities), due to the normalizing constants of the categorical models (Eqs.~\eqref{eqn:CBM_category_probabilities} and \eqref{eqn:CBT_category_probabilities_as_odds}).
While Monte Carlo approximations to this integral are possible that can enable gradient-based learning of $q(\+B)$, given a fixed computational budget the quality of these approximations becomes increasingly suspect in high dimensions, such as when the number of categories grows.

\red{TODO: Work on this; does this show up anywhere in inference when we deal specifically with probit or logit models?} \red{TODO: If the regression weights are independent across categories, though, monte carlo sampling should be ok?!} \red{TODO:  Consider whether I want to be more specific, as per this outdated description from the workshop paper: ``due to the SDO normalizing constant ; we must handle $N$ terms of the form $\E_q [\log C_{i,\SDOy_i^{SDO}}(\SDObeta)]$.}

Instead, we define a surrogate objective $\L_{\CB}(q)$ that lower bounds the log marginal likelihood for any \CB~model:
\begin{subequations}
{\footnotesize 
	\begin{align}
	\log p_\CB(\+y) & \stackrel{\eqref{eqn:IB_lower_bounds_the_marginal_likelihood_of_a_corresponding_CB_model}}{>} 
	\log p_\IB(\IBmatrixY = \+E(\+y)) \label{eqn:marginal_IB_log_likelihood_is_a_lower_bound_on_CB_marginal_log_likelihood} \\
	 & \geq \ELBO_\IB(q; \IBmatrixY = \+E(\+y))  := \L_{\CB}(q).  \label{eqn:IB_ELBO_as_surrogate_lower_bound_for_CB}
	\end{align}
\label{eqn:surrogate_objective_for_CB_models} 
}
\end{subequations}
We call this a \textit{surrogate lower bound} because there are two bounds at work: the bound relating \CB~to \IB~in Eq.~\eqref{eqn:marginal_IB_log_likelihood_is_a_lower_bound_on_CB_marginal_log_likelihood} and the traditional ELBO (via Jensen's inequality) in Eq.~\eqref{eqn:IB_ELBO_as_surrogate_lower_bound_for_CB}.  (Recall from Eq.~\eqref{eqn:IB_lower_bounds_the_marginal_likelihood_of_a_corresponding_CB_model} that the the former bound requires that the \CB~model and its  \IB~base have the \emph{same} prior density $\pi$ over weights.)  
This yields a surrogate objective $\L_{\CB}(q)$ which is exactly  the traditional ELBO applied to the \IB~model.   
As justified by this surrogate, we can solve our Bayesian inference problem for \textit{categorical} regression by applying well-known variational inference scheme for \textit{binary} regression on a one-hot transformation of the categorical data.

\red{TODO: absorb Via the surrogate bound relation of Eq.~\eqref{eqn:surrogate_objective_for_CB_models}, we can provably optimize a lower bound on a \CB~model by instead maximizing the traditional ELBO of its corresponding \IB~model. This argument holds for any selected approximate posterior family $\mathcal{Q}$ over $\+B$.}


   \red{TODO: Absorb: The needed expressions follow immediately from the analogous expressions for binary logit or probit regression under the appropriate data augmentation (recall Sec.~\ref{sec:models_overview}) and independence.  We provide more information in Sec.~\ref{sec:procedure_for_posterior_estimation}. }



\subsection{Procedure for posterior estimation}
\label{sec:procedure_for_posterior_estimation}

\red{CONFIRM:  We no longer need to mention the strategy of introducing whatever aux vars are necessary to get conditional conjugacy for the IB models.}

We now outline scalable procedures for closed-form coordinate ascent variational inference (CAVI) that will estimate an optimal approximate posterior $q^*(\+B)$ under the surrogate objective $\ELBO_\IB$. Especially in high-dimensional settings, these procedures are far more scalable than the difficult task of directly optimizing the truly-categorical model (via the objective $\ELBO_\CB$).



Closed-form CAVI procedures for univariate binary regression (Eq. \eqref{eqn:general_binary_regression}) are well-known when the prior on weights $\BINARYvecbeta$ is Gaussian and the function $H$ corresponds to either logit or probit regression, as reviewed in Sec.~\ref{sec:models_overview}.
Both involve augmenting observed binary data $\BINARYvecy \in \set{0,1}^N$  with auxiliary variables $\BINARYvecz \in \R^N$ such that the augmented model is conditionally conjugate while the original model is preserved through marginalization.   In particular, we can obtain a conditionally conjugate model by either constructing $\BINARYvecz$ via truncated normal augmentation for probit regression \citep{albert1993bayesian} or via \pga for logistic regression \citep{polson2013bayesian}.
Extrapolating from this univariate binary model to the $K$ independent bits of the \IB~model (Eq.~\eqref{eqn:IB_likelihood}), we immediately obtain conditional conjugacy by augmentation with a matrix $\IBmatrixZ \in \R^{N \times K}$ whose $k$th column $\IBvecz_k$ uses the relevant univariate augmentation strategy.

Our variational approximation of the \emph{augmented} $K$-bit IB model assumes a mean-field factorization: $q(\IBmatrixZ,  \IBmatrixB) = q(\IBmatrixZ) q(\IBmatrixB)$.
Under this choice, deriving a coordinate ascent algorithm to find the $q$ that optimizes $\L_{\CB}(q)$ in Eq.~\eqref{eqn:surrogate_objective_for_CB_models} follows the standard variational recipe~\citep{blei2017variational}.
First, because both prior and likelihood are independent across categories $k$, our mean-field posterior also simplifies as independent across categories without further approximation: $q(\IBmatrixZ,  \IBmatrixB) = \prod_{k=1}^K q(\IBvecz_k) q(\IBvecbeta_k)$.  From there, one exploits known applications of CAVI to univariate binary models specific to the chosen link function, either logit~\citep{durante2019conditionally} or probit~\citep{consonni2007mean,armagan2011note,fasano2019scalable}. 
Procedurally, from a suitable initial value of $q$, each factor of $q$ is updated to maximize $\ELBO_\IB$ while holding others fixed, using closed-form updates arising from conditional conjugacy. 
For concrete realizations of the required updates for a $K$-bit \IB~model, see Sec.~\ref{sec:VI_for_CB_probit_models} for the probit link and Sec.~\ref{sec:VI_for_CB_logit_models} for the logit.  
 Since this posterior estimation procedure operates by doing CAVI for a surrogate \IB~model, we call it \textit{IB-CAVI (independent binary coordinate ascent variational inference)}. 

After iterating updates to each factor until convergence, the resulting variational density $q^*$ over $\IBmatrixB$ is a local maximum of the $\ELBO_\IB$ \citep{ormerod2010explaining}.  By Eq.~\eqref{eqn:surrogate_objective_for_CB_models}, $q^*(\IBmatrixB)$ is therefore a local maximum of a surrogate bound on the \CB~model, and thus we can treat $q^*(\IBmatrixB)$ as an approximation to the ideal (intractable) posterior $p_{\CB}(\+B | \+y)$ of the categorical model. 

Our IB-CAVI procedure is not specific to a particular \CB~model. One optimization run can produce a posterior $q^*$ suitable for \textit{multiple} \CB~target models, as long as the \IB~model is a base. 
For example, performing CAVI for the \IB-Probit model provides a $q^*$ suitable for \textit{any} \CB-Probit  model (\CBM-Probit, \CBC-Probit, etc.).  


\textbf{Runtime cost of IB-CAVI.}
The per-iteration runtime cost for logit models is  $O(M^3K + N M^2K )$ (see Alg.~\ref{alg:ib_cavi_for_cb_logit} and  Sec.~\ref{sec:complexity_of_IB_Logit}), where $K$ is the number of categories, $M$ is the number of features, and $N$ is the number of training examples.  For probit models, the per-iteration runtime drops to $O(NMK)$, with further reductions under sparsity (Alg.~\ref{alg:ib_cavi_for_cb_probit} and Sec. \ref{sec:computational_complexity_for_IB_probit}).
When the Gaussian prior $\pi(\+B)$ is chosen to be independent across category-specific weights, under either link function our IB-CAVI approach is \emph{embarrassingly parallel} across categories. This makes our IB-CAVI approach particularly suitable for data with hundreds or thousands of categories. For example, to fit the \IB-Probit in parallel, each worker solves a single category's binary regression problem to convergence at cost $O(NM)$ per iteration.


\section{Prediction via Bayesian Model Averaging}\label{sec:determining_a_good_target_for_an_IB_approximation}

Given a posterior over weights $q^*(\+B)$ via the IB-CAVI procedure from Sec.~\ref{sec:posterior_estimation}, how can we make useful \emph{predictions} of the category labels $y_* \in \set{1 \ldots K}$ for new observations with covariates $\+x_*$?
Clearly we must employ a truly-categorical \CB~likelihood to obtain valid predictions, as the \IB~likelihood can produce any $K$-bit vector, not just a $1$-of-$K$ choice. However, empirical investigations in Sec.~\ref{sec:evaluating_CB_models_as_targets_of_an_IB_approximation_wrt_soft_predictions}  (see esp. Fig.~\ref{fig:evaluating_CBM_and_CBT_as_targets_of_an_IB_approximation}), with further results in Fig.~\ref{fig:BMA_results_logit_ordered}, suggest that there are substantial dataset-specific tradeoffs in approximation quality (\IB~can approximate \CBC~better than \CBM~on some data, and vice versa on other data) and goodness-of-fit.   Needing to select a specific \CB~likelihood (\CBC~or \CBM) in advance for a dataset would be challenging. 


To avoid the problem of selecting a specific \CB~likelihood, we take advantage of the fact that our IB-CAVI procedure produces a posterior $q^*$ suitable for \textit{multiple} \CB~likelihoods. 
Thus, to make predictions we perform a \emph{Bayesian model average} over all applicable \CB~likelihoods.
We find this significantly improves prediction quality at \emph{no additional training cost}.

We model the problem with two random variables.
First, let $c \in \set{\CBM, \CBC}$ indicate the selected model, with given prior probabilities $p(c) = \pi_c \in (0, 1)$ such that $\sum_c \pi_c = 1$.
We recommend a uniform setting: $\pi_{\CBM} = \pi_{\CBC} = 0.5$.\footnote{In case of an intercepts-only model, one might consider setting the prior weight on \CBM~to 1.0; see Prop.~\ref{prop:IB_gives_an_identifiability_restriction_for_CBM_but_not_CBT_in_the_intercepts_only_setting}.} 
Second, we have the predicted quantity of interest $\Delta$ (such as future category label $y_*$), for which $p_c( \Delta | \+B)$ is known. 
Following ~\citet{madigan1996bayesian}, our BMA prediction procedure forms the posterior predictive for $\Delta$ given training data $\+y$ via the sum rule,
\begin{align}
	p( \Delta | \+y ) 
	&= \sum_{c} p( \Delta | c, \+y) \underbrace{p( c | \+y)}_{w_c}.
\label{eqn:bayesian_model_average}
\end{align}
The first term can be approximated as:
\begin{align}
	p( \Delta | c, \+y) = \int p_c( \Delta | \+B) p( \+B | \+y ) \wrt{\+B}
	\approx \frac{1}{T}\sum_{t=1}^T p_c( \Delta | \+B^t)
	\notag 
\end{align}
using $T$ samples from our approximate posterior $\+B^t \sim q$.

The second term $w_c = p( c | \+y)$ defines the posterior probability of choosing model $c$, which via Bayes rule is
\begin{align}
	w_c &= \frac{p_c(\+y)\pi_c}{p_{\CBC}(\+y)\pi_{\CBC} + p_{\CBM}(\+y)\pi_{\CBM}}.
\end{align}
Each evidence term $p_c(\+y)$ is defined in Eq.~\eqref{eqn:marginal_likelihood_of_CB_model}. While this term cannot be computed directly due to an intractable integral, we can estimate it via a Monte Carlo (MC) approximation of the conventional evidence lower bound defined in Eq.~\eqref{eqn:ELBO_for_CB}. For each model $c$, we estimate the evidence as:
\begin{align}
\log p_c( \+y ) \approx 
\textstyle \frac{1}{S} \sum_{s=1}^S \log p_c( \+y | \+B^s ) + D_\text{KL}(q(\+B) \parallel \pi(\+B)), \notag 
\end{align}
where the first term is model $c$'s likelihood averaged over $S$ MC samples from our approximate posterior $\+B^s \sim q$, and the second KL term has a closed-form solution because our prior $\pi(\+B)$ and chosen variational factor $q(\+B)$ are exponential family distributions \citep{nielsen2010entropies}.
Alternatively, we could use recent importance-sampling bounds~\citep{burda2015importance} to get even more accurate approximations at the cost of increased computation.

\textbf{Runtime cost of BMA.}
The posterior weights $w_c \in [0,1]$ for each model type $c$ can be computed once for each training set $\+y$, at a linear cost in the number of examples $N$ and categories $K$, and then stored in memory for all future uses.
Then, using pre-computed weights, the cost of computing the probability of each new example's category $y_*$ given $\+x_*$ has a linear cost in $K$ and the number of samples $S$.

\red{TODO: Absorb -- So our approximate inference strategy has multiple approximation targets, but we can apply \BMA~to improve the quality of downstream inference without \textit{any} additional model training.}

\red{THINK/DISCUSS: This strategy is a bit worrisome - if we can sample from the conventional ELBO for this purpose, why can't we just sample all along? (After all, that's what ADVI does, not to mention Linderman in his inference approaches, right?)  We only need to do it $I$ times rather than once, where $I$ is the number of CAVI iterations.  Perhaps, however, it's a big deal that we are doing the sampling at the \textit{end} of approximate inference; we might get high variance estimates early on, and that could potentially really slow down learning. }

\red{TODO: Check into whether the claimed parallelism across categories only holds because we assume the betas are correlated across covariates $m$ for a given category $k$, but are independent across categories. Presumably we could relax that restriction to get a more expressive model, but at an inference cost.}

\red{TODO: Does the parallelism across categories requires us to assume that the betas are not correlated across categories? (I think it might not; because the IB model assumes this, so if the prior does, so will the optimal variational family.)}

\section{Experiments}
\red{TODO: add overview which provides some ``narrative" as to the takeaway from each experiment.  E.g., the intrusion detection experiment demonstrates scalability (many categories).}

We now assess the speed and quality of our proposed \CB~models with \IB~posterior approximations.  Reproducible details for all experiments are in 
Sec.~\ref{sec:supplemental_info_for_experiments}.  \red{TODO: Eric says there are too many important details offloaded to Appendix.}


\begin{figure}[!t]
\centering
\includegraphics[width=.4\textwidth]{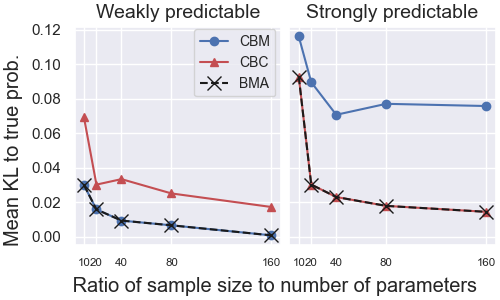}
\caption{\textbf{Demonstrating the value of BMA prediction.}  Each point corresponds to a simulated dataset ($K=3$,  $M=6$) with number of examples $N$ increasing along $x$-axis. 
For each dataset, we perform IB-CAVI posterior estimation then make predictions using two pure \CB~likelihoods, \CBM-Logit and \CBC-Logit, as well as Bayesian model averaging (BMA, Sec.~\ref{sec:determining_a_good_target_for_an_IB_approximation}) of these two models.
The y-axis reports the mean KL divergence from predictions to the true probabilities, averaged across the test set (lower is better).
\emph{Left:} Categories are \textit{weakly} predictable from covariates ($\sigma_\text{high} = 0.1$ in the generative process of Sec.~\ref{sec:simulating_data}). \emph{Right:} \textit{Strong} predictability ($\sigma_\text{high} = 2.0$).
}
\label{fig:BMA_results_logit_ordered}
\end{figure}

\input{inputs/fig_combined_likelihood_and_accuracy}

\begin{figure*}[!h]
\centering
\begin{tabular}{c c c c l}
    Small Simulated 
    & Detergents
    & Process Starts
    & Large Simulated  
    &
\\
    {\scriptsize $K{=}3, M{=}6, N{=}1000$}  
    &  {\scriptsize $K{=}6, M{=}6, N{=}2657$ } 
    & {\scriptsize $K{=}1553, M{=}1553, N{=}17724$}
    & {\scriptsize $K{=}500, M{=}1000, N{=}200000$}  
    & 
\\
    
    \includegraphics[height=2.5cm]{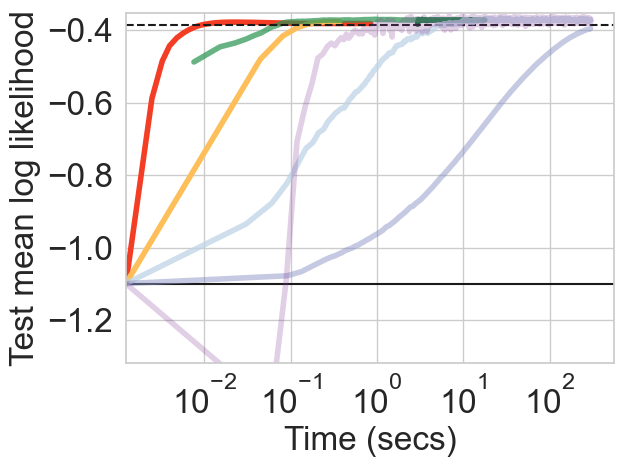} 
&
    \includegraphics[height=2.5cm]{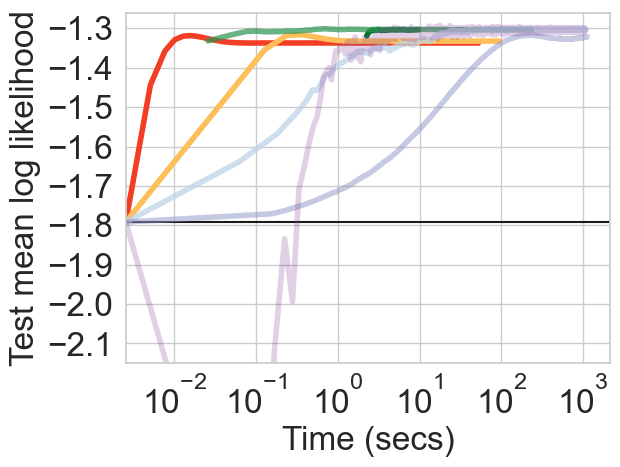}
& \includegraphics[height=2.5cm]{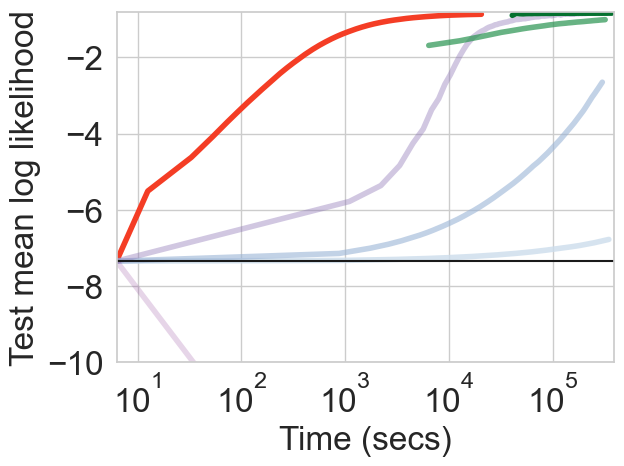}
&
 \includegraphics[height=2.5cm]{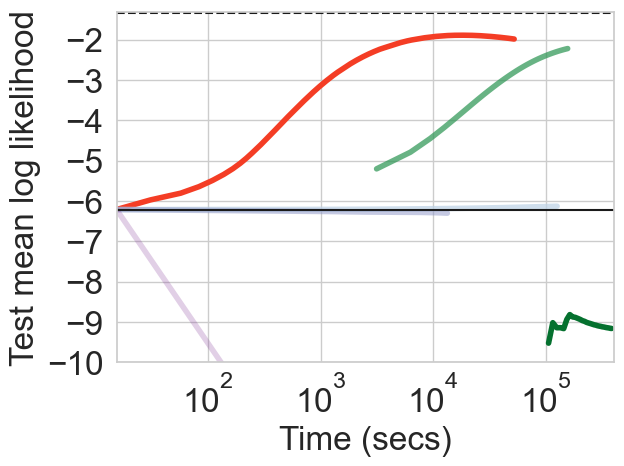}
& \includegraphics[height=2.5cm]{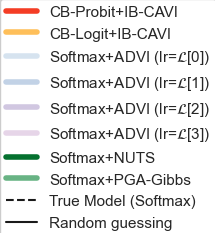}
\\
    \includegraphics[height=2.5cm]{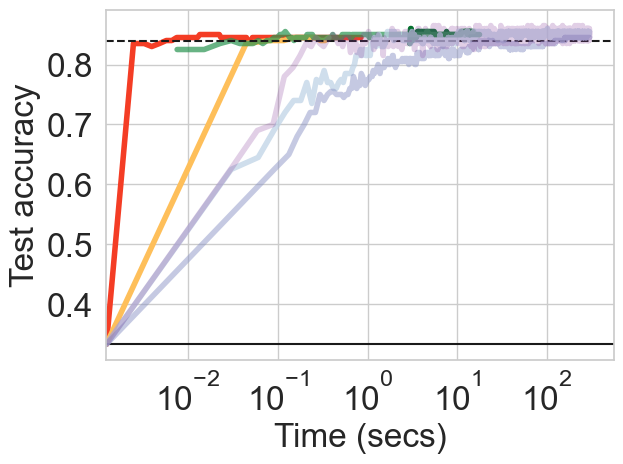}
   & \includegraphics[height=2.5cm]{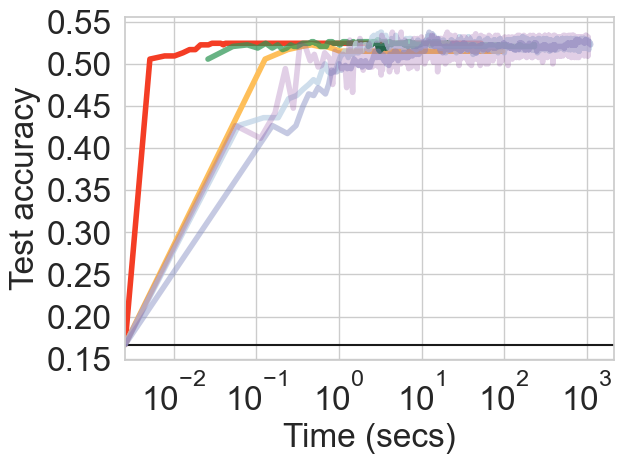}
& \includegraphics[height=2.5cm]{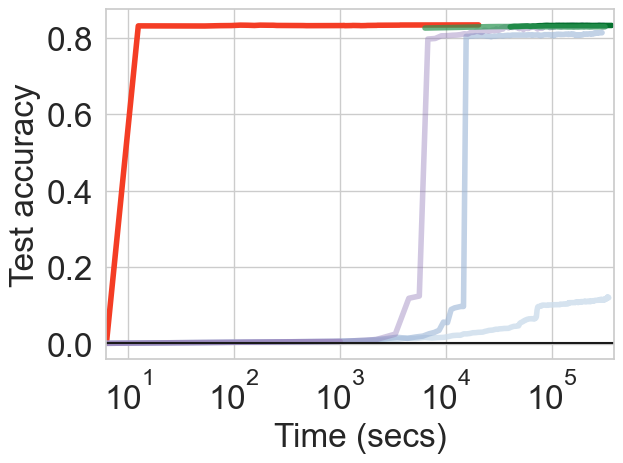}
&
    \includegraphics[height=2.5cm]{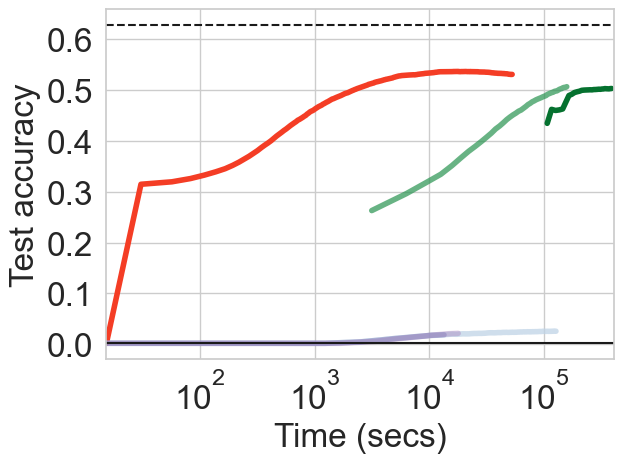}
\end{tabular}
    \caption{\textbf{Prediction quality by training time.} Bayesian inference methods are compared on real and simulated datasets with $K$ categories, $M$ covariates, and $N$ instances. Prediction quality is measured by holdout log likelihood (top) and accuracy (bottom).  For ADVI, we try the learning rates $\mathcal{L}:=(0.01, 0.1, 1.0, 10, 100)$ recommended by \cite{kucukelbir2017automatic}, adjusted to $10^{-1} \mathcal{L}$ in the larger simulated dataset to reduce divergence. If a line is absent for ADVI, the method diverged.  All methods were  initialized at the zero matrix (corresponding to random guessing), but we do not treat the initialization as a sample for MCMC methods.  Note that IB-CAVI's parallelism over $K$ was \textit{not} exploited here; using that, IB-CAVI's training time could be further reduced.
}
\label{fig:holdout_perf_over_time}
\end{figure*}


\subsection{Demonstrating value of BMA for predictions} \label{sec:experiment_on_averaging_approximation_targets_for_enhanced_performance}

Sec.~\ref{sec:determining_a_good_target_for_an_IB_approximation} described a Bayesian model averaging (BMA) technique to automatically determine the best \CB~target of an \IB~approximation for a given dataset.  Here we review the effectiveness of that technique. 
We generated simulated datasets using $K=3$ categories and $M=6$ covariates at varying sizes and levels of $y$-from-$x$ ``predictability'' (details in Sec.~\ref{sec:simulating_data}).
After fitting one approximate posterior $q$ via \IB-CAVI, we used this one posterior to make probabilistic predictions about the category labels of the held-out test set given corresponding covariates using three likelihoods: \CBC~only, \CBM~only, and our BMA approach that averages the two.
Fig.~\ref{fig:BMA_results_logit_ordered} plots the quality of predictions as the data-to-parameter ratio changes.
The first major takeaway is that our BMA averaging technique always matches the best possible prediction quality.
The second takeaway is that averaging is needed: our datasets with weakly predictable outcomes favor \CBM~likelihoods, while those with strongly predictable outcomes favor \CBC.

\subsection{IB-CAVI versus maximum likelihood} \label{sec:IB_CAVI_vs_MLE}

Using simulated data generated by a softmax model (Sec.~\ref{sec:simulating_data}), we assessed the quality of a CB-Probit model with IB-Probit-CAVI posterior estimation against a well-specified baseline: maximum likelihood estimation (MLE) for the softmax model. 
Figure \ref{fig:performance_simulations_combined_likelihood_and_accuracy} shows that our approach provides clear benefits over the MLE in terms of test log likelihood when the data are not sufficiently informative (i.e. the number of samples $N$ is small relative to the number of parameters $P$).
When there are few samples relative to parameters, the benefits of modeling uncertainty outweigh the cost of our approximation. 
When data are abundant ($N \gg P$), we expect the well-specified softmax MLE to do well, and comfortingly our prediction quality is quite close to it.
This performance is assuring especially because the data was not generated by a CB or IB model.

\subsection{Prediction quality by training time} \label{sec:holdout_perf_over_time}

In Fig.~\ref{fig:holdout_perf_over_time}, we study how different Bayesian methods for fitting categorical regressions perform on heldout test data as a function of training time.
We study two real datasets, detergent choices \citep{imai2005bayesian} and the computer process starts of one user ($K{=}1553$, Sec.~\ref{sec:intrusion_detection_experiment}), 
as well as a large and small simulated dataset generated by a softmax model  (Sec.~\ref{sec:simulating_data}).
In addition to our IB-CAVI, we compared to two ``gold standard'' MCMC samplers for softmax models: the No U-Turn Sampler (NUTS) \citep{hoffman2014no} and a Gibbs sampler available via \pga \citep{polson2013bayesian}.
We also compared to automatic differentiation variational inference (ADVI) \citep{kucukelbir2017automatic} for the softmax model.


Overall, we find that our IB-CAVI can deliver \emph{indistinguishable accuracy} and \emph{little-to-no cost in log likelihood} compared to alternative methods for categorical data, while requiring \emph{far less time} to get there.  Moreover, IB-CAVI is \emph{reliable}; each update is exact and optimal, and unlike alternatives does not require correctly choosing a learning rate (as with ADVI) or tuning period length (as with NUTS).

These conclusions are corroborated on two other real datasets, glass identification \cite{dua2019} and Anuran frog calls \citep{anuranFrogDataset2017}, in Sec.~\ref{sec:holdout_performance_over_time_supplemental}.  The heldout log likelihood plots sometimes suggest that IB-CAVI eventually overfits slightly (see \emph{Detergents} in Fig.~\ref{fig:holdout_perf_over_time}), though we can see the same behavior from Gibbs samplers (see Fig.~\ref{fig:holdout_perf_over_time_supplemental}).



\subsection{Assessing quality of the IB approximation}
\label{sec:approximation_quality_experiments}

To more directly assess the quality of the approximation required by our IB-CAVI approach, we compared the predicted category probabilities (non-negative vectors of size $K$ that sum to one) to those of the true model used to generate simulated data. 
Table~\ref{tab:approximation_error_when_using_IB_CAVI_and_BMA} quantifies that the KL divergence from predicted to truth is always less than 0.10 across a range of dataset sizes.
Histograms comparing the posterior distribution over category probabilities inferred by IB-CAVI, NUTS, and ADVI are visualized in Fig.~\ref {fig:violin_plots_posterior_over_category_probabilities}.
These empirical checks, combined with our previous experiments in Sec.~\ref{sec:IB_CAVI_vs_MLE} and \ref{sec:holdout_perf_over_time} reveal that \IB~approximation can be used to obtain high-quality probabilistic predictions that are competitive with more expensive truly-categorical approaches, especially when used with Bayesian model averaging (Sec.~\ref{sec:experiment_on_averaging_approximation_targets_for_enhanced_performance}).

Future research could provide an analytical characterization of the approximation error. 
We emphasize that comparing the induced posterior over category probabilities is preferred, given our stated goals in Sec.~\ref{sec:problem_formulation}. While comparisons of the posterior over weights $q(\+B)$ are possible, these are intermediate quantities less relevant to applications and may be confounded by identifiability issues.



\subsection{Intrusion detection in a cybersecurity application} \label{sec:intrusion_detection_experiment}

Now we show how our IB-CAVI can address a real problem in which there are many categorical outcomes. User behavior analytics (UBA) is a branch of cybersecurity which attempts to learn the ``normal" usage patterns of users on a computer network.  Deviations from normal behavior can point to suspicious or malicious activity that may be caused by unauthorized access to the network (often due to compromised user credentials), which we refer to as \textit{intrusion}.

We analyze $U=32$ users from Los Alamos National Laboratory's corporate, internal computer network~\citep{kent2015comprehensive}. For each user, we train a customized categorical GLM to learn a probability distribution over that user's process starts given previously started processes.  There were $K=1,553$ unique processes started. Our autoregressive-like featurization strategy (see Sec.~\ref{sec:featurization_for_intrusion_detection}) produced  $P=K(K+1)=2.4$ million model parameters, which far exceeds the number of process starts  per user ($N \approx 18,000$).
This is a \textit{small data-per-parameter} ($N<<P$) regime, where Bayesian methods can show benefits over maximum likelihood point estimation, as we observed in Fig.~\ref{fig:performance_simulations_combined_likelihood_and_accuracy} (left).

A model of user behavior is useful if its scores drop when an unauthorized user is using the computer.  Thus, we evaluate the user models via \textit{simulated intrusions}, using each user's model to score holdout process start sequences from the self vs. $(U-1)=31$ other users.  The quality of each model can be summarized with an \emph{intrusion detection score}, defined for each user $u=1,\hdots,32$ as
\[ \texttt{ID-score}(u) = \ell_{uu} - \textstyle \frac{1}{U-1} \sum_{v \neq u} \ell_{uv}\]
where $\ell_{uv}$ is the mean log likelihood of holdout process starts from user $v$ when scored with the model of user $u$.  
A higher score means a better ability to distinguish the target user $u$ from the other users.

Using these intrusion detection scores, we can compare two different variational approaches to learning categorical GLMs: IB-CAVI (applied to CB-Probit models) vs. ADVI (applied to Softmax models). We selected CB-Probit over CB-Logit due to its lower runtime costs (see Sec.~\ref{sec:procedure_for_posterior_estimation}). Based on the process start results in Fig.~\ref{fig:holdout_perf_over_time},  we set the ADVI learning rate to $1.0$, and we grant ADVI much longer training time (200 minutes per user) than IB-CAVI (20 minutes per user).  Results are shown in Figure~\ref{fig:process_start_self_vs_other_heatmap}. We find that despite ADVI's 10-fold advantage in training time, IB-CAVI produces markedly better intrusion detection performance, always matching or beating ADVI across all 32 users.


\begin{figure}[!t]
\centering
\includegraphics[width=.35\textwidth]{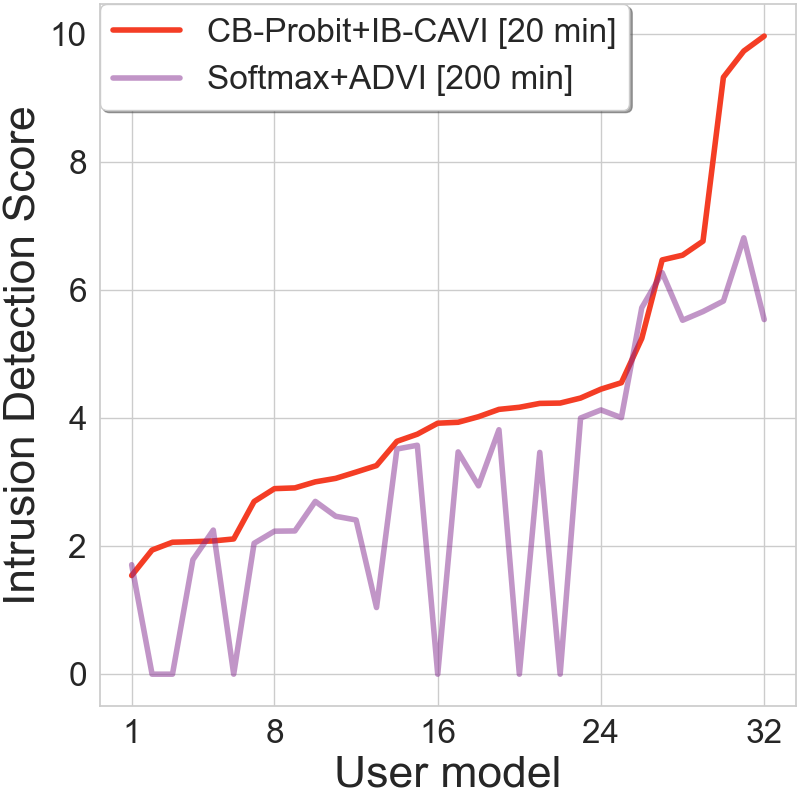}
\caption{\textbf{Scalable intrusion detection with IB-CAVI.} We train 32 categorical models to learn the computer usage patterns of 32 users from a corporate network.  Many ($K=1,553$) unique computer processes were started. Our autoregressive modeling strategy produced on the order of 2.4 million ($=K^2$) parameters.  Due to the large scale, we compare only variational methods.  We grant IB-CAVI much shorter training time (20 minutes per user) than ADVI (200 minutes per user), yet find that IB-CAVI produces markedly better intrusion detection performance. 
}
\label{fig:process_start_self_vs_other_heatmap}
\end{figure}





\section{Conclusion}

We have defined a class of categorical models amenable to a binary approximation that enables fast and closed-form posterior estimation.
Our key technical innovations are a likelihood bound that justifies this approximation and the use of model averaging to improve prediction quality without increased training cost.
We find that on real data the posterior estimated via our binary approximation can deliver categorical predictions of similar quality as more expensive baseline methods in a fraction of the time.

Future work could provide a theoretical characterization of approximation error or determine if other constructions of \CB~models are useful.
One could also extend our \IB~approximation to more sophisticated hierarchical models or models with variable selection priors. 

\red{INCORPORATE: We show how to model \textit{categorical} data with an independent \textit{binary} model.}

\red{INCORPORATE: We optimize using IB which is independent of the CFB choice.  Can we then plug the estimated parameters into either or both of the models?  Is there anything about the optiization problem using the IB model which is specific to the underlying CFB model?}

\red{INCORPORATE: The key insight is that we can provably optimize a lower bound on a categorical-from-binary model by instead maximizing the traditional ELBO of the independent binary model.}

\red{TODO: Scan rest of document and try to make sure we speak of "a" CBM/CBT model and not "the" CBM/CBT model; the latter should be reserved for when we pick a link (CBT-Probit, CBM-Logit, etc.)}

\section*{Acknowledgements}
\input{acks.tex}

\bibliography{references_icml.bib}

\clearpage

\addtocontents{toc}{\protect\setcounter{tocdepth}{2}}

\appendix
\counterwithin{table}{section}
\setcounter{table}{0}
\counterwithin{figure}{section}
\setcounter{figure}{0}

\renewcommand{\contentsname}{Appendix Contents}
\tableofcontents

\section*{Code and Data Availability} \label{sec:}

Open-source python code for reproducing experiments can be found at
\url{\codeURL}

The above repository also contains code to download the simulated datasets and open-access real datasets (Detergent, Glass, and Frog Calls, Cybersecurity events) used in this work, along with any preprocessing used. 

\section{Background: General binary regression} \label{sec:additional_interpretation_for_the_general_binary_regression_model}
Here we obtain an additional interpretation for the ``success" probability in a general binary regression model (Eq.~\eqref{eqn:general_binary_regression}) if we are willing to make a few additional assumptions.  When $H$ is the cdf of a distribution $\mathcal{H}$ that is a location family with a symmetric density (so $\mathcal{H}$ could be Gaussian, Logistic, Cauchy, Laplace, Student-t, and so on), a simple argument reveals the further interpretation that $\P(\BINARYy_i = 1 \cond \BINARYvecbeta) = \P\set{\text{drawing a positive value from $\mathcal{H}$ when its mean is $\+x_i^T \BINARYvecbeta$}}$. We formalize this in Prop.~\ref{prop:additional_interpretation_for_the_general_binary_regression_model}.

\begin{proposition}
Let $\mathcal{H}_{\mu}$ be a symmetric location-scale family of probability measures that has a density $h_\mu$ which is  continuous almost everywhere with respect to Lebesgue measure.  Let $\mathcal{H}_{\mu}$ have expected value $\mu$ and  variance fixed to 1. Let $z \sim \mathcal{H}_\mu$, and  denote the cdf of $\mathcal{H}_{\mu}$ by $H_\mu$.  Then $H_0(\mu)=\P(z \geq 0)$. 
\label{prop:additional_interpretation_for_the_general_binary_regression_model}
\end{proposition}
\begin{proof}
\begin{align*}
\P(z \geq 0) &= 1 -H_\mu(0) \\
&= \ds\int_0^\infty h_\mu(x) \wrt{x} && \tinytext{by Riemann integrability}\\
&= \ds\int_0^\infty h_0(x-\mu) \wrt{x}  && \tinytext{as a location-scale family with unit variance}\\ 
&= \ds\int_{-\mu}^\infty h_0(u) \wrt{u}  && \tinytext{by substituting $u=x-\mu$}\\
&= 1- H_0(-\mu) \\
&=H_0(\mu). && \tinytext{by symmetry}
\end{align*}
\end{proof}


\section{Categorical-from-binary (CB) models}

\subsection{Impossibility of exactness in the likelihood bound}

\begin{proposition}
Equality cannot be achieved in the likelihood bound of Eq.~\eqref{eqn:CFB_likelihood_larger_than_IB}.
\label{prop:impossibility_of_exactness_in_the_likelihood_bound}
\end{proposition}

\begin{proof}
By way of contradiction, suppose that
\begin{align*}
p_\CB(y_i = k \cond \+B) = p_\IB(\IBvecy_i = \+e_{k} \cond \IBmatrixB = \+B)
\labelit \label{eqn:likelihood_bound_if_equality_were_possible}
\end{align*}
 Summing these terms over $k=1, \hdots, K$  must yield a value of 1, since the left hand side is a probability distribution over the $K$ categories.  This says that $p_\IB(\Upsilon \cond \+B) =1$, where we define $\Upsilon \subsetneq \set{0,1}^K$ as the space of one-hot encoded vectors. So by countable additivity, 
 \begin{align*}
 p_\IB(\Upsilon^c \cond \+B) = 0
 \labelit \label{eqn:under_a_false_assumption_the_IB_model_assigns_zero_probability_to_one_hot_vectors} 
 \end{align*}

At the same time,  we must have that
 \begin{align*}
 0 < p_\IB(\+v \cond \+B) < 1, \quad \forall \+v \in \set{0,1}^K
 \labelit \label{eqn:IB_likelihood_yields_value_between_0_and_1_non_inclusive_for_any_vector} 
 \end{align*}
Eq.~\eqref{eqn:IB_likelihood_yields_value_between_0_and_1_non_inclusive_for_any_vector} follows immediately from Eq.~\eqref{eqn:IB_likelihood}, because a cdf $H$ satisfies that $H(x), 1-H(x) \in (0,1)$ for any finite real-valued $x$, and $p_\IB(\+v \cond \+B)$ is the product of $K$ such terms.

Now Eqs.~\eqref{eqn:IB_likelihood_yields_value_between_0_and_1_non_inclusive_for_any_vector} and \eqref{eqn:under_a_false_assumption_the_IB_model_assigns_zero_probability_to_one_hot_vectors} contradict, so the hypothesis in Eq.~\eqref{eqn:likelihood_bound_if_equality_were_possible} is false. 
\end{proof}

\subsection{The \CBC~model as a normalized odds model} \label{sec:CBC_as_normalized_odds}

In the main paper, the \CBC~model was given as:
\begin{align}
p_\CBC(y_i {=} k \cond \+B) = 
\df{H( \eta_{ik}) \ds\prod_{j \neq k} (1-H(\eta_{ij}))}{\ds\sum_{\ell=1}^K  H(\eta_{i\ell}) \ds\prod_{j \neq \ell} (1-H(\eta_{ij}))    }
\label{eqn:CBC_category_probabilities_appendix}
\end{align}
where $\eta_{ik}=\+x_i^T \+\beta_k$.

A \CBC~model has an alternate expression as a \textit{normalized odds} model:
\begin{align*}
p_\CBC(y_i = k \cond \+B) = \df{H( \eta_{ik} ) / \big(1-H( \eta_{ik})\big)}{ \sum_{\ell=1}^K H( \eta_{i\ell} ) / \big(1-H( \eta_{i\ell})\big)}
	\labelit \label{eqn:CBT_category_probabilities_as_odds}
\end{align*}

Here we show that the two representations for the \CBC~model (Eqs.~\eqref{eqn:CBC_category_probabilities_appendix} and \eqref{eqn:CBT_category_probabilities_as_odds}) are equal.  Observe
\begin{align*}
p(y_i=k \cond \+B) 
&\stackrel{\eqref{eqn:CBC_category_probabilities}}{=} \df{ H(\eta_{ik})  \ds\prod_{j \neq k} \big(1-H(\eta_{ij})\big)     }{\ds\sum_{\ell=1}^K  H( \eta_{i\ell})\ds\prod_{j \neq \ell} \big(1-H(\eta_{ij})\big)} \\ 
&\stackrel{1}{=} \df{ H(\eta_{ik})  \ds\prod_{j \neq k} \big(1-H(\eta_{ij})\big)     }{\ds\sum_{\ell=1}^K  \df{H( \eta_{i\ell})}{\big(1-H(\eta_{i\ell})\big)} \ds\prod_{j=1}^K \big(1-H(\eta_{ij})\big)}    \\
&\stackrel{2}{=}  \df{ \ H(\eta_{ik}) / \big(1-H(\eta_{ik})\big) }{\ds\sum_{\ell=1}^K H(\eta_{i\ell}) / \big(1-H(\eta_{i\ell})\big)  }
\end{align*}
Equality (1) just multiplies by $1=\frac{a}{a}$ in the numerator and denominator, and Equality (2) just cancels. 

\subsection{Membership of \CBC~and \CBM~models in the \CB~class.}
\label{sec:CBT_and_CBM_are_CB_models_appendix}

\begin{proposition}
\CBC~and \CBM~models are categorical-from-binary (\CB)~models. 
\label{prop:CBT_and_CBM_are_CB_models_appendix}
\end{proposition}

\begin{proof}
We begin with \CBM~models. For any $k=1,...,K$,
\[ p_\IB(\IBvecy_i = \+e_k \cond \IBmatrixB = \+B) <
p_\CBM(y_i =k \cond \+B)  \]
holds if 
\begin{align*}
w_k \ds\prod_{j \neq k} (1-w_j) &<\df{w_k}{\sum_{\ell=1}^K w_\ell} 
\labelit \label{eqn:main_proof_likelihoods_reduced_to_weights}
\end{align*}
for $w_k \in (0,1), k=1,...,K$. The implication follows from the \IB~and \CB~likelihood equations (Eqs.~\eqref{eqn:CBM_category_probabilities} and~\eqref{eqn:IB_likelihood}) and the fact that any cdf $H$ maps into $(0,1)$.  Some algebra reduces Eq.~\eqref{eqn:main_proof_likelihoods_reduced_to_weights} to  
\begin{align*}
\left(
\ds\sum_{\ell=1}^K w_\ell
\right) \left( \ds\prod_{j \neq k} (1-w_j) \right) &<1 
\labelit \label{eqn:main_proof_likelihoods_reduced_to_weights_use}
\end{align*}
We establish Eq.~\eqref{eqn:main_proof_likelihoods_reduced_to_weights_use} by induction on $K$.  Without loss of generality, we assume $k=1$.  
\begin{itemize}
\item \textit{(Base.)}  We show Eq.~\eqref{eqn:main_proof_likelihoods_reduced_to_weights_use} holds for $K=2$:
\begin{align*}
w_1 (1-w_2) + w_2 ( 1- w_2) &\stackrel{(w_1 < 1)}{<} (1-w_2) (1+w_2) \\
&\stackrel{(w_2 < 1)}{<} 1.	
\end{align*}
\item \textit{(Step.)} We show that if Eq.~\eqref{eqn:main_proof_likelihoods_reduced_to_weights_use} holds for some $K$ then it holds for $K+1$. We define $\alpha_{k,K} := \sum_{\ell=1}^K w_\ell \prod_{j \neq k} (1-w_j)$ and assume $\alpha_{k,K} <1$. Now
{\footnotesize \begin{align*}
\alpha_{k,K+1} &= \explaintermbrace{$<1$ by hypoth.}{\alpha_{k,K}} (1-w_{K+1}) + w_{K+1} \explaintermbrace{$<1$}{\prod_{j \neq k} (1-w_j)} \\
&< (1-w_{K+1}) + w_{K+1} = 1.
\end{align*}}
\end{itemize}

Now we proceed to \CBC~models. We must show that for any $k=1,...,K$,
\begin{align*} 
p_\IB(\IBvecy_i = \+e_k \cond \IBmatrixB = \+B) <
p_\CBC(y_i =k \cond \+B) 
\labelit \label{eqn:WTS_IB_is_lower_bound_on_CBC}
\end{align*}
By the likelihood equations for the \IB~and \CBC~models (Eqs.~\eqref{eqn:IB_likelihood} and \eqref{eqn:CBC_category_probabilities}), we can write
\begin{align*}
p_\CBC(y_i =k \cond \+B) = \df{p_\IB(\IBvecy_i = \+e_k \cond \IBmatrixB = \+B)}{p_\IB(\Upsilon \cond \IBmatrixB = \+B)},
\labelit \label{eqn:CBC_as_normalized_IB}	
\end{align*}
where $\Upsilon \subset \set{0,1}^K$ is the space of one-hot encoded vectors. Now take any $\+v \in \Upsilon^c$. Then 
\begin{align*}
p_\IB(\Upsilon \cond \+B) & \stackrel{subadditivity}{\leq}	p_\IB(\set{0,1}^K \cond 
\+B) - p(\+v \cond \+B) \\
&\stackrel{Eq.~\eqref{eqn:IB_likelihood_yields_value_between_0_and_1_non_inclusive_for_any_vector}}{<} p_\IB(\set{0,1}^K \cond 
\+B) = 1. 
\end{align*}
So the denominator of Eq.~\eqref{eqn:CBC_as_normalized_IB} is non-negative and less than one, which implies Eq.~\eqref{eqn:WTS_IB_is_lower_bound_on_CBC}.
\end{proof}

\subsection{Impossibility of exact inference on a \CBM~or \CBC~model via inference on an \IB~model}\label{sec:impossibility_of_exact_inference_on_a_CB_model_via_inference_on_an_IB_model}

As discussed in Sec.~\ref{sec:models_overview}, efficient Bayesian inference exists for \IB~models.  As a result, we would like to relate \CB~models to \IB~models in such as way as to obtain efficient inference for the \textit{categorical} models. 

Here we consider whether an \textit{exact} relation exists.   In  Sec.~\ref{sec:can_CB_models_be_identified_via_IB_models}, we show that whereas \CBC~and \CBM~models are non-identifiable, 
\IB~models are identifiable (see Def.~\ref{def:identifiability}).  Now suppose there were an identifiability constraint 
such that a \CB~model under the identifiability constraint provided the same probabilities as an \IB~model applied to one-hot-encoded category representations.  Then there would be no need to consider \textit{approximate} inference, as inference on the regression weights $\IBmatrixB$ for the \IB~model would be \textit{exactly} equivalent to inference on the regression weights for the (identified) \CB~model.   

 Unfortunately, no such identifiability constraint exists for either \CBC~or \CBM~models.  While previous work \cite{johndrow2013diagonal} has suggested that the \CBC~model can be identified via the \IB~model,  Prop.~\ref{prop:IB_gives_an_identifiability_restriction_for_CBM_but_not_CBT_in_the_intercepts_only_setting} shows that this cannot be true.   On the other hand,  while Prop.~\ref{prop:IB_gives_an_identifiability_restriction_for_CBM_but_not_CBT_in_the_intercepts_only_setting} may elicit hope that the \CBM~model could be identified via an  \IB~model, Sec.~\ref{sec:evaluating_CB_models_as_targets_of_an_IB_approximation_wrt_soft_predictions} reveals empirically that this is no longer true once covariates are introduced. 
   
\subsubsection{Can \CB~models be identified via \IB~models? } \label{sec:can_CB_models_be_identified_via_IB_models}

Let us begin with a general definition of identifiability. 

\begin{definition}{[\textit{\cite{cole2020parameter} pp.35}]} A likelihood $p(\cdot \cond \theta)$ with parameter $\theta$ is globally identifiable if $p(\cdot \cond \theta_1) = p(\cdot \cond \theta_2)$ implies that $\theta_1=\theta_2$. A model is locally identifiable if there exists an open neighborhood in the parameter space of $\theta$ such that this is true. Otherwise a model is non-identifiable.
\label{def:identifiability}
\end{definition}

For models considered in the main body of this paper, identifiability requires
\begin{align*}
p_m(\cdot \cond \+B_1) = p_m(\cdot \cond \+B_2) \implies \+B_1=\+B_2
\labelit \label{eqn:global_identifiability_in_our_models}	
\end{align*}
where $m$ can take the value of either \IB,~\CBC, or  \CBM~models.     \red{TODO: This presumably also requires the design matrix $\+X$ to be full-rank. See if I can connect this to some statement about identifiability in GLM's.}

We now investigate the identifiability of these models in the simplest possible scenario: the intercepts-only (no covariates) setting.  In this setting, $M=1$, $\+x_i = 1$ for all $i=1,2,\hdots,N$, and the regression matrix simplifies to a vector $\+\beta \in \R^K$. Although this setting is simple, it is sufficient to provide some insight about identifiability.  

\begin{proposition}
The \CBC~and \CBM~models are non-identifiable in the intercepts-only setting. 
\end{proposition}
 
\begin{proof}
Let $(p_1, \hdots, p_K)$ be a  probability mass function over $K$ categories.   Then we can construct regression weights $\+\beta^\CBM, \+\beta^\CBC \in \R^K$ for the two models by taking the $k$th entry of each vector to be given by 
\begin{align}
\beta_k^{\CBM} & \in \biggset{ H^{-1}(rp_k) }_{\,r \in (0, \min_\ell 1/p_\ell)}	 \label{eqn:beta_entry_for_CBM_given_desired_pmf_in_intercepts_only_setting} \\
\beta_k^{\CBC} & \in \biggset{ H^{-1}\bigg( \df{sp_k}{1+sp_k} \bigg)  }_{s>0} 	\label{eqn:beta_entry_for_CBT_given_desired_pmf_in_intercepts_only_setting}
\end{align}
where Eq.~\eqref{eqn:beta_entry_for_CBT_given_desired_pmf_in_intercepts_only_setting} follows from setting the model's category probabilities in \eqref{eqn:CBT_category_probabilities_as_odds} to $p_k$. Therefore  Eq.~\eqref{eqn:global_identifiability_in_our_models} is not satisfied for any open neighborhood of the parameter space.
\end{proof}

 Now let us consider the \IB~model, specifically in the case where we use it to do inference with one-hot encoded representations of categorical data, $\IBvecy_i = \+e_{y_i}$ for $y_i \in \set{1,\hdots, K}$.  If we set
\begin{align*}
p_\IB(1,0,0\hdots,0) &= p_1 \\
p_\IB(0,1,0\hdots,0) &= p_2 \\ \vdots  \\ p_\IB(0,0,0\hdots,1) &= p_K
\end{align*}
and $p_\IB(\+v) =0$ for all vectors $\+v \in \set{0,1}^K$ that are not one-hot, then we have transformed $(p_1, \hdots, p_K)$ into a probability mass function over K-bits.  Via Eq.~\eqref{eqn:independent_binary_model}, this pmf implies a unique
vector of regression weights $\IBvecbeta^\IB \in \R^K$, with $k$th entry given by 
\begin{align}
\IBbeta_k^{\IB} &= H^{-1}(p_k)	 
\label{eqn:beta_entry_for_IB_given_desired_pmf_in_intercepts_only_setting}
\end{align}
So unlike the \CBC~and \CBM~models, the \IB~model \textit{is}  identifiable (and globally so), at least when we apply it only to one-hot encoded data.  \red{TODO: Should we report on the more general situation?}

\begin{remark}
If we consider the special case where $p_1,...,p_K$ are the empirical category frequencies in a $K$-class intercepts-only dataset, $p_k = \frac{1}{N} \sum_{i=1}^N \indicate{y_i=k}$, then Eqs.~\eqref{eqn:beta_entry_for_CBM_given_desired_pmf_in_intercepts_only_setting}, \eqref{eqn:beta_entry_for_CBT_given_desired_pmf_in_intercepts_only_setting}, and \eqref{eqn:beta_entry_for_IB_given_desired_pmf_in_intercepts_only_setting} give the maximum likelihood estimators.  Note in particular that since \CBC~and \CBM~models are non-identifiable, the maximum likelihood estimators are not unique.  \red{TODO: Confirm that setting the modeled category probabilities to $p_k$ works to find the MLE.  Since the \CBM~and \CBC~ models are under-identified, we can't start by reasoning about the p's and then apply the invariance of MLE property to obtain the MLE of the beta's. I think a simple chain-like argument works though.}
\label{rk:CBM_and_CBT_lack_unique_MLEs}
\end{remark}

So whereas the \CBC~and \CBM~models are non-identifiable, the \IB~model is globally  identifiable.  Moreover,  \CBC~and \CBM~models are ``built from" \IB~models in the sense described in Sec.~\ref{sec:two_classes_of_CB_models}.  These observations give rise to a natural question.   Does the \IB ~model give an \textit{identifiability constraint } for the \CBC~or \CBM~models?  That is, given a set of regression weights $\+B$ that produce an equivalent likelihood for the categorical model, can we choose a representative by setting $\+B = \IBmatrixB_\text{MLE}$, where $\IBmatrixB_\text{MLE}$ is the maximum likelihood estimate of the regression weights of an \IB~model given one-hot encoded representations of categorical data?    We address this question mathematically in Prop.~\ref{prop:IB_gives_an_identifiability_restriction_for_CBM_but_not_CBT_in_the_intercepts_only_setting} for the intercepts-only case, and empirically in Sec.~\ref{sec:evaluating_CB_models_as_targets_of_an_IB_approximation_wrt_soft_predictions} for the with-covariates case.


\begin{proposition}
In the intercepts-only setting, the \IB~model gives an identifiability constraint for the \CBM~model, but not for the \CBC~model.
\label{prop:IB_gives_an_identifiability_restriction_for_CBM_but_not_CBT_in_the_intercepts_only_setting}
\end{proposition}

\begin{proof}
For the \CBM~model, take r=1 in Eq.~\eqref{eqn:beta_entry_for_CBM_given_desired_pmf_in_intercepts_only_setting}, and we obtain \eqref{eqn:beta_entry_for_IB_given_desired_pmf_in_intercepts_only_setting}.   For the \CBC~model, we require 
\begin{align*}
H^{-1}(p_k) &\in \bigg\{ H^{-1} \bigg( \df{sp_k}{1+sp_k} \bigg) \bigg\}_{s >0}  \\ &\implies \exists s>0 : p_k = \df{s-1}{s} \implies p_k = 1/K. 
\end{align*} 
Since the empirical probabilities are not always uniform, the \IB~model cannot provide an identifiability constraint for the \CBC~model.
\end{proof}

\begin{remark}
From Prop.~\ref{prop:IB_gives_an_identifiability_restriction_for_CBM_but_not_CBT_in_the_intercepts_only_setting}, it follows that in the intercepts-only setting, $\IBmatrixB_{\text{MLE}}$, the maximum likelihood estimate of the regression weights of an \IB~model given one-hot encoded representations of categorical data, is an MLE for the \CBM~model, but not for the \CBC~model. 	
\end{remark}


\red{TODO: Somewhere add the remark about the intercepts only setting already implying that the general (with-covaraites) model is not globally identifiable.}

\subsection{Evaluating \CB~models as targets of an \IB~approximation} \label{sec:evaluating_CB_models_as_targets_of_an_IB_approximation}

In this section, we evaluate two different classes of \CB~models (\CBC~and \CBM) in terms of how well they serve as targets of an \IB~approximation.  In Sec~\ref{sec:evaluating_CB_models_as_targets_of_an_IB_approximation_wrt_soft_predictions}, we make an evaluation in terms of ``soft" predictions (i.e. likelihood). In Sec.~\ref{sec:evaluating_CB_models_as_targets_of_an_IB_approximation_wrt_hard_predictions}, we make an evaluation in terms of ``hard" predictions (i.e. misclassification rates).   

\subsubsection{Soft predictions} \label{sec:evaluating_CB_models_as_targets_of_an_IB_approximation_wrt_soft_predictions}

Now we evaluate the quality of \CBC~and \CBM~as targets of an \IB~approximation with respect to \textit{soft predictions} (the probability vectors produced by a categorical likelihood). 

\paragraph{Methodology.}
We generate two datasets using the data generation technique of Section~\ref{sec:data_generation}, fixing in both the response predictability at $\sigma_{\text{high}} = 0.1$ to create a challenging problem.
The smaller dataset has $K{=}3, M{=}1, N{=}1800$. 
The larger dataset has $K{=}20, M{=}40, N{=}12000$.


For each dataset, we perform gradient descent (using automatic differentiation of the relevant likelihoods from multiple initialization seeds) to estimate $\IBmatrixB_{\text{MLE}}^\CBC$ and $\IBmatrixB_{\text{MLE}}^\CBM$, the MLEs for the \CBC-Probit and \CBM-Probit models, respectively. Similarly, we estimate $\IBmatrixB_{\text{MLE}}^\IB$, the MLE for the \IB-Probit model when it is fit on one-hot encoded representations of the categorical data.

Then, for each example $i$ in the training data, we compute the categorical probability vector $\+s_i \in \Delta_{K-1}$ using the weights that minimize the truly-categorical \CB~likelihood, as well as probability vector $\hat{\+s}_i \in \Delta_{K-1}$ corresponding to the weights estimated to minimize the IB likelihood. (Recall that vector $\+ s_i$ can be produced given weights $\+B$ via Eq.~\eqref{eqn:categorical_regression}).
Suppose $k$ is the class with the highest probability in the vector $\hat{\+s}_i$: we wish to compare the signed error between the ``ideal'' $\+s_{ik}$ and our approximations $\hat{\+s}_{ik}$ in order to assess the quality of our IB approximation.



\paragraph{Results.} Figure~\ref{fig:evaluating_CBM_and_CBT_as_targets_of_an_IB_approximation} demonstrates two important points:

\begin{enumerate}
\item The MLE for the \IB~model is not an \textit{exact} MLE for the either the \CBC~or \CBM~models.  As a result, neither of the two models is a globally identified \IB~model. (If they were, then the signed error would always equal 0.) This provides an empirical refutation that IB gives an identifiability constraint for \CB~models.  Previously, \citet{johndrow2013diagonal} suggested that such a relationship might hold (at least for \CBC~models).  Prop.~\ref{prop:IB_gives_an_identifiability_restriction_for_CBM_but_not_CBT_in_the_intercepts_only_setting} proved that such a relationship cannot hold for \CBC~models in the intercepts-only setting ($M=0$), and these results provide an empirical refutation for \CBC~models in the with-covariates setting ($M\geq 1$). Moreover, while Prop.~\ref{prop:IB_gives_an_identifiability_restriction_for_CBM_but_not_CBT_in_the_intercepts_only_setting} revealed that the identification relationship \textit{does} hold for \CBM~models in the intercepts-only setting ($M=0$); it apparently does not hold in general.  
\item Neither \CBC~models nor \CBM~models are uniformly dominant as a target of an \IB~approximation. For the smaller dataset, the \CBM~model is superior to \CBC; the approximation error is lower in the top right panel than the top left panel. However, for the larger dataset, the \CBC~model is superior to \CBM; the approximation error is lower in the bottom left panel than the bottom right panel.  Thus, \textit{either} of the $\set{\CBC, \CBM}$ models can provide superior performance to the other as targets of an \IB~approximation.  The relative advantage depend on properties of the data. 
\red{TODO: Consider updating the comparison so that it depends on category predictability. This would help emphasize that which model to choose as the approximation target can't be determined a priori, as it depends on the level of response predictability.}
\end{enumerate}

Figure~\ref{fig:evaluating_CBM_and_CBT_as_targets_of_an_IB_approximation} also provides information on the \textit{amount} of error incurred by using an \IB~approximation.  For similar results in the context of \IB-CAVI algorithm, see Table~\ref{tab:approximation_error_when_using_IB_CAVI_and_BMA}.

\begin{figure}[htp!]
\centering 
\begin{tabular}{c}
CBC \hspace{3cm} CBM
\\
\includegraphics[width=.46\textwidth]{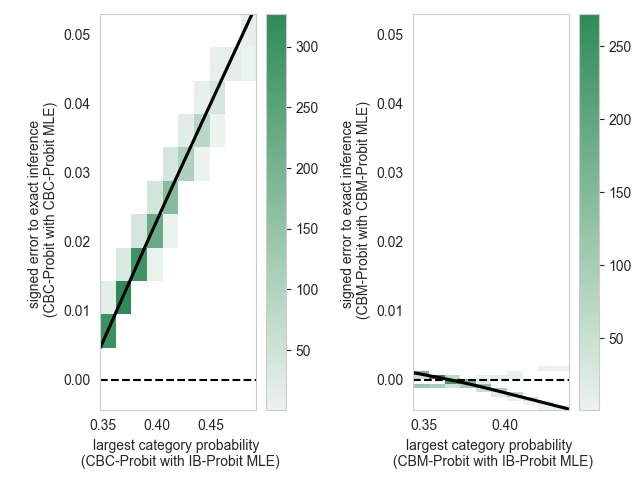}  \\
\includegraphics[width=.46\textwidth]{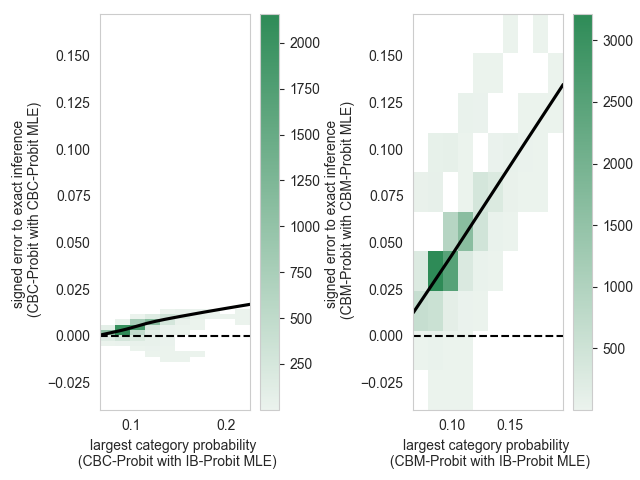} \\
 \end{tabular}
\caption{
2D histograms of signed error $s_{ik} - \hat{s}_{ik}$ in the probability mass assigned to the most probable category $k$ by a \CB~model and its \IB~approximation, using ML estimated weights.
Each panel shows a distribution across training examples, where each example $i$ contributes to the bin corresponding to the probability of its most probable category $k$ (x-axis) and its signed error (y-axis).
Darker colors imply more examples belong to that bin.
\emph{Top panels:} Smaller dataset with $K=3, M=1, N=1800, \sigma^2_{high}=0.1$.
\emph{Bottom panels:} Larger dataset with $K=20, M=40, N=12,000, \sigma^2_{high}=0.1$.
\emph{Left panels:} A \CBC~model is used to estimate the weights that determine $\+s_i$. 
\emph{Right panels:} A \CBM~model is used to estimate the weights that determine $\+s_i$.
The solid black line gives nonparametric lowess (locally weighted linear regression) estimates of the expected relationship between largest category probability and error.
}
\label{fig:evaluating_CBM_and_CBT_as_targets_of_an_IB_approximation}
\end{figure}

\subsubsection{Hard predictions} \label{sec:evaluating_CB_models_as_targets_of_an_IB_approximation_wrt_hard_predictions}

With respect to \textit{hard predictions} (misclassification rates), the quality of the \IB~approximation strategy is invariant to whether \CBC~or \CBM~is used as the target of the approximation.  This fact is implied by the proposition below. 

\begin{proposition}
For fixed regression weights and covariates, the likelihoods for \CBC~and \CBM~put the most probability mass on the same category.  That is
\begin{align*}
k^* := \argmax_k p_{\CBC}(y_i=k \cond \+\beta) = \argmax_k  p_{\CBM}(y_i=k \cond \+\beta).  
\end{align*}

Moreover, the most likely category matches that given by the most likely one-hot vector $\+e_k$ from the IB model: 
\[ k^* = \argmax_k  p_{\IB}(\IBvecy_i = \+e_k \cond \+\beta).\]
%
\label{prop:IB_plus_OHT_and_IB_plus_NSP_make_the_same_predictions}
\end{proposition}

\begin{proof}
Define the linear predictor $\eta_{ik} := \+x_i^T \+\beta_k$.  Then 
\begin{align*}
p_{\CBM}(y_i=k \cond \+\beta) & \stackrel{\eqref{eqn:CBM_category_probabilities}}{=} \df{   H(\eta_{ik})  }{\sum_{k=1}^K    H( \eta_{ik}) }  
\propto f_{\CBM}(\eta_{ik})  \\
p_{\CBC}(y_i=k \cond \+\beta) 
&\stackrel{\eqref{eqn:CBT_category_probabilities_as_odds}}{=} \df{  H(\eta_{ik}) / \big( 1 - H(\eta_{ik}) \big) }{\sum_{k=1}^K    H( \eta_{ik}) / \big( 1 - H(\eta_{ik}) \big) } \propto f_{\CBC}(\eta_{ik}) \\
p_{\IB}(\IBvecy_i = \+e_k \cond \+\beta)  & \stackrel{\eqref{eqn:IB_likelihood}, \, \eqref{eqn:CBT_category_probabilities_as_odds}}{=}  H(\eta_{ik}) / \big( 1 - H(\eta_{ik}) \big) \propto f_{IB}(\eta_{ik}) 
\end{align*}
where we have introduced notation $f_{\CBM}$, $f_{\CBC}$, $f_{\IB}$ to refer to the potential functions (unnormalized category probabilities) for each model.  Now since $H$ is an increasing function, $f_{\CBC}$, $f_{\CBM}$ $f_{\IB}$ are all increasing functions of the linear predictor $\eta_{ik}$.  Thus,
\begin{align*}
\argmax_k p_{\CBC}(y_i=k \cond \+\beta) &= \argmax_k  p_{\CBM}(y_i=k \cond \+\beta) \\
&= \argmax_k  p_{\IB}(\IBvecy = \+e_k \cond \+\beta) \\
&= \argmax_k \eta_{ik}
\end{align*}

and the proposition holds. 
\end{proof}

\begin{remark}
Proposition \ref{prop:IB_plus_OHT_and_IB_plus_NSP_make_the_same_predictions} can be extended to a more general proposition that the two approximation strategies provide identical rankings for \textit{all} categories. \red{TODO: Consider writing it this way.} Indeed, we observe this phenomenon in practice. 
\end{remark}

%
%

\section{General variational algorithm for CB models} \label{sec:general_variational_algorithm}

Here provide a strategy for performing variational inference with \emph{any} model $\M$ whose joint density includes a \CB~likelihood; that is, the joint density has the form
\[ p_M(\+y,\+u) = p_{\CB}(\+y \cond \+u) \pi(\+u) \]
where $\+y$ are observed random variables, $\+u = (u_v)_{v=1}^V$ are unobserved random variables, $p_{\CB}$ is a \CB~likelihood, and $\pi$ is a  prior density.  This strategy can be applied to the \CBC~and \CBM~models, but it is also extensible to more complicated graphical models, such as hierarchical models or models with variable selection priors (such as the normal-gamma prior \cite{brown2010inference} which generalizes Bayesian lasso,  or the horseshoe prior \cite{carvalho2009handling}, for which a conditionally conjugate formulation exists \cite{makalic2015simple}).  
We summarize our VI strategy in Algorithm 1. 

\textbf{Algorithm 1}: \texttt{IB-CAVI} \textit{for approximate Bayesian inference on any model with a \CB~likelihood}.  Given $N$ observed categorical responses $\+y \in \set{1,...,K}^N$ and covariates $\+X \in \R^{N \times M}$
\begin{enumerate}
\item Form the matrix of one-hot encoded responses $\+E_\+y = (\+e_{y_1}^T, \hdots, \+e_{y_N}^T)^T \in \R^{N \times K}$.
\item Take the \IB~model $\M_\IB$ which is the \textit{base} (see Sec.~\ref{sec:categorical_from_binary_models}) of the \CB~model, and use it to compute \textit{surrogate complete conditionals}: $\set{\log p_{\M_\IB}(u_v \cond \+u_{-v}, \+E_\+y)}$.
\item Take $\Q$ to be a mean-field family with factorization: $q(u_1, ..., u_V) = \prod_{v=1}^V q_v(u_v)$. (The regression weights $\IBmatrixB$ are included in this set, as are auxillary variables $\IBmatrixZ$ if augmentation is used.  If the \CB~likelihood is embedded within a more complicated graphical model, there may be other unobserved random variables as well.)
 \item Define the objective: $\L_{\M}(q) = \E_q [\log \frac{p_{\M_\IB}(\+E_\+y, \+u)}{q(\+u)}]$. 
 \item Optimize $\L_{\M}(q)$ using optimal coordinate ascent updates~\citep{blei2017variational}):
$q_v(u_v) \propto \exp \big\{ \E_{q_{-v}} \big[  \log p_{\M_\IB}(u_v \, | \, \+u_{-v}, \+E_\+y)\big] \big\}$.      If the complete conditional is an exponential family with natural parameter $\eta_v$,  so is its optimal update,  with natural parameter given by 
\begin{align*}
\nu_v = \E_{q_{-v}} [\eta_v(\+u_{-v}, \+E_\+y)]
\labelit \label{eqn:IB_CAVI_for_EFCC}
\end{align*}
 \end{enumerate}
 
The updates in Algorithm 1 will yield a density $q^*$ that is a local maximum of the $\ELBO$ of the surrogate model \citep{ormerod2010explaining},  and therefore a local maximum of a surrogate bound on the intended truly categorical model. For a \CB~model with a probit or logit link, a Gaussian prior on the regression weights, and use of appropriate augmentation, all conditionals enjoy closed-form updates in Eq.~\eqref{eqn:IB_CAVI_for_EFCC}, and the objective function $\L_\mathcal{M}$ is also available in closed-form, which is useful for convergence monitoring.   For details, see Secs.~\ref{sec:VI_for_CB_probit_models} and~\ref{sec:VI_for_CB_logit_models}.  


\section{Variational inference for \CB-Probit Models} \label{sec:VI_for_CB_probit_models}

Here we present closed-form variational inference for \CB-Probit models. The inference follows naturally from our IB-CAVI procedure in Algorithm 1.  

\subsection{Distributional preliminaries}

\subsubsection{Entropy facts about Multivariate Gaussian}

If $p, q$ are the densities of two different $d$-variate Gaussian distributions with parameters $\+\mu_p,  \+\Sigma_p$ and $\+\mu_q,  \+\Sigma_q$,  respectively,  then  the entropy is given by
\begin{align}
\H[q]= \half \log \bigg[ (2 \pi e)^d | \+\Sigma_q |\bigg]  
\label{eqn:entropy_mvns}	
\end{align}
 
The KL divergence is given by
{\scriptsize 
\begin{align*}
\KL{q}{p} &=  \frac{1}{2} \bigg[ \log \frac{|\+\Sigma_p|}{|\+\Sigma_q|} - d \\
& \quad \quad + (\+\mu_q - \+\mu_p)^T  \+\Sigma_p^{-1} (\+\mu_q - \+\mu_p) + \tr \big( \+\Sigma_p^{-1} \+\Sigma_q \big)  \bigg]  
\labelit \label{eqn:kl_divergence_mvns}
\end{align*}
}
The cross-entropy of two multivariate Gaussians can then be determined from \eqref{eqn:entropy_mvns} and \eqref{eqn:kl_divergence_mvns} via the relation
\begin{align*}
 \H[q,p] = \H[q] + \KL{q}{p} 
\labelit \label{eqn:cross_entropy_of_two_mvns} 
  \end{align*}

\subsubsection{Univariate normals truncated to positive or negative reals}
\label{sec:normal_plus_and_normal_minus}

The univariate truncated normal distribution $\TruncatedNormal(\mu,  \sigma^2,  \Upsilon)$ results when a normal distribution $\N(\mu,  \sigma^2)$ is truncated to some set $\Upsilon \subseteq \R$.    Note that the parameters $\mu, \sigma^2$ denote the mean and variance of the \textit{parent} normal distribution;  i.e.  if $X \sim \TruncatedNormal(\mu,  \sigma^2,  \Upsilon)$ then $\E[X] \neq \mu$ (unless $\Upsilon = \R$). 

If we assume that the truncation set is an interval $\Upsilon = (a,b)$ for $a,b \in \R$,  then the distribution $\TruncatedNormal(\mu,  \sigma^2,  (a,b))$ has p.d.f.

\begin{align*}
f(x ; \mu,  \sigma^2,  a,  b) = \df{\phi_{\mu, \sigma^2} (x) }{\Phi_{\mu, \sigma^2} (b)  - \Phi_{\mu, \sigma^2} (a) } \indicate{a \leq x \leq b}(x)
\end{align*} 

where $\phi_{\mu,  \sigma^2}$ and $\Phi_{\mu,  \sigma^2}$ denote the pdf and cdf,  respectively,  of a univariate normal distribution with mean $\mu$ and variance $\sigma^2$.   

We will work with distributions truncated to the positive or negative reals,  and so we define special notation:  $\N_+(\mu,  \sigma^2) := \TruncatedNormal(\mu,  \sigma^2,  [0, \infty))$ and $\N_-(\mu,  \sigma^2) := \TruncatedNormal(\mu,  \sigma^2,  (-\infty, 0))$.   In particular,  we will work with random variables of the form $T_+ \sim \N_+(\mu,  1)$ and $T_- \sim N_-(\mu, 1)$.    Based on this construction,  it is straightforward to derive
{\footnotesize
\begin{align*}
f_{T_+}(x) & = \df{\phi(x-\mu)}{1 - \Phi(-\mu)} \indicate{x \geq 0},  \quad 
f_{T_-}(x) = \df{\phi(x-\mu)}{\Phi(-\mu)} \indicate{x < 0} \\
\E[T_+] &= \mu + \df{\phi(-\mu)}{1 - \Phi(-\mu)},  \quad
\E[T_-] = \mu - \df{\phi(-\mu)}{\Phi(-\mu)} 
\labelit \label{eqn:mean_of_normal_plus_and_minus} \\
\Var[T_+] &= 1 - \mu ( \E[T_+]  - \mu) - (\E[T_+] - \mu)^2
\labelit \label{eqn:variance_of_normal_plus} \\
\Var[T_-] &=  1 - \mu ( \E[T_-]  - \mu) - (\E[T_-] - \mu)^2 
\labelit \label{eqn:variance_of_normal_minus} \\
\H[T_+] &= \ln \big(\sqrt{2 \pi e } \,  [1 - \Phi(-\mu)]  \big)  - \df{\mu \phi(-\mu)}{2 (1-\Phi(-\mu))}  
\labelit \label{eqn:entropy_of_normal_plus} \\
\H[T_-] &=   \ln \big(\sqrt{2 \pi e } \,  \Phi(-\mu) \big)  + \df{\mu \phi(-\mu)}{2 \Phi(-\mu)}  
\labelit \label{eqn:entropy_of_normal_minus} 
\end{align*}  
}where we use $\phi$ and $\Phi$ to refer to the pdf and cdf, respectively,  of the standard normal distribution, and where $\H[X] = - \int f(x) \ln f(x) \,  dx$ represents the differential entropy of a random variable $X$ with density $f$.

\begin{remark}{\remarktitle{Representation in terms of perturbation of parent mean}}
It is sometimes convenient to express the expectation of a truncated random variable as a perturbation of the expectation of its parent (pre-truncated) Gaussian random variable.    To this end,  for $T_s \in \set{T_+, T_-}$,  we write 
\begin{align*}
\E[T_s] &= \mu + \delta_s (\mu),  \quad 
\delta_s (\mu) : =
\begin{cases}
\df{\phi(-\mu)}{1 - \Phi(-\mu)} ,  & s = + \\
- \df{\phi(-\mu)}{\Phi(-\mu)} ,  & s = - \\
\end{cases} \\
\labelit \label{eqn:mean_of_normal_plus_or_minus_as_perturb}
\end{align*} 
which holds by Eq.~\eqref{eqn:mean_of_normal_plus_and_minus}.
\end{remark}

\begin{remark}{\remarktitle{Second moments}}
For $T_s \in \set{T_+, T_-}$,  we have
{\scriptsize 
\begin{align*}
\E[T_s^2] &= \Var[T_s] + \E^2[T_s] \\
& \stackrel{1}{=} 1 - \mu (\E[T_s] - \mu) - (E[T_s] - \mu)^2 + \E^2[T_s] \\
& = 1 - \mu E[T_s]  + \cancel{\mu^2} - \cancel{E^2[T_s]} + 2 \mu E[T_s] - \cancel{\mu^2} + \cancel{E^2[T_s] } \\
& = 1 + \mu E[T_s]
\labelit \label{eqn:second_moment_for_normal_plus_or_minus}
\end{align*}
}where (1) holds by \eqref{eqn:variance_of_normal_plus} and \eqref{eqn:variance_of_normal_minus}.
\end{remark}

\subsection{Models}
\subsubsection{\CB-Probit Model} \label{sec:CB_probit_model}

A Bayesian \CB-Probit model is a categorical GLM which generates smulti-class outcomes $y_i \in \set{1,...,K},  \,  i=1,...,N$ by
{\footnotesize 
	\begin{subequations}
	\begin{align}
	\+\beta_{k} & \iid \N(\+\mu_0,  \+\Sigma_0) \\ 
	 \quad p_{ik} & = \text{any \CB-Probit category probabilities} \label{eqn:choice_of_CFB_category_probs_in_the_CFB_Probit_model} \\
	y_{i} &\sim \text{Cat}(p_{i1}, \ldots p_{iK}).
	\end{align}\label{eqn:CB_probit_model}
	\end{subequations}
}The form for the category probabilities in Eq.~(\ref{eqn:choice_of_CFB_category_probs_in_the_CFB_Probit_model}) depends on the choice of \CB~model; for instance, for the \CBM-Probit and \CBC-Probit models we have 
\begin{align*}
p^{\CBM-\text{Probit}}_{ik} & =
		\frac{ \Phi( \+x_{i}^T \+\beta_{k} ) }
		{ \sum_{\ell=1}^K \Phi( \+x_{i}^T \+\beta_{\ell} ) } \\
p^{\CBC-\text{Probit}}_{ik} & =
		\frac{ \Phi( \+x_{i}^T \+\beta_{k} )  \prod_{h \neq k} \big(1-\Phi( \+x_{i}^T \+\beta_{h} )\big)  }
		{ \sum_{\ell=1}^K \Phi( \+x_{i}^T \+\beta_{\ell} ) \prod_{h \neq \ell} \big(1- \Phi( \+x_{i}^T \+\beta_{h} ) \big)} 	
\end{align*}
for standard Gaussian cdf $\Phi$,   known covariates $\+x_i \in \R^M$, and  unknown parameters $\+B \in \R^{M  \times K}$ ( $\+\beta_k$ is used to designate the $K$-th column of $\+B$).  \red{TODO: Notation here is different than in the main body of the paper due to c not y.}
 
 \subsubsection{\IB-Probit Model}
 \label{sec:IB_probit_base_model} 
 
 The base model for a \CB-Probit model is an \IB-Probit model.  With a Gaussian prior, the model is:
\begin{align*}
\+\beta_k &\iid \N(\+\mu_0,  \+\Sigma_0),  \quad k=1,...,K \\
\IBy_{ik} \cond \+\beta_k &\indsim \Bernoulli \big(  \Phi(\+x_i^T \+\beta_k) \big),  \quad i=1,...,N \\
\labelit \label{eqn:IB_probit}
\end{align*}
for known covariates $\+x_i \in \R^M$ and unknown parameters $\+B \in \R^{M  \times K}$ ( $\+\beta_k$ is used to designate the $K$-th column of $\+B$).   The binary responses $\IBy_{ik}$ are the $k$-th element of $\IBvecy_i \in \set{0,1}^K$, where $\IBvecy_i =\+e_{y_i}$ is the one-hot indicator vector with value of 1 only at entry $y_i \in \set{1,\hdots,K}$.


 \subsubsection{Augmented \IB-Probit Model} \label{sec:augmented_IB_probit_model}
 
 Following Albert \& Chib \cite{albert1993bayesian}, we may foster inference on the independent binary probit regression model by instead working with an augmented model. 

\begin{align*}
\+\beta_k &\iid \N(\+\mu_0,  \+\Sigma_0),  \quad k=1,...,K \\
z_{ik} \cond \+\beta_k &\indsim \N(\+x_i^T \+\beta_k,  1),  \quad i=1,\hdots,N  \\
\IBy_{ik} &= 
\begin{cases}
1 & z_{ik} \geq 0 \\
0 & \text{otherwise}, 
\end{cases}
 \quad i=1,\hdots,N  \\
\labelit \label{eqn:IB_probit_augmented}
\end{align*}
where we have introduced augmented variables $z_{ik}$. We use $\+Z \in \R^{N \times K}$ to represent the matrix whose $(i,k)$th entry is $z_{ik}$, and $\+z_k$ to represent the $k$th column of $\+Z$. As we will see in Sec.~\ref{sec:IB_probit_ccs}, the  augmented model is nice to work with, as it has exponential family complete conditionals.

\subsection{Variational inference}

Algorithm~\ref{alg:ib_cavi_for_cb_probit} provides closed-form coordinate ascent variational inference (CAVI) for the augmented \IB-Probit model. By Eq.~\eqref{eqn:surrogate_objective_for_CB_models}, this  gives closed-form CAVI for any \CB-Probit model. 


\renewcommand{\thealgocf}{2}
   
\newlength{\commentWidth}
\setlength{\commentWidth}{7.3cm}
\newcommand{\atcp}[1]{\tcp*[r]{\makebox[\commentWidth]{#1\hfill}}}

\begin{algorithm2e*}
\DontPrintSemicolon
\SetAlgoLined
 \SetKwInOut{Input}{Input}
 \SetKwInOut{HyperParams}{HyperParams}
 \SetKwInOut{Output}{Output}
 \SetKwRepeat{Do}{do}{while}

\begin{table}[H]
\vspace{-3.5mm}
\begin{tabular}{ll}
 \makecell[tl]{\textbf{Input:} \\ 
  $\+y \in \{1,\hdots,K\}^N$: $N$ responses from $K$ categories \\ 
    $\+X \in \R^{N \times M}$: Matrix whose $i$th row ($i=1,\hdots,N$) \\
        \quad gives the covariates associated with response $y_i$.  \\
    \textbf{Output:} \\ 
    $\set{(\qvecmu_{k},\qSigma_{k})}_{k=1}^K$: Parameters for $q(\+B) = \prod_k \mathcal{N}(\beta_k|\qvecmu_k, \qSigma_k)$, \\ \quad variational density (Eq.~\eqref{eqn:mf_variational_family_for_IB_probit})  on regression weights$^\dag$   \\  
    $\set{\set{\wt{\eta}_{ik}}_{k=1}^K}_{i=1}^N$: Parameters for $q(\+z) = \prod_i \prod_k \mathcal{T}\mathcal{N}(z_{ik} | \wt{\eta}_{ik}, 1, \Upsilon_{ik}) $, \\ \quad variational density (Eq.~\eqref{eqn:mf_variational_family_for_IB_probit}) on auxiliary variables$^\dag$   \\
\hfill \tinytext{$^\dag$: Here $\N$ and $\mathcal{T}\mathcal{N}$ refer to \textit{densities} rather than measures}
    }
    &
    \makecell[tl]{\textbf{Hyperparameters / Settings:} \\
     $(\+\mu_{0}, \+\Sigma_0)$ : Mean and covariance of prior on weights \\ \quad $\pi(\+B) = \prod_{k} \mathcal{N}(\+\beta_k | \+\mu_0, \+\Sigma_0)$ \\
     $\set{\qvecmu_{k}^{(0)}}_{k=1}^K$ : Initial variational mean \\ \quad on regression weights \\
    \textit{Termination condition} : e.g. number of iterations \\
    \quad or convergence threshold on $\ELBO_{\IB-\text{Probit}}$
    } 
\end{tabular}
\end{table}

\For{$k \gets 1$ \KwTo $K$}{
$\qSigma_k \gets \bigg( \+\Sigma_0^{-1} +  \+X^T \+X \bigg)^{-1}$\atcp{Set $\qSigma_k$ for $q( \+\beta_k | \cdot, \qSigma_k)$ via Eq.~\eqref{eqn:IB_cavi_update_for_covariance_of_beta_k_from_CB_Probit}}

\While{termination condition not satisfied} {

\For{$i \gets 1$ \KwTo $N$}{

	$\wt{\eta}_{ik} \gets \+x_i^T \qvecmu_k$  \atcp{Update $q(z_{ik} | \wt{\eta}_{ik} )$ via Eq.~\eqref{eqn:IB_cavi_update_for_expected_linear_predictor_from_CB_Probit}}

	$\E_{q}[ z_{ik} ] \gets \begin{cases}
\wt{\eta}_{ik} + \df{\phi{(-\wt{\eta}_{ik})}}{1 - \Phi{(-\wt{\eta}_{ik})}},& \IBy_{ik} = 1\\
\wt{\eta}_{ik} - \df{\phi{(-\wt{\eta}_{ik})}}{\Phi(-\wt{\eta}_{ik})}, & \text{otherwise}
\end{cases}$\atcp{Expectation computed via Eq.~\eqref{eqn:IB_CAVI_for_CB_Probit_taking_expectation_of_aux_var}}
}

$\qvecmu_k \gets \qSigma_k \left( \+\Sigma_0^{-1} \+\mu_0 +  \+X^T  \E_{q}[ \+z_{k}] \right)$\atcp{Update $q( \+\beta_k | \qmu_k, \cdot)$ via Eq.~\eqref{eqn:IB_cavi_update_for_mean_of_beta_k_from_CB_Probit}}
 (Optional) Compute $\ELBO_{\IB-\text{Probit}}$ via \eqref{eqn:elbo_high_level_IB_probit}
 }
 }
 \caption{Independent Binary Coordinate Ascent Variational Inference (\IB-CAVI) for \CB-Probit models}\label{alg:ib_cavi_for_cb_probit}
\end{algorithm2e*}

\subsubsection{Complete conditionals}\label{sec:IB_probit_ccs}

The augmented \IB-Probit model (Eq.~\eqref{eqn:IB_probit_augmented}) contains Bayesian linear regression on the auxiliary variables $\+z_k$.   In this way,  we obtain the complete conditionals
{\footnotesize
\begin{align*}
z_{ik} \cond \+\beta_{1}, ..., \+\beta_{K},  \IBy_{ik} & \sim 
\begin{cases}
\N_+ \bigg(  \+x_{i}^T \+\beta_{k} \; , \; 1 \bigg), & \IBy_{ik} = 1 \\
\N_- \bigg( \+x_{i}^T \+\beta_{k} \; ,  \; 1 \bigg),  & \text{otherwise} 
\end{cases}  
\labelit \label{eqn:complete_conditions_for_IB_probit_with_augmentation} 
\end{align*}
}where $\N_+$ and $\N_-$ are truncated normal distributions defined in Section \ref{sec:normal_plus_and_normal_minus}, and
{\footnotesize
\begin{subequations}
	\begin{align*}
\+\beta_{k} \cond \+z_{k}  &\sim \N (\+\mu_{k},  \+\Sigma_{k}), \\
\+\mu_{k}  &= \+\Sigma_{k} \bigg( \+\Sigma_0^{-1} \+\mu_0 + \+X^T \+z_{k} \bigg),  \quad \\
\+\Sigma_{k} &= \bigg( \+\Sigma_0^{-1} + \+X^T \+X \bigg)^{-1}
\end{align*}\label{eqn:cc_for_beta_in_IB_Probit_with_augmentation} 
\end{subequations}
}

\subsubsection{Variational family}
We take the mean-field variational family for the  augmented \IB-Probit model  (Eq.~\eqref{eqn:IB_probit_augmented}) to have density given by
{\footnotesize 
	\begin{align*}
	& q(\+B, \+Z) \stackrel{(1)}{=}   q(\+B) q(\+Z) \\
	& \stackrel{(2)}{=}  \ds\prod_{k=1}^K  \explaintermbrace{$\N(\qvecmu_{k},  \qSigma_{k})$\; }{q(\+\beta_{k})}  \ds\prod_{i=1}^{N} \explaintermbrace{\; $\TruncatedNormal (\wt{\eta}_{ik},  1,  \Upsilon_{ik})$}{q(z_{ik})} 
	\labelit \label{eqn:mf_variational_family_for_IB_probit}
	\\ & \text{where} \; \Upsilon_{ik} = 
	\begin{cases}
	\R^+, & \IBy_{ik} = 1 \\
	\R^-, & \IBy_{ik} = 0 \\
	\end{cases}
	\labelit \label{eqn:trunation_region_for_aux_var_in_hierarchical_CFB_probit}
	\end{align*}
}Equality (1) is by mean-field assumption.  Equality (2) holds without any additional assumption. Since the complete conditionals are both exponential families, the optimal variational factors with respect to the lower bound $\ELBO_\IB$ are in the same exponential families,  with natural parameters given by the variational expectation of the natural parameters of the corresponding complete conditionals (as in Eq.~\ref{eqn:IB_CAVI_for_EFCC}).

\subsubsection{Coordinate ascent updates}

 Here we derive the parameters for the updates,  using the notation of Eq.~\eqref{eqn:mf_variational_family_for_IB_probit}.


\paragraph{Updates to $ \{q(\+\beta_{k}) \}_{k=1}^K$} Since the natural parameters of a multivariate Gaussian are the precision and precision-weighted mean,  we reparametrize the surrogate complete conditional for each $\+\beta_{k}$ in Eq.~\eqref{eqn:cc_for_beta_in_IB_Probit_with_augmentation}  before taking variational expectations of the parameters.   Hence,   the optimal update to each $q(\+\beta_{k} \cond \qvecmu_{k},  \qSigma_{k})$ with respect to the objective $\ELBO_\IB$ is given by
{\footnotesize
 \begin{align*}
\qSigma_{k}^{-1} &=  \E_{q_{-\+\beta_{k}}} \bigg[  \+\Sigma_0^{-1} + \+X^T \+X \bigg]   =  \+\Sigma_0^{-1} +  \+X^T \+X \\
\qSigma_{k}^{-1} \qvecmu_{k} &=  \E_{q_{-\+\beta_{k}}} \bigg[  \+\Sigma_0^{-1} \+\mu_0 + \+X^T  \+z_{k} \bigg]  \\
&=  \+\Sigma_0^{-1} \+\mu_0 +  \+X^T  \E_{q_{\+z_{k}}}[ \+z_{k}]  
 \end{align*}
 }where $\E_q[\+z_{k}]$ is given explicitly below.   

Thus,  in standard parameterization, we update
{\footnotesize
\begin{subequations}
 \begin{align}
 \qvecmu_{k} &= \qSigma_{k} \bigg(  \+\Sigma_0^{-1} \+\mu_0 +  \+X^T  \E_{q_{\+z_{k}}}[ \+z_{k}]  \bigg)  \label{eqn:IB_cavi_update_for_mean_of_beta_k_from_CB_Probit} \\
\qSigma_{k} &=  \bigg( \+\Sigma_0^{-1} +  \+X^T \+X \bigg)^{-1} \label{eqn:IB_cavi_update_for_covariance_of_beta_k_from_CB_Probit}
 \end{align}
\label{eqn:IB_cavi_update_for_beta_k_from_CB_probit}
 \end{subequations}
}where $\E_q[\+z_{k}]   \in \R^{N}$ has $i$-th entry given by 
{\footnotesize
\begin{align*}
\E_q[z_{ik}] = 
\begin{cases}
\wt{\eta}_{ik} + \df{\phi{(-\wt{\eta}_{ik})}}{1 - \Phi{(-\wt{\eta}_{ik})}},& \IBy_{ik} = 1\\
\wt{\eta}_{ik} - \df{\phi{(-\wt{\eta}_{ik})}}{\Phi(-\wt{\eta}_{ik})}, & \text{otherwise} 
\end{cases} 
\labelit \label{eqn:IB_CAVI_for_CB_Probit_taking_expectation_of_aux_var}
\end{align*} 
}by properties of the truncated normal distribution (Section \ref{sec:normal_plus_and_normal_minus}).   Recall that $\phi$ and $\Phi$ refer to the pdf and cdf,  respectively,  of the standard normal. 

\paragraph{Updates to $ \set{q(z_{ik})}_{ik}$}

In \eqref{eqn:complete_conditions_for_IB_probit_with_augmentation},  we saw that the surrogate complete conditional for each $z_{ik}$ has the form $\TruncatedNormal(\eta_{ik}, 1, \Upsilon_{ik})$,  where $\Upsilon_{ik}$ is defined as in \eqref{eqn:trunation_region_for_aux_var_in_hierarchical_CFB_probit}. But since each such distribution is in the exponential family with natural parameter $\eta_{ik}$,  the optimal update for each $q(z_{ik})$ is given by
{\footnotesize
\begin{align*}
\wt{\eta}_{ik} =\E[\+x_{i}^T \+\beta_{k}] = \+x_{i}^T \qvecmu_{k} 
\labelit \label{eqn:IB_cavi_update_for_expected_linear_predictor_from_CB_Probit} 
\end{align*} 
}

\subsubsection{The evidence lower bound}

We provide the evidence lower bound for the \IB-Probit model, $\ELBO_\IB$, in the case of independent $\N(\+\mu_0 = \+0,  \+\Sigma_0 = \+I)$ priors on $\+\beta_k$ for $k=1,...,K$.  Using $\IBmatrixY \in \set{0,1}^{N \times K}$ to represent the matrix whose $i$th row is the one-hot encoded vector $\IBvecy_i = \+e(y_i)$, we have 
\begin{align*}
&\ELBO(q)  = \explaintermbrace{energy}{\E_q [\log p(\IBmatrixY,  \+Z, \+\beta)]} + \explaintermbrace{entropy}{- \E_q [\log q (\+X,  \+B )]}\\
&= \ds\sum_{k=1}^K  \bigg[ \explaintermbrace{(A)}{\ds\sum_{i=1}^N \E_q [\log p(\IBy_{ik},  z_{ik} \cond \+\beta_k) ]} \; + \; \explaintermbrace{(B)}{\E_q [ \log p(\+\beta_k)]} \; + \\
&\; \explaintermbrace{(C)}{-\ds\sum_{i=1}^N \E_q [\log q(z_{ik})]} \; + \; \explaintermbrace{(D)}{ - \E_q \log [q(\+\beta_k)]} \bigg]
\labelit \label{eqn:elbo_high_level_IB_probit}
\end{align*}

Term (A) is given by a sum whose $i$th summand is
{\scriptsize
\[
\begin{array}{l}
 \E_q [\log p(\IBy_{ik},  z_{ik} \cond \+\beta_k) ]\\
 =  \E_q \bigg [\log \bigg\{  \df{1}{\sqrt{2 \pi}} \exp \bigg( - \half (z_{ik} - \+x_i^T \+\beta_k)^2\bigg)  \bigg( \indicate{z_{ik} < 0}^{\indicate{\IBy_{ik}=0}}  \indicate{z_{ik} \geq 0}^{\indicate{\IBy_{ik}=1}} \bigg) \bigg\} \bigg] \\
 = -\half \log 2\pi - \half \E (z_{ik} - \+x_i^T \+\beta_k)^2 \\
 \quad \quad \quad \quad \quad - \cancel{\half \E_q \bigg[ \indicate{\IBy_{ik}=0} \log  \indicate{z_{ik} < 0} + \indicate{\IBy_{ik}=1} \log  \indicate{z_{ik} \geq 0}  \bigg]} \\
 \\
 =  -\half \log 2\pi - \half \E_q [z_{ik}^2] + \E_q[z_{ik}] \+x_i^T \E_q[\+\beta_k] - \half \E_q [(\+x_i^T \+\beta_k)^2]  \\
   \stackrel{1}{=}  -\half ( \log 2\pi +1 ) - \half \wt{\eta}_{ik} \E_q [z_{ik}] + \E_q[z_{ik}] \wt{\eta}_{ik} - \half \E_q [(\+x_i^T \+\beta_k)^2]  \\
  \stackrel{2}{=}  -\half ( \log 2\pi +1 ) + \half \E_q[z_{ik}] \wt{\eta}_{ik}   - \half \bigg(\+x_i^T \wt{\+\Sigma_k} \+x_i + \wt{\eta}_{ik}^2  \bigg)  \\
 \stackrel{3}{=}  -\half ( \log 2\pi +1 ) +  \half \wt{\eta}_{ik} \delta_{\IBy_{ik}} (\wt{\eta}_{ik}) -\half \+x_i^T \wt{\+\Sigma_k} \+x_i 
 \end{array} \]
where 
\begin{align*}
\delta_{\IBy_{ik}} (\wt{\eta}_{ik}) :& =
\begin{cases}
\df{\phi(-\wt{\eta}_{ik})}{1 - \Phi(-\wt{\eta}_{ik})} ,  & \IBy_{ik} = 1\\
- \df{\phi(-\wt{\eta}_{ik})}{\Phi(-\wt{\eta}_{ik})} ,  & \IBy_{ik} = 0
\end{cases}
\end{align*}
}Equation (1) holds by Eq.~\eqref{eqn:IB_cavi_update_for_expected_linear_predictor_from_CB_Probit}  and by Eq.~\eqref{eqn:second_moment_for_normal_plus_or_minus}, (2) holds by applying the second moment decomposition $\E_q[W^2] = \Var_q[W] + \E_q^2[W]$ in the case where $W= \+x_i^T \+\beta_k$,  and (3) holds by applying Eq.~\eqref{eqn:mean_of_normal_plus_or_minus_as_perturb} to express $\E_q[z_{ik}]$ as a perturbation of $\wt{\eta}_{ik}$.  Recall that $\phi$ and $\Phi$ are the pdf and cdf,  respectively,  of the standard normal distribution.


Term (B) is the negative cross-entropy of two Gaussians.  With a $\N(\+0, \+I)$ prior, we have 
\begin{align*}
\E_q [\log p(\+\beta_k) ] &= -\df{M}{2} \log (2\pi) - \half \E_q [\+\beta_k^T \+\beta_k]  \\
& \stackrel{1}{=}  -\df{M}{2} \log (2\pi)  -\half \big( \trace (\qSigma_k) + \qvecmu_k^T \qvecmu_k \big)
\end{align*}
where (1) holds since $\E_q [\+\beta_k^T \+\beta_k] = \ds\sum_{m=1}^M \E_q [\beta_{km}^2] =\ds\sum_{m=1}^M  \Var_q[\beta_{km}] + \E_q^2[\beta_{km}]$.

Term (C) is the sum of entropies of truncated normal distributions,  where the $i$-th element in the sum is

\begin{itemize}
 \item  the entropy of $\N_+(\+x_i^T \wt{\+\mu}_k,  1)$ when $\IBy_{ik}=1$,   in which case the entropy is given by Eq.~\eqref{eqn:entropy_of_normal_plus} 
 \item the entropy of  $\N_-(\+x_i^T \wt{\+\mu}_k,  1)$ when $\IBy_{ik}=0$,  in which case the entropy is given by  Eq.~\eqref{eqn:entropy_of_normal_minus}.  
 \end{itemize}

Term (D) is the entropy of the multivariate Gaussian $\N(\qvecmu_k, \qSigma_k)$, which is given by 
\[  - \E_q \log [q(\+\beta_k)] = \half \ln | \qSigma_k | + \df{M}{2} ( 1 + \ln 2\pi) \]

\subsection{Computational complexity}  \label{sec:computational_complexity_for_IB_probit}
The inference requires a one-time up-front computation of complexity $\mathcal{O}(M^3 + M^2N)$ due to the inversion in step Eq.~\eqref{eqn:IB_cavi_update_for_covariance_of_beta_k_from_CB_Probit}. Note that often the matrix $\+X$ will be sparse, in which case the complexity of this up-front step can be reduced.  Note that this up-front inversion could be avoided (at the cost of losing information about correlations across the $M$ covariates) by tweaking the variational family in Eq.~\eqref{eqn:mf_variational_family_for_IB_probit} to make a stronger (\textit{fully-mean field}) variational assumption  $q(\+B) = \prod_{m=1}^M  \prod_{k=1}^K  q(\beta_{mk})$, where each  $q(\beta_{mk})$ is the density of a univariate Gaussian $\N(\wt{\mu}_{mk},  \wt{\sigma}^2_{mk})$.  Inference would proceed identically as before, except that the variational covariance for the $k$th category $\wt{\+\Sigma}_k$ would become a \textit{diagonal} $M \times M$ covariance matrix whose $m$th entry is given by $\wt{\sigma}^2_{mk}= [ (\+\Sigma_0)_{mk} + \sum_{i=1}^N x_{im}^2]^{-1}$. With this simplification, the up-front computation would have complexity $\mathcal{O}(MNK)$.

Afterwards, the computational complexity for a single CAVI update is $\mathcal{O}(MNK)$, where $M$ is the number of covariates, $N$ is the number of samples, and $K$ is the number of categories. 
The complexity for each substep of a single CAVI update is given in the Table~\ref{tab:computational_complexity_of_IB_CAVI_for_Probit}	.   

\begin{table}[h!]
{\scriptsize
\begin{tabular}{r|l|l|l}
\textbf{Variable}&  \textbf{Step} & \textbf{Per-iteration} & \textbf{Note} \\
&& \textbf{complexity} & \\
\toprule 
$\+B$ & covariance  & pre-multiplied & $\wt{\+\Sigma}_k = \explainterm{MxM}{\+\Sigma}$ is the same for all $k$ \\
&\eqref{eqn:IB_cavi_update_for_covariance_of_beta_k_from_CB_Probit}  && and constant over iterations.\\
$\+B$ & mean  & $\mathcal{O}(MNK)$ & $\explainterm{MxK}{\wt{\+\mu}} =  \explainterm{MxM}{\+\Sigma}  \explainterm{MxN}{\+X^T} \, \explainterm{NxK}{\E_q[\+Z]}$ \\
&\eqref{eqn:IB_cavi_update_for_mean_of_beta_k_from_CB_Probit}&& the first two terms can be \\
&&& pre-multiplied  \\
$\+Z$ & \eqref{eqn:IB_cavi_update_for_expected_linear_predictor_from_CB_Probit} & $\mathcal{O}(MNK)$ & $\explainterm{NxK}{\wt{\+\eta}} = \explainterm{NxM}{\+X}\explainterm{MxK}{\wt{\+\mu}}$  \\
$\+Z$  &\eqref{eqn:IB_CAVI_for_CB_Probit_taking_expectation_of_aux_var} & $\mathcal{O}(NK)$ & We suppress the complexity of \\
&&&evaluating the \\
&&&Gaussian cdf. 
\end{tabular}
}
\caption{\textit{The computational complexity of CAVI updates for \IB-Probit.} $\+B$ is the matrix of regression weights and $\+Z$ are auxiliary variables added for conditional conjugacy.}
\label{tab:computational_complexity_of_IB_CAVI_for_Probit}	
\end{table}

Moreover, the entire inference procedure (across all iterations) is embarassingly parallel over the $K$ categories, so distributed computation can reduce the complexity for the all CAVI steps to $\mathcal{O}(MNI)$, where $I$ is the number of CAVI iterations.

\subsection{Sparsity considerations} \label{sec:sparsity_considerations_for_IB_probit}
When $N \times K$ is large, the matrix $\wt{\+\eta}$ may not fit into memory.  However, when the design matrix $\+X$ is sparse \red{(TODO: add interpretation about what would cause the expected value of beta to stay at 0)}, $\wt{\+\eta}$ may be highly sparse ; indeed, $\wt{\eta}_{ik}=0$ whenever at least one of $\set{\wt{\+\mu}_{mk}, \+X_{im}}$ is 0 for all $m=1,...,M$.    In this setting, we can represent $\E_q[\+Z]$ efficiently, since we can see from Eq.~\eqref{eqn:IB_CAVI_for_CB_Probit_taking_expectation_of_aux_var} that only two values are possible when $\eta_{ik}=0$.
\begin{align}
\E_q[z_{ik}] = 
\begin{cases}
2 \phi(0),& \eta_{ik}=0, y_i = k\\
-2 \phi(0),& \eta_{ik}=0, y_i \neq k
\end{cases} 
\label{eqn:cavi_update_expected_aux_variable_under_sparsity}
\end{align}

Define
\begin{align*}
\E_q[\+Z]^* &: (\E_q[\+Z]^*)_{ik} = 
	\begin{cases}	
(\E_q[\+Z])_{ik}, & \eta_{ik} \neq 0 \\
	0, & \eta_{ik} =0 \\
	\end{cases} \\
\E_q[\+Z]^\dagger &: (\E_q[\+Z]^\dagger)_{ik} = 
	\begin{cases}	
1, & \eta_{ik} = 0, y_i =k \\
	0, & \text{otherwise}\\
	\end{cases} \\
\E_q[\+Z]^\ddagger &: (\E_q[\+Z]^\ddagger)_{ik} = 
	\begin{cases}	
1, & \eta_{ik} = 0, y_i \neq k \\
	0, & \text{otherwise}\\
	\end{cases} 
\end{align*}

Then we can avoid representing $\E_q[\+Z]$ as a large dense $N \times K$ matrix of floats by rewriting Eq.~\eqref{eqn:IB_CAVI_for_CB_Probit_taking_expectation_of_aux_var} in matrix form as  
\[ E_q[\+Z] =  \E_q[\+Z]^* + 2 \phi(0) \bigg( \E_q[\+Z]^\dagger - \E_q[\+Z]^\ddagger \bigg) \]

\section{Variational inference for \CB-Logit Models} \label{sec:VI_for_CB_logit_models}

Here we present closed-form variational inference for \CB-Logit models. The inference follows naturally from our IB-CAVI procedure in Algorithm 1.  

\subsection{Distributional preliminaries}

\begin{definition}
A non-negative random variable $X$ has a \textit{\polyaGamma distribution} \cite{polson2013bayesian} with parameters $b>0$ and $c \in \R$, denoted as $X \sim \PG(b,c)$, if 
\[X \stackrel{D}{=} \df{1}{2 \pi^2} \ds\sum_{r=1}^\infty \df{\gamma_r}{(r- 1/2)^2 + c^2/(4\pi^2)} \]
where the $\gamma_r \sim \GammaDist(b,1)$ are independent Gamma random variables,  and where $\stackrel{D}{=}$ indicates equality in distribution. 
\end{definition}

The density of a $\PG(b,c)$ random variable can be written as \cite{polson2013bayesian}:
\begin{align*}
f(x; b, c) =  \cosh^b(c/2) e^{-\frac{c^2}{2} x} g(x; b,  0)
\labelit \label{eqn:density_of_polya_gamma_with_parameters_b_and_c}
\end{align*}
where $g(x; b, 0)$ is the density of a $\PG(b,0)$ random variable
\[g(x; b,  0)  = \df{2^{b-1}}{\Gamma(b)} \ds\sum_{n=0}^\infty  (-1)^n \df{\Gamma(n+b)}{\Gamma(n+1)}  \df{(2n+b)}{\sqrt{2 \pi x^3}} e^{-\frac{(2n+b)^2}{8x}}\]

So $f$ is constructed from $g$ via an \textit{exponential tilt} and a renormalization. 

The expectation of a \polyaGamma random variable is given by \cite{polson2013bayesian} as
\begin{align*}
\E [X] = \df{b}{2c} \text{tanh}(c/2) = \df{b}{2c} \df{e^c -1}{e^c + 1}
\end{align*}

\subsection{Models}
\subsubsection{\CB-Logit Model} \label{sec:CB_logit_model}

A Bayesian \CB-Logit model is a categorical GLM which generates  multi-class outcomes $y_i \in \set{1,...,K},  \,  i=1,...,N$ by
{ \footnotesize 
	\begin{subequations}
	\begin{align}
	\+\beta_{k} & \iid \N(\+\mu_0,  \+\Sigma_0) \\ 
	 \quad p_{ik} & = \text{any \CB-Logit category probabilities} \label{eqn:choice_of_CFB_category_probs_in_the_CFB_Logit_model} \\
	y_{i} &\sim \text{Cat}(p_{i1}, \ldots p_{iK}).
	\end{align}\label{eqn:CB_probit_model}
	\end{subequations}
}The form for the category probabilities in Eq.~(\ref{eqn:choice_of_CFB_category_probs_in_the_CFB_Logit_model}) depends on the choice of \CB~model; for instance, for the \CBM-Logit and \CBC-Logit models we have 
\begin{align*}
p^{\CBM-\text{Logit}}_{ik} & =
		\frac{ \LogisticCDF( \+x_{i}^T \+\beta_{k} ) }
		{ \sum_{\ell=1}^K \LogisticCDF( \+x_{i}^T \+\beta_{\ell} ) } \\
p^{\CBC-\text{Logit}}_{ik} & =
		\frac{\LogisticCDF( \+x_{i}^T \+\beta_{k} )  \prod_{h \neq k} \big(1-\LogisticCDF( \+x_{i}^T \+\beta_{h} )\big)  }
		{ \sum_{\ell=1}^K \LogisticCDF( \+x_{i}^T \+\beta_{\ell} ) \prod_{h \neq \ell} \big(1- \LogisticCDF( \+x_{i}^T \+\beta_{h} ) \big)} 	
\end{align*}
for  standard Logistic cdf $\LogisticCDF$,   known covariates $\+x_i \in \R^M$, and  unknown parameters $\+B \in \R^{M  \times K}$ ( $\+\beta_k$ is used to designate the $k$-th column of $\+B$).
  
 \subsubsection{\IB-Logit Model}
 \label{sec:IB_logit_base_model} 
 
 The base model for a \CB-Logit model is an \IB-Logit model.  With a Gaussian prior, the model is:
\begin{align*}
\+\beta_k &\iid \N(\+\mu_0,  \+\Sigma_0),  \quad k=1,...,K \\
\IBy_{ik} \cond \+\beta_k &\indsim \Bernoulli \big(  \LogisticCDF(\+x_i^T \+\beta_k) \big),  \quad i=1,...,N \\
\labelit \label{eqn:IB_logit}
\end{align*}
for known binary responses $\IBy_{ik}$,  known covariates $\+x_i \in \R^M$ and unknown parameters $\+B \in \R^{M  \times K}$ ( $\+\beta_k$ is used to designate the $K$-th column of $\+B$).  We write $\IBmatrixY \in \set{0,1}^{N \times K}$ to represent the matrix with one-hot encoded rows such that $\IBmatrixY_{i,k} = 1$ if the $i$th outcome was the $k$th category (i.e. if $y_i=k$), and $\IBvecy_k$ to represent the $k$th column of $\IBmatrixY$.

 \subsubsection{Augmented \IB-Logit Model} \label{sec:augmented_IB_probit_model}
 The main idea  is to introduce auxiliary latent variables $\+\omega_k = (\omega_{1k},  ...,  \omega_{Nk})$ with Polya-Gamma distribution to make the model of Eq.~\eqref{eqn:IB_logit} fully conditionally conjugate.       The model is fully conditionally conjugate in the sense that the complete conditionals and the priors form conjugate pairs;  that is $p(\+\beta_k \cond \+w_k,  \IBvecy_k)$ is in the same family  (Gaussian) as $p(\+\beta_k)$,  and each $p(w_{ik} \cond \+\beta_k,  \IBvecy_k)$ is in the same family $(\PG)$ as $p(w_{ik})$.  Thus,  inference on the augmented model is easy.    Marginalizing over these auxillary variables in the posterior distribution yields the desired target posterior on $\+B = (\+\beta_1, ..., \+\beta_K)$.  We use $\+\Omega \in \R^{N \times K}$ to represent the matrix whose $k$th column is $\+\omega_k$.  

We now form the augmented model.  Conditional on each $\+\beta_k$,  we take $\set{(\IBy_{ik},  \omega_{ik})}_{i=1}^N$ to be independent random pairs such that $\IBy_{ik}$ and $\omega_{ik}$ are also independent,  where
\begin{align*}
\+\beta_k &\iid \N(\+\mu_0, \+\Sigma_0 ), \quad k=1,...,K \\
\IBy_{ik} \cond \+\beta_k &\indsim \Bernoulli \bp{ \frac{\exp \{\+x_i^T \+\beta_k \}}{1+\exp \{\+x_i^T \+\beta_k \}}},  \quad i=1,...,N \\
\omega_{ik} \cond \+\beta &\indsim  \PG(1,  \+x_i^T \+\beta_k)\\
\labelit \label{eqn:IB_model_with_pga}
\end{align*}
The \textit{augmented posterior density} for the $k$th binary logistic regression is given by 
\[ p(\+\beta_k, \+\omega_k \cond \IBvecy_k) \propto  \bb{\ds\prod_{i=1}^N  p(\IBy_{ik} \cond \+\beta_k) p(\omega_{ik} \cond \+\beta_k)} p(\+\beta_k)\]
And clearly
\[  \ds\int_{\R_+^N}   p(\+\beta_k, \+\omega_k \cond \IBvecy_k) d\+\omega_k = p(\+\beta_k \cond \IBvecy_k) \]
which is the target posterior density for the $k$ binary logistic regression.

A straightforward argument (see Section \ref{sec:derivation_complete_conditionals_bayes_binom_reg_with_pga}) reveals that the complete conditionals for the $k$th binary logistic regression are given by
\begin{subequations}
\begin{align*}
(\omega_{ik} \cond \+\beta_k) &\sim \PG(1,  \+x_i^T \+\beta_k) 
 \labelit \label{eqn:gibbs_update_to_omega_i_under_pg_augmentation}\\
(\+\beta_k \cond \IBvecy_k,  \+\omega_k) &\sim \N(  \+\mu_{\+\omega_k},   \+\Sigma_{\+\omega_k}) 
 \labelit \label{eqn:gibbs_update_to_beta_under_pg_augmentation}
\end{align*}
\label{eqn:gibbs_sampler_bayesian_logistic_regression_with_pg_augmentation}
\end{subequations}
where
\begin{align*}
\+\Sigma_{\+\omega_k} &= \bp{\+X^T \+W_{\+\omega_k} \+X + \+\Sigma_0^{-1} }^{-1} 
 \labelit \label{eqn:gibbs_sampler_update_Sigma} \\
\+\mu_{\+\omega_k} &= \+\Sigma_{\+\omega} \bp{ \+X^T \+\kappa_k + \+\Sigma_0^{-1} \+\mu_0} 
 \labelit \label{eqn:gibbs_sampler_update_mu}
 \end{align*}
and
\begin{subequations}
\begin{align}
\+\kappa_k  &= (\IBy_{1k} - \half, ..., \IBy_{Nk} - \half)^T  \label{eqn:def_of_vec_kappa_k}\\
\+W_{\+\omega_k} & \text{\,  is the diagonal matrix of  $\omega_{ik}$'s} 
\end{align}
\label{eqn:def_of_kappa_and_Omega}
\end{subequations}

Another straightforward argument (see Section 2 of \cite{choi2013polya}) reveals that the complete conditionals \eqref{eqn:gibbs_sampler_bayesian_logistic_regression_with_pg_augmentation} for the Bayesian logistic regression model under \pga form a valid Gibbs sampler.

\subsection{Variational inference}
We are able to easily construct a mean-field variational inference algorithm using these complete conditionals \eqref{eqn:gibbs_sampler_bayesian_logistic_regression_with_pg_augmentation},  since the complete conditionals are in the exponential family.   Algorithm~\ref{alg:ib_cavi_for_cb_logit} provides closed-form coordinate ascent variational inference (CAVI) for the augmented \IB-Logit model. By Eq.~\eqref{eqn:surrogate_objective_for_CB_models}, this gives closed-form CAVI for any \CB-Logit model. 

\renewcommand{\thealgocf}{3}
   
\begin{algorithm2e*}
\SetAlgoLined
 \SetKwInOut{Input}{Input}
 \SetKwInOut{HyperParams}{HyperParams}
 \SetKwInOut{Output}{Output}
 \SetKwRepeat{Do}{do}{while}

\begin{table}[H]
\vspace{-3.5mm}
\begin{tabular}{ll}
 \makecell[tl]{\textbf{Input:} \\ 
  $\+y \in \{1,\hdots,K\}^N$: A vector of $N$ conditionally independent  \\     \quad responses from $K$ categories \\ 
    $\+X \in \R^{N \times M}$: A matrix whose $i$th row ($i=1,\hdots,N$) \\
        \quad gives the covariates associated with response $y_i$.  \\
    \textbf{Output:} \\  
    $\set{(\qvecmu_{k},\qSigma_{k})}_{k=1}^K$: Parameters for $q(\+B) = \prod_{k} \N (\+\beta_k \cond \qvecmu_k, \qSigma_k)$,  \\ \quad variational density (Eq.~\eqref{eqn:variational_family_ib_logit})  on regression weights$^\dag$  \\  
    $\set{\set{\wt{c}_{ik}}_{k=1}^K}_{i=1}^N$: Parameters for $q(\+\Omega) = \prod_i \prod_{k}  \text{PG} (\omega_{ik} \cond 1, \qcik)$, \\ \quad variational density  (Eq.~\eqref{eqn:variational_family_ib_logit}) on auxiliary variables$^\dag$  \\
\hfill \tinytext{$^\dag$: Here $\N$ and $\text{PG}$ refer to \textit{densities} rather than measures.}
    } 

    &
    \makecell[tl]{\textbf{Hyperparameters / Settings:} \\
     $(\+\mu_{0}, \+\Sigma_0)$ : Mean and covariance of prior on weights \\ \quad $\pi(\+B) = \prod_{k} \mathcal{N}(\+\beta_k | \+\mu_0, \+\Sigma_0)$ \\
     $\set{(\qvecmu_{k}^{(0)}, \qSigma_{k}^{(0)})}_{k=1}^K$ : Initial variational parameters \\ \quad on regression weights \\
     \texttt{Termination condition} : e.g. number of iterations \\
    \quad or convergence threshold on $\ELBO_{\IB-\text{Logit}}$
    } 
\end{tabular}
\end{table}

\For{$k \gets 1$ \KwTo $K$}{
$\+\kappa_k  \gets (\IBy_{1k} - \half, ..., \IBy_{Nk} - \half)^T$ \atcp{where $\IBy_{ik}:=1$ iff $y_{i}=k$ (Sec.~\ref{sec:IB_logit_base_model})} 
\While{termination condition not satisfied} {

\For{$i \gets 1$ \KwTo $N$}{

	$\qcik \gets  \bp{ \+x_i^T \qSigma_k \+x_i + ( \+x_i^T \qvecmu_k)^2}^{1/2}$    \atcp{Update $q (\omega_{ik} \cond 1, \qcik)$ via Eq.~\eqref{eqn:IB_cavi_updates_for_c_ik_from_CB_Logit}}

	$\E_{q}[\omega_{ik}] \gets  \df{1}{2\qcik} \df{e^{\qcik} -1}{e^{\qcik} + 1}$\atcp{Expectation computed via Eq.~\eqref{eqn:diagonal_matrix_of_expected_pga_variables_for_kth_category}}
}
$\+W_k \gets$  diag. matrix from  $(\E_{q}[\omega_{ik}])_{i=1}^N$; \\
$ \qSigma_k \gets \bp{\+X^T \+W_k \+X + \+\Sigma_0^{-1} }^{-1} $\atcp{Update $q( \+\beta_k | \cdot, \qSigma_k)$ via Eq.~\eqref{eqn:IB_cavi_updates_for_beta_k_covariance_from_CB_Logit}}
$\qvecmu_k \gets \qSigma_k  \bp{ \+X^T \+\kappa_k + \+\Sigma_0^{-1} \+\mu_0} $\atcp{Update $q( \+\beta_k | \qvecmu_k, \cdot)$ via Eq.~\eqref{eqn:IB_cavi_updates_for_beta_k_mean_from_CB_Logit} }
 (Optional) Compute  $\ELBO_{\IB-\text{Logit}}$ via \eqref{eqn:elbo_for_IB_Logit_with_pga}
 }
 }
 

 
 \caption{Independent Binary Coordinate Ascent Variational Inference (\IB-CAVI) for \CB-Logit models}\label{alg:ib_cavi_for_cb_logit}
\end{algorithm2e*}

\vspace{1mm}


\begin{proposition}
\label{prop:vi_for_bayes_logreg_with_pga}
An optimal mean-field coordinate ascent variational inference algorithm for estimating the posterior of the  \IB-Logit model with \pga (Eq.~\eqref{eqn:IB_model_with_pga}) by using a member of the variational family whose density factorizes as 
\begin{align}
 q(\+\Omega, \+B ) =  q(\+\Omega)  q(\+B )
\label{eqn:mf_family_for_IB_Logit_with_polya_gamma_augmentation}
\end{align}
can be obtained by taking the variational family to have the further factorization  
\begin{align}
 q(\+\Omega, \+B) = \prod_{k=1}^K \, \explaintermbrace{\; $\N (\+\beta_k ; \qvecmu_k, \qSigma_k)$}{q(\+\beta_k)}  \ds\prod_{i=1}^N \explaintermbrace{\; $\text{PG} (\omega_{ik} ; \qbik, \qcik)$}{q(\omega_{ik})} 
\label{eqn:variational_family_ib_logit}
\end{align}
with parameter updates given by
\begin{subequations}
\begin{align}
\qbik &= 1 \\
\qcik &=  \bp{ \+x_i^T \qSigma_k \+x_i + ( \+x_i^T \qvecmu_k)^2}^{1/2}  \label{eqn:IB_cavi_updates_for_c_ik_from_CB_Logit}
\end{align}
\label{eqn:IB_cavi_updates_for_omega_ik_from_CB_Logit}
\end{subequations}

\begin{subequations}
\begin{align}
\qSigma_k &= \bp{\+X^T \+W_{\E_{q}[\+\omega_k]} \+X + \+\Sigma_0^{-1} }^{-1} \label{eqn:IB_cavi_updates_for_beta_k_covariance_from_CB_Logit} \\
\qvecmu_k &= \qSigma_k  \bp{ \+X^T \+\kappa_k + \+\Sigma_0^{-1} \+\mu_0} \label{eqn:IB_cavi_updates_for_beta_k_mean_from_CB_Logit} 
\end{align}
\label{eqn:IB_cavi_updates_for_beta_k_from_CB_Logit}
\end{subequations}

where $\+W_{\E_{q}[\+\omega_k]}$ is the $N \times K$ diagonal matrix with diagonal entries
\begin{align}
\E_{q}[\omega_{ik}] =  \df{\qbik}{2\qcik} \df{e^{\qcik} -1}{e^{\qcik} + 1}
\label{eqn:diagonal_matrix_of_expected_pga_variables_for_kth_category}
\end{align}
and where $\+\kappa_k$ was defined in Eq.~\eqref{eqn:def_of_vec_kappa_k}.
\end{proposition}

\begin{proof}   The complete conditionals  \eqref{eqn:gibbs_sampler_bayesian_logistic_regression_with_pg_augmentation}  are in the exponential family.   For the Gaussian this is well-known.   For the $\PG(1,  c_i )$ distribution,  this is immediate from Eq.~\eqref{eqn:density_of_polya_gamma_with_parameters_b_and_c}.   Due to the membership of the complete conditionals in the exponential family,  we can can apply Eq.~\eqref{eqn:IB_CAVI_for_EFCC} to determine that the optimal variational updates are in the same exponential family,  with parameters given below.   The additional independence structure in the variational family is obtained without further approximation by application of a known recipe \citep{blei2017variational}.


\paragraph{Updating the variational distribution on $\+\beta$.}  By Eq.~\eqref{eqn:IB_CAVI_for_EFCC},  the optimal variational distribution at any update is Normal.    The natural parameters of $\N(\+a, \+B)$ are given by
\begin{equation}
\qeta(\+a, \+B) = (\+B^{-1},  \+B^{-1} \+a) := (\qeta^\N_1,  \qeta^\N_2).
\label{eqn:gaussian_conversion_to_natural_parameters}
\end{equation}
 For the variational normal distribution on $\+\beta_k$,  we find that the the first coordinate of the natural variational parameter is given by:
\begin{align*}
\qeta^\N_1 &=  \E_{q_{-\+\beta_k}} [ \+\eta^\N_1] \stackrel{\eqref{eqn:gaussian_conversion_to_natural_parameters}}{=}  \E_{q_{-\+\beta_k}} [\+\Sigma_{\+\omega_k}^{-1} ] \\
&\stackrel{\eqref{eqn:gibbs_sampler_update_Sigma}}{=}  \E_{q_{-\+\beta_k}} [\+X^T \+W_{\+\omega_k} \+X + \+\Sigma_0^{-1}] \\
&= \+X^T \+W_{\E_q[\+\omega_k]} \+X + \+\Sigma_0^{-1} \\
\intertext{and therefore,  by inverting the natural parameter transformation \eqref{eqn:gaussian_conversion_to_natural_parameters}}
 \qSigma_k &= (\qeta^\N_1)^{-1} = \bp{ \+X^T \+W_{\E_q[\+\omega_k]} \+X + \+\Sigma_0^{-1}}^{-1}
\end{align*}

Similarly,  the second coordinate of the natural variational parameter is given by 
\begin{align*}
\qeta^\N_2 &=  \E_{q_{-\+\beta_k}} [ \+\eta^\N_2]  = \E_{q_{-\+\beta_k}} [\+\Sigma_{\+\omega_k}^{-1} \+\mu_{\+\omega_k} ] \\
&\stackrel{\eqref{eqn:gibbs_sampler_update_mu}}{=}  \E_{q_{-\+\beta_k}} [\+X^T \+\kappa_k + \+\Sigma_0^{-1} \+\mu_0 ] \\
&=\+X^T \+\kappa_k + \+\Sigma_0^{-1} \+\mu_0 \\
\intertext{and therefore,  by inverting the natural parameter transformation \eqref{eqn:gaussian_conversion_to_natural_parameters}}
 \qvecmu_k &=  (\qeta^\N_1)^{-1}  \qeta^\N_2 =  \qSigma_k  \bp{ \+X^T \+\kappa_k + \+\Sigma_0^{-1} \+\mu_0} 
\end{align*}

%
 
\paragraph{Updating the variational distribution on $\+\Omega$.}  By \eqref{eqn:gibbs_update_to_omega_i_under_pg_augmentation},  the complete conditional on each $\omega_{ik}$ has a \PG \, distribution.   Moreover,  
\begin{equation}
\+\eta_{ik}^{\PG}(c_{ik}) = c_{ik}^2
\label{eqn:polya_gamma_conversion_to_natural_parameters}
\end{equation}
is a natural parameter for the $\PG(1,  c_{ik} )$ distribution,  as is immediate from \eqref{eqn:density_of_polya_gamma_with_parameters_b_and_c}.

Thus,  we apply Eq.~\eqref{eqn:IB_CAVI_for_EFCC} to determine that the optimal variational update is also $\PG$ with natural parameter given by

\begin{align*}
\qeta_{ik}^{\PG} & \stackrel{Eq.~\eqref{eqn:IB_CAVI_for_EFCC}}{=} \E_{q_{-\omega_{ik}}} [\+\eta_{ik}^{\PG} ] \\
& \stackrel{\eqref{eqn:polya_gamma_conversion_to_natural_parameters}}{=} \E_{q_{-\omega_{ik}}} [c_{ik}^2 ] \\
&=   \Var_{q_{-\omega_{ik}}}[c_{ik}] + \E_{q_{-\omega_{ik}}} [c_{ik}]^2 \\
&\stackrel{\eqref{eqn:gibbs_update_to_omega_i_under_pg_augmentation}}{=}   \Var_{q_{\+\beta_k}} [\+x_i^T \+\beta_k] + \E_{q_{\+\beta_k}} [\+x_i^T \+\beta_k]^2 \\
&=\+x_i^T   \Var_{q_{\+\beta_k}} [\+\beta_k]  \, \+x_i  + (\+x_i^T \E_{q_{\+\beta_k}} [\+\beta_k])^2 \\
&=\+x_i^T   \qSigma_k \, \+x_i  + (\+x_i^T \qvecmu_k)^2 \\
\intertext{and therefore,  by inverting the natural parameter transformation \eqref{eqn:polya_gamma_conversion_to_natural_parameters}}
 \qc_{ik} &= \bp{\+x_i^T   \qSigma_k \, \+x_i  + (\+x_i^T \qvecmu_k)^2}^{1/2}
 \labelit \label{eqn:optimal_variational_parameter_for_bayesian_logreg_with_pga}
\end{align*}
where it suffices to take the positive square root since the density of $\PG(1,c)$ is symmetric around $c=0$.
\end{proof}

\subsubsection{The evidence lower bound}

Here we provide an expression for the ELBO.  

\begin{proposition}

The Evidence Lower Bound (ELBO) for the \IB-Logit with \pga (Eq.~\eqref{eqn:IB_model_with_pga}) when using the mean-field variational approximation \eqref{eqn:mf_family_for_IB_Logit_with_polya_gamma_augmentation} is given by
\begin{align*}
\ELBO[ q(\+B,  \+\Omega)] = \ds\sum_{k=1}^K \ELBO[ q(\+\beta_k,  \+\omega_k)]
\end{align*}
where
\begin{align*}
\ELBO&[ q(\+\beta_k,  \+\omega_k)] = \half d + \half \logdet{ \qSigma_k} + \half \logdet{\+\Sigma_0^{-1}} \\
&- \half (\qvecmu_k - \+\mu_0)^T \+\Sigma_0^{-1} (\qvecmu_k - \+\mu_0) - \half \tr (\+\Sigma_0^{-1} \qSigma_k) \\
&  + \ds\sum_{i=1}^N (\IBy_{ik} - \half) \, \+x_i^T \qvecmu_k - \log \big[ 1 +  \exp ( -\qcik) \big] - \half \qcik \\
\labelit \label{eqn:elbo_for_IB_Logit_with_pga} 
\end{align*}
and where $d$ is the number of rows of $\+B$.

\end{proposition}

\begin{proof} 

{\footnotesize
\begin{align*}
\ELBO &[ q(\+B,  \+\Omega)] \\
&= \E_{q(\+B, \+\Omega)} [ \log p(\IBmatrixY, \+B, \+\Omega)] - \E_{q(\+B, \+\Omega)} [ \log q(\+B, \+\Omega)] \\
&= \ds\sum_{k=1}^K \bigg[ \E_{q(\+\beta_k)} [\log p(\+\beta_k)] \\
& \quad \quad \quad + \ds\sum_{i=1}^N \E_{q(\+\beta_k)}  \E_{q(\omega_{ik})} [\log p(\IBy_{ik},  \omega_{ik} \cond \+\beta_k)] \\
& \quad \quad \quad -   \E_{q(\+\beta_k)} [\log q(\+\beta_k)]  \\
& \quad \quad \quad -  \ds\sum_{i=1}^N \E_{q(\+\beta_k)}  \E_{q(\omega_{ik})} [\log q(\omega_{ik})]  \bigg] \\
&= \ds\sum_{k=1}^K \bigg[ - \KL{q_{\+\beta_k}(\+\beta_k)}{p_{\+\beta_k}(\+\beta_k)} \\
&\quad \quad \quad +  \ds\sum_{i=1}^N \E_{q(\+\beta_k)}  \E_{q(\omega_{ik})} [\log p(\IBy_{ik},  \omega_{ik} \cond \+\beta_k) \\
&\quad\quad\quad - \log q(\omega_{ik})  ]  \bigg] 
\labelit \label{eqn:elbo_for_IB_logit_with_pga_expanded_expression}
\end{align*}
}

The first term is the negative KL divergence between two Gaussians,  which is well-known (and is given by the first line of Eq.~\eqref{eqn:elbo_for_IB_Logit_with_pga}).

For the second term,  we read Lemma 1 of \cite{durante2019conditionally} from right to left to obtain 
\begin{align*}
\E_{q(\omega_{ik})} &[\log p(\IBy_{ik},  \omega_{ik} \cond \+\beta_k) - \log q(\omega_{ik})  ] = \log \bar{p}(\IBy_{ik} \cond \+\beta_k) \\
&= (\IBy_{ik} - \half) \, \+x_i^T \+\beta_k - \half \qcik  \\
& \quad\quad\quad - \frac{1}{4} \qcik^{-1} \tanh (\half \qcik) \big[ (\+x_i^T \+\beta_k)^2 - \qcik^2 \big]\\
&\quad \quad \quad  -  \log \big[ 1 +  \exp ( -\qcik) \big] 
 \labelit  \label{eqn:jakkola_lower_bound}
\end{align*}
where $\log \bar{p}(\IBy_{ik} \cond \+\beta_k) \leq \log p(\IBy_{ik} \cond \+\beta_k)$ is the well-known quadratic lower-bound given by \cite{jaakkola2000bayesian}.\footnote{That is, the \textit{exact} ELBO for the \IB-Logit model after \pga has a summand which can be expressed as the expected value of the the well-known quadratic lower-bound given by \cite{jaakkola2000bayesian}.}

Taking the expectation w.r.t $q_{\+\beta_k}$ of Eq.~\eqref{eqn:jakkola_lower_bound}, we obtain
\begin{align*}
\E_{q(\+\beta_k)}  &\E_{q(\omega_{ik})} [\log p(\IBy_{ik},  \omega_{ik} \cond \+\beta_k) - \log q(\omega_{ik})  ]   \\
&=  (\IBy_{ik} - \half) \, \+x_i^T \qvecmu_k - \half \qcik  -  \log \big[ 1 +  \exp ( -\qcik) \big] 
\labelit \label{eqn:expectation_of_jakkola_lower_bound_wrt_variational_beta}
\end{align*}

where the third term in the sum disappears since $\E_{q(\+\beta_k)}[(\+x_i^T \+\beta_k)^2] = \qcik^2$,  as we saw in the argument leading to Eq.~\eqref{eqn:optimal_variational_parameter_for_bayesian_logreg_with_pga}. 

Taking the sum of Eq.~\eqref{eqn:expectation_of_jakkola_lower_bound_wrt_variational_beta}
across $N$ observations produces the second term in Eq.~\eqref{eqn:elbo_for_IB_logit_with_pga_expanded_expression}.
\end{proof}

\subsection{Computational complexity}  
\label{sec:complexity_of_IB_Logit}

The complexity for each iteration of CAVI for an \IB-Logit model is $\mathcal{O}(M^3K + NM^2K)$. Details on each substep are given in  Table~\ref{tab:computational_complexity_of_IB_CAVI_for_Logit}. Note in particular that, unlike with the \IB-Probit model, the \IB-Logit model requires a matrix inversion \textit{at each step of inference}, rather than just once up-front.  This increases the per-iteration complexity from $\mathcal{O}(MNK)$. (Recall that Sec.~\ref{sec:computational_complexity_for_IB_probit} provides more information on the complexity of CAVI for \IB-Probit.)

If one imposes an additional variational assumption beyond that given in Eq.~\ref{eqn:mf_family_for_IB_Logit_with_polya_gamma_augmentation}, namely that each category's regression weights are independent across covariates, $q(\+\beta_k) = \prod_{m=1}^M q(\beta_{mk})$, then the additional computational complexity imposed by \IB-Logit over \IB-Probit can be avoided.  This strategy may make sense when the choice of link (Logit over Probit) is more important than modeling the correlations in regression weights across covariates.  

As with the \IB-Probit model, the entire inference procedure (across all iterations) is embarassingly parallel over the $K$ categories.  
Thus, distributed computation over $K$ nodes can reduce the complexity for the entire inference procedure to $\mathcal{O}((M^3 + NM^2)I)$, where $I$ is the number of CAVI iterations.  Recall also that sparsity of the matrix $\+X$ can reduce the complexity of these steps. 
 
\begin{table}[h!]
{\scriptsize
\begin{tabular}{r|l|l|l}
\textbf{Variable}&  \textbf{Step} & \textbf{Per-iteration} & \textbf{Note}\\
&& \textbf{complexity} & \\
\toprule 
$\+B$ & covariance \eqref{eqn:IB_cavi_updates_for_beta_k_covariance_from_CB_Logit}  & $\mathcal{O}(M^3K + NM^2K)$ & matrix inversion  \\
$\+B$ & mean \eqref{eqn:IB_cavi_updates_for_beta_k_mean_from_CB_Logit} & $\mathcal{O}(MNK + M^2K)$ & covariance matrix \\
& & & not pre-computable \\
$\+\Omega$ & \eqref{eqn:IB_cavi_updates_for_omega_ik_from_CB_Logit}  & $\mathcal{O}(NM^2 + NMK)$ & $ \explainterm{NxM}{\+X} \; \explainterm{MxM}{\wt{\+\Sigma}} \; \explainterm{MxN}{\+X^T}$ \\
& & & and $\explainterm{NxM}{\+X} \; \explainterm{MxK}{\wt{\+\mu}}$  \\
\end{tabular}
}
\caption{\textit{The computational complexity of CAVI updates for \IB-Probit.} $\+B$ is the matrix of regression weights and $\+\Omega$ are auxiliary variables added for conditional conjugacy.}
\label{tab:computational_complexity_of_IB_CAVI_for_Logit}	
\end{table}

\section{Alternative methods for Bayesian inference in categorical GLMs} \label{sec:supplemental_info_on_alternative_inference_methods}

In this section, we provide further information about alternative approaches to Bayesian inference for categorical GLMs.  For orientation, see Table~\ref{tab:comparison_of_Bayesian_categorical_models}, for which an extended caption is given in Sec.~\ref{sec:extended_caption_for_feature_table}.  In particular, the first five rows of Table~\ref{tab:comparison_of_Bayesian_categorical_models} provide alternative approaches to \CB-Models with IB-CAVI.     
\begin{itemize}
\item In Sec.~\ref{sec:ADVI}, we describe   automatic differentation variational inference, which can be used for Bayesian inference with the softmax model (row 1).   We include this approach in our experiments.
\item We do not consider the MNP models (rows 2 and 5) due to the lack of closed-form category probabilities, which can complicate inference in high dimensions.   
\item In Sec.~\ref{sec:gibbs_sampling_for_multi_logit_regression_with_pga}, we provide the construction of a Gibbs sampler for the softmax (more specifically, for the multi-logit model, which is the softmax model but with one category's vector of regression weights fixed to $\+0$ for identifiability) after \pga~(row 3).  We include this approach in our experiments. We cannot construct closed-form CAVI for softmax after \pga, as we show in Sec.~\ref{sec:blocker_to_cavi_for_multiclass_reg_with_canonical_link}.  
\item In Sec.~\ref{sec:multinomial_reg_with_stick_breaking_link}, we consider the stick-breaking construction of the softmax (row 4). We highlight the category asymmetry of this method, which causes us to not consider this approach further in our experiments. 
\end{itemize}

\red{TODO: Move to experiments?}

\subsection{Extended caption for Table~\ref{tab:comparison_of_Bayesian_categorical_models} } \label{sec:extended_caption_for_feature_table}

An extended caption for Table~\ref{tab:comparison_of_Bayesian_categorical_models} follows.  See the main body of the text for citations for these methods. 

\paragraph{Further details on columns:}
\begin{itemize}
\item \texttt{Category symmetry} refers to symmetric handling of categories. 
\item \texttt{Latent linear regression} reports the existence of latent auxiliary variables $z_i$, one for which observation, for which the regression weights $\+\beta$ have a linear regression likelihood. (This enables easy extensibility, e.g. to hierarchical models or variable selection priors.) 
\item \texttt{Auxiliary variable independence} is satisfied when the latent auxiliary variables, one for each observation, are conditionally independent across categories given all observations and all other unobserved random variables. (Non-existence of such auxiliary variables is considered to meet the criterion.) 
\item \texttt{Closed-form likelihood} refers to closed-form category probabilities in the marginal likelihood. 
\item \texttt{Conditional conjugacy} refers to the state whereby a (non-trivial) conjugate prior exists for each complete conditional.  
\item \texttt{Closed-form variational inference} refers to the existence of a known coordinate ascent variational inference algorithm with closed-form updates.
\item \texttt{Embarassingly parallel across categories} refers to the state where  the inference can be performed separately on each category's regression weights.
\end{itemize}

\paragraph{Further details on specific cells:} 
The lack of closed-form CAVI for Softmax+PGA is reviewed in Sec.~\ref{sec:blocker_to_cavi_for_multiclass_reg_with_canonical_link}. The category asymmetry of the SB-Softmax+PGA method is discussed in Sec.~\ref{sec:multinomial_reg_with_stick_breaking_link}.  The latent linear regression property of IB-CAVI can be exploited for closed-form hierarchical modeling, as mentioned in Sec.~\ref{sec:general_variational_algorithm}, but we do not consider such models in this paper. 

\red{TODO: Need to update the table to track developments that have been made to the table in the main body of the paper.}

\subsection{Automatic differentiation variational inference} \label{sec:ADVI}

Automatic differentiation variational inference (ADVI) \cite{kucukelbir2017automatic} is a generic variational inference method that applies to a large class of Bayesian models.  Its objective function is known as the ADVI evidence lower bound (ADVI ELBO):
\begin{align}
 \mathcal{L}(\+\lambda) &= \E_{\norm (\+\epsilon; \+0, \+I)}  \bigg[ 
	\log p \bigg( \+y, T^{-1} \big( S^{-1}_{\+\lambda} (\+\epsilon) \big) \bigg) \nonumber \\
	& + \logabsdet{J_{T^{-1}} \big( S^{-1}_\+\lambda (\+\epsilon) \big)}  \bigg]  + \mathbb{H} [q(\+\zeta ; \+\lambda )]
\label{eqn:ADVI_ELBO}	
\end{align}
where $\+\lambda$ are the variational parameters, $T : \Theta \to \R^P, \; \+\theta \mapsto \+\zeta$ is a  differentiable bijection to give the model parameters $\+\theta$ unbounded support, and $S_\+\lambda : \R^P  \to \R^P, \, \+\zeta \mapsto \+\epsilon$ is a (deterministic) standardization function.
	
The gradients for the ADVI objective are given in \cite{kucukelbir2017automatic}.   We assume a Gaussian mean-field variational family on $\+\zeta \in \R^p$, the transformed unobserved random variables, i.e. the variational parameters are $\+\lambda = \big(\qvecmu, \text{diag}(\qvecsigma^2)\big)$ and the variational density is given by
\[ q(\+\zeta ; \+\lambda) = \prod_{p=1}^P  \explaintermbrace{$\N(\qmu_p, \qsigma_p^2)$}{q(\zeta_p ; \+\lambda_p)}  \] 
\begin{subequations}
Under this Gaussian mean field assumption, the gradients are given by \cite{kucukelbir2017automatic}
\begin{align}
\nabla_{\qvecmu} \mathcal{L} &= \E_{\N(\+\epsilon; \+0, \+I)} \bigg[ \explaintermbrace{$1 \times p$}{\nabla_{\+\theta} \log p(\+y, \+\theta)} \, \explaintermbrace{$p \times p$}{\nabla_{\+\zeta} T^{-1} (\+\zeta)} \nonumber\\
& \quad +\explaintermbrace{$1 \times p$}{\nabla_{\+\zeta} \logabsdet{J_{T^{-1}}(\+\zeta)}} \bigg] \label{eqn:advi_gradient_wrt_mu} \\
\nabla_{\qvecomega} \mathcal{L} &= \E_{\N(\+\epsilon; \+0, \+I)} \bigg[ \bigg( \explaintermbrace{$1 \times p$}{\nabla_{\+\theta} \log p(\+y, \+\theta)}\, \explaintermbrace{$p \times p$}{\nabla_{\+\zeta} T^{-1} (\+\zeta)} \nonumber\\
& \quad + \explaintermbrace{$1 \times p$}{\nabla_{\+\zeta} \logabsdet{J_{T^{-1}}(\+\zeta)}} \bigg)  \explaintermbrace{$p \times p$}{\nabla_{\qvecomega} S_{\+\lambda}^{-1}(\+\epsilon)} \bigg] + \+1 \label{eqn:advi_gradient_wrt_omega}
\end{align}
where we have defined $\qvecomega = (\qomega_1, \hdots \qomega_p) \in \R^p$ as the element-wise log of the variational standard deviations, $\qomega_p = \log \qsigma_p$.  This transformation gives $\qvecomega$ unbounded real-valued support.
\label{eqn:advi_gradient}
\end{subequations}

In the case of categorical GLMS (Eq.~\eqref{eqn:categorical_regression}), the model parameter is given by $\+\theta = \text{vec}(\+B)$.  Here the model parameter $\+\theta$ already has unconstrained real-valued support, so the ADVI gradients \eqref{eqn:advi_gradient} simplify greatly. Since $T$ is the identity function, we have  $
J_{T^{-1}}(\+\zeta) = \nabla_{\+\zeta} T^{-1} (\+\zeta) = \+I_p$ and $\nabla_{\+\zeta} \logabsdet{J_{T^{-1}}(\+\zeta)} = \+0_p$.  Therefore, the ADVI gradients become
\begin{subequations}
\begin{align}
\nabla_{\qvecmu} \mathcal{L} &= \E_{\N(\+\epsilon; \+0, \+I)} \bigg[ \explaintermbrace{$1 \times p$}{\nabla_{\+\theta} \log p(\+y, \+\theta)}  \bigg] \label{eqn:advi_gradient_wrt_mu_under_real_parameter} \\
\nabla_{\qvecomega} \mathcal{L} &= \E_{\N(\+\epsilon; \+0, \+I)} \bigg[  \explaintermbrace{$1 \times p$}{\nabla_{\+\theta} \log p(\+y, \+\theta)} \;  \explaintermbrace{$p \times p$}{\nabla_{\qvecomega} S_{\+\lambda}^{-1}(\+\epsilon)} \bigg] + \+1_p \label{eqn:advi_gradient_wrt_omega_under_real_parameter}
\end{align}
\label{eqn:advi_gradient_under_real_parameter}
\end{subequations}
If we take $S_\lambda$ to be an elliptical standardization \cite{kucukelbir2017automatic}, we obtain 
\begin{align*}
\nabla_{\qvecomega} S_{\+\lambda}^{-1}(\+\epsilon) &= \text{diag}(\+\epsilon^T \exp(\qvecomega)) \\
&=   
\begin{bmatrix}
    \epsilon_1 \exp(\qomega_1) & & \\
    & \ddots & \\
    & &  \epsilon_p \exp(\qomega_p)
  \end{bmatrix}
\end{align*}
So (stochastic) gradient steps to optimize the ADVI ELBO are conceptually straightforward to compute using Monte Carlo sampling and automatic differentiation of the joint density with respect to the parameter $\+\theta$.   However, the generic framework comes at a price.  Whereas CAVI provides exact analytic solutions to each coordinate ascent update, ADVI must chase gradients, and these gradients are stochastic.  As we will see, this can slow down inference; moreover, ADVI introduces multiple optimization hyperparameters (learning rate, number of Monte Carlo samples, and more).  For a given problem, finding appropriate values of these hyperparameters can be challenging.

\subsection{Gibbs sampling for multi-logit regression with \pga} \label{sec:gibbs_sampling_for_multi_logit_regression_with_pga}

\subsubsection{Bayesian Binomial Regression with \pga}
\label{sec:derivation_complete_conditionals_bayes_binom_reg_with_pga}

Here we derive the complete conditionals for Bayesian Binomial Regression with \pga.  Bayesian logistic regression is a special case, and Bayesian multiclass logistic regression is an extension (see Sec.~\ref{sec:bayesian_multiclass_logistic_regression_with_pga}).   This derivation will be useful for Sec.~\ref{sec:blocker_to_cavi_for_multiclass_reg_with_canonical_link}, where we demonstrate the lack of closed-form CAVI for softmax regression with \pga. 

Our derivation largely follows the simple derivation given in Section 2 of \cite{choi2013polya},  but provides some extra detail.\footnote{We also make the generalization from logistic to binomial regression explicit.  Although the tweak is straightforward, this expression is nicely more general and is also something we use when we handle the stick-breaking multi-class logistic regression.}   For the derivation,  recall the density of a $\PG(b,c)$ random variable \cite{polson2013bayesian}:
\begin{align*}
f(x; b, c) =  \cosh^b(c/2) e^{-\frac{c^2}{2} x} g(x; b,  0)
\labelit \label{eqn:density_of_polya_gamma_with_parameters_b_and_c}
\end{align*}

where $h(\omega) := g(x; b, 0)$ is the density of a $\PG(b,0)$ random variable
\begin{align*}
g(x; b,  0)  = \df{2^{b-1}}{\Gamma(b)} \ds\sum_{n=0}^\infty  (-1)^n \df{\Gamma(n+b)}{\Gamma(n+1)}  \df{(2n+b)}{\sqrt{2 \pi x^3}} e^{-\frac{(2n+b)^2}{8x}}
\labelit \label{eqn:definition_of_g_in_expansion_of_the_PG_density}
\end{align*}

So $f$ is constructed from $g$ via an \textit{exponential tilt} and a renormalization.  Note that to derive the complete conditionals, we will not need the form of $g(x;b,0)$ anywhere in the derivation;  we merely use \eqref{eqn:density_of_polya_gamma_with_parameters_b_and_c}.   


\begin{proposition}

For the Bayesian Binomial Regression Model\footnote{Note that we use $N$ to denote the number of observations and $n_i$ to denote the number of binomial trials per observation.} with \pga

\begin{align*}
\+\beta &\sim \N(\+\mu_0, \+\Sigma_0 ) \\
y_i \cond \+\beta &\indsim \Binomial \bp{n_i,  \frac{\exp \{\+x_i^T \+\beta \}}{1+\exp \{\+x_i^T \+\beta \}}},  \quad i=1,...,N \\
\omega _i \cond \+\beta &\indsim  \PG(n_i,  \+x_i^T \+\beta),  \quad i=1,...,N \\
\labelit \label{eqn:appendix_bayesian_logreg_model_with_pga}
\end{align*}


the complete conditional distributions are 

\begin{align*}
(\omega_i \cond \+\beta) &\sim \PG(n_i,  \+x_i^T \+\beta) 
 \labelit \label{eqn:appendix_gibbs_update_to_omega_i_under_pg_augmentation}\\
(\+\beta \cond \+y,  \+\omega) &\sim \N(  \+\mu_{\+\omega},   \+\Sigma_{\+\omega}) 
 \labelit \label{eqn:appendix_gibbs_update_to_beta_under_pg_augmentation}\\
\intertext{where}
\+\Sigma_{\+\omega} &= \bp{\+X^T \+\Omega_{\+\omega} \+X + \+\Sigma_0^{-1} }^{-1} 
 \labelit \label{eqn:appendix_sampler_update_Sigma} \\
\+\mu_{\+\omega} &= \+\Sigma_{\+\omega} \bp{ \+X^T \+\kappa + \+\Sigma_0^{-1} \+\mu_0} 
 \labelit \label{eqn:appendix_gibbs_sampler_update_mu}
\intertext{where}
\+\kappa  &= (y_1 - n_1 /2, ..., y_N - n_N /2)  \\
\+\Omega_{\+\omega} & \text{\,  is the diagonal matrix of  $\+\omega_i$'s} \\ 
\labelit \label{eqn:appendix_def_of_kappa_and_Omega}
\end{align*}

\label{prop:complete_conditionals_for_bayesian_binomial_regression_with_pga}
\end{proposition}

\begin{proof}

That \eqref{eqn:appendix_gibbs_update_to_omega_i_under_pg_augmentation}  is the complete conditional for $\omega_i$ follows immediately from the conditional independence of $\+\omega$ and $\+y$ in the model.    In particular,  the posterior density is given by
\begin{align*} 
p(\+\beta, \+\omega \cond \+y) \propto  \bb{\ds\prod_{i=1}^N  p(y_i \cond \+\beta) p(\omega_i \cond \+\beta)} p(\+\beta)
\labelit \label{eqn:appendix_joint_density_bayesian_binom_reg_with_pga}
\end{align*}

Hence,  clearly, 
\[ p(\omega_i \cond \+\beta,  \+y,  \+\omega_{-i}) \propto  p(\omega_i \cond \+\beta)\]

It remains to show that \eqref{eqn:appendix_gibbs_update_to_beta_under_pg_augmentation}  is the complete conditional for $\+\beta$
{\scriptsize 
\begin{align*} 
p(\+\beta &\cond \+\omega,  \+y) \propto   \bb{\ds\prod_{i=1}^N  p(y_i \cond \+\beta) p(\omega_i \cond \+\beta)} p(\+\beta) \\
&\stackrel{1}{\propto}  \bigg[   \ds\prod_{i=1}^N \frac{e^{ (\+x_i^T \+\beta) y_i}}{ \big( 1+e^{\+x_i^T \+\beta} \big)^{n_i} } \bigg]  \bigg[   \cosh^{n_i} \big(\df{\+x_i^T \+\beta}{2} \big) e^{-\half (\+x_i^T \+\beta)^2 \omega_i}  h(\omega_i) \bigg]  p(\+\beta) \\
&\stackrel{2}{\propto}   \bigg[   \ds\prod_{i=1}^N \frac{e^{ (\+x_i^T \+\beta) y_i}}{ \cancel{ \big(1+e^{\+x_i^T \+\beta} \big)^{n_i} }} \bigg]  \bigg[   \df{\cancel{ \big( 1+e^{\+x_i^T \+\beta} \big)^{n_i} }}{ \cancel{2^{n_i}} e^{\half ( \+x_i^T \+\beta) n_i }} e^{-\half (\+x_i^T \+\beta)^2 \omega_i}  \cancel{h(\omega_i)} \bigg]  p(\+\beta) \\
&\stackrel{3}{\propto}   p(\+\beta)  \exp \bigg\{  \ds\sum_{i=1}^N  (y_i - \df{n_i}{2}) (\+x_i^T \+\beta) - \frac{\omega_i}{2} (\+x_i^T \+\beta)^2  \bigg\}
\labelit \label{eqn:gist_of_proof_of_ccs_for_bayes_binom_reg_with_pga}
\end{align*}
}where (1) fills in forms for densities (using $h(\omega) := g(\omega; 1, 0)$),  (2) uses that $\cosh(z) = \frac{1+e^{2z}}{2e^z}$ (and absorbs $h(\omega_i)$ and $2^{-n_i}$ into the constant of proportionality),  and (3) reveals that the complete conditional is Gaussian.

To obtain an explicit form for the multivariate Gaussian,  we need what \cite{choi2013polya} calls a routine Bayesian regression-type calculation:
{\scriptsize 
\begin{align*} 
p(\+\beta \cond \+\omega,  \+y) &\propto  p(\+\beta)  \exp \bigg\{  \ds\sum_{i=1}^N  (y_i - \df{n_i}{2}) (\+x_i^T \+\beta) - \frac{\omega_i}{2} (\+x_i^T \+\beta)^2  \bigg\}  \\
\intertext{setting $\kappa_i := y_i - \df{n_i}{2}$}
&\stackrel{1}{\propto}   p(\+\beta)  \exp \bigg\{  \ds\sum_{i=1}^N  -\frac{\omega_i}{2} \bigg( \+x_i^T \+\beta - \frac{\kappa_i}{\omega_i} \bigg)^2  \bigg\} \\
\intertext{defining $\+z := (\frac{\kappa_1}{\omega_1},  ...,  \frac{\kappa_N}{\omega_N})$ and $\+\Omega := \text{diag} (\omega_1, ..., \omega_N)$} 
&\stackrel{2}{\propto}  p(\+\beta)   \exp \bigg\{  -\half  (\+z - \+X \+\beta)^T \+\Omega   (\+z - \+X \+\beta) \bigg\} \\ 
&\stackrel{3}{\propto}  p(\+\beta)   \exp \bigg\{  -\half  (\+X^+ \+z - \+\beta)^T  \+X^T \+\Omega \+X   (\+X^+ \+z - \+\beta) \bigg\} 
\labelit \label{eqn:deriving_the_posterior_for_beta_for_bayesian_logistic_regression_with_pga}
\end{align*}
}where (1) is by completing the square,  (2) writes the weighted sum of squares in matrix notation, and (3) isolates $\+\beta$,  using $\+X^{+}$,  the Moore-Penrose psuedo-inverse of $\+X$.\footnote{Specifically,  since $\+X\+X^+ = \+I$,  we use 
{\scriptsize 
\begin{align*}
(\+z - \+X \+\beta)^T \+\Omega   (\+z - \+X \+\beta) &= (\+X\+\beta - \+z)^T \+\Omega (\+X\+\beta - \+z)\\ &= \bigg( \+X (\+\beta - \+X^+ \+z ) \bigg)^T \+\Omega  \bigg(  \+X (\+\beta - \+X^+ \+z ) \bigg) \\
&= (\+\beta - \+X^+ \+z)^T \+X^T \Omega \+X   (\+\beta - \+X^+ \+z) 
\end{align*}
}
.}  

Thus,  we see that $p(\+\beta \cond \+\omega,  \+y)$ is proportional to the product of two multivariate Gaussians:  $p(\+\beta)$,  which has mean $\+\mu_0$ and covariance $\+\Sigma_0$,  and another Gaussian,  which has mean $\+X^+ \+z$ and covariance $(\+X^T \+\Omega \+X)^{-1}$.   We know from the exponential family representation of the Gaussian that the result can be obtained by summing at the scale of natural parameters -- which for the Gaussian are the precision and precision-weighted mean.   Using this,  we obtain
\begin{align*}
p(\+\beta \cond \+\omega,  \+y) & \sim \N(\+\mu_{\+\omega},  \+\Sigma_{\+\omega} )
\intertext{where}
\+\Sigma_{\+\omega} &= \bp{\+\Sigma_0^{-1} +  \+X^T \+\Omega_{\+\omega} \+X }^{-1}  \\
\+\mu_{\+\omega} &= \+\Sigma_{\+\omega} \bp{   \+\Sigma_0^{-1} \+\mu_0 + \+X^T \+\Omega_{\+\omega} \cancel{\+X} \cancel{\+X^+}  \+z }  \\
&= \+\Sigma_{\+\omega} \bp{   \+\Sigma_0^{-1} \+\mu_0 + \+X^T \+\kappa }  
\end{align*}

\end{proof}

\begin{remark}
In the special case where $n_i \equiv 1$, Bayesian binomial regression reduces to Bayesian logistic regression.	
\end{remark}

\subsubsection{Bayesian multi-logit regression with \pga}  \label{sec:bayesian_multiclass_logistic_regression_with_pga}

\red{TODO: Relate multinomial to categorical, or reduce to categorical.} \cite{held2006bayesian} show that for multi-class logistic regression with the standard,  canonical (multi-logit) link,  the conditional likelihood $L(\+\beta_k \cond \+y,  \+\beta_{-k})$ over categorical outcomes $\+y \in \set{1,\hdots, K}^N$ has the form of a logistic regression on the class indicators $\IBy_{ik} \in \set{0,1}$. This observation motivates the conversion of Bayesian multinomial regression into a conditionally conjugate model.  We present the complete conditionals here as they are used to construct a Gibbs sampler in the experiments (see Sec.~\ref{sec:modeling_strategies}).   However, in Sec.~\ref{sec:blocker_to_cavi_for_multiclass_reg_with_canonical_link} we show that the construction does not yield closed-form CAVI updates (as reported in row 3 of Table 1).  

First,  following \cite{held2006bayesian},  note that we can represent the complete conditionals for $\+\beta_k,  k=1,...,K-1$ in terms of the conditional likelihoods $L(\+\beta_k \cond \+y,  \+\beta_{-k})$
\[ p(\+\beta_k \cond \+y, \+\beta_{-k}) \propto p(\+\beta_k) \, L(\+\beta_k \cond \+y,  \+\beta{-k})  \]
where the conditional likelihoods satisfy
\begin{align*}
L(\+\beta_k \cond \+y,  \+\beta_{-k}) & \propto \ds\prod_{i=1}^N \ds\prod_{k=1}^K p_{ik}^{\IBy_{ik}} \\
&\propto \ds\prod_{i=1}^N (\gamma_{ik})^{\IBy_{ik}} (1-\gamma_{ik})^{1-\IBy_{ik}}  \\
\intertext{where} 
\gamma_{ik} &= \df{\exp ( \+x_i^T \+\beta_k - C_{ik}) }{1+\exp ( \+x_i^T \+\beta_k - C_{ik})}  \labelit \label{eqn:conditional_category_prob_held_and_holmes} \\
C_{ik} & := \log \sum_{j \neq k} \exp (\+x_i^T \+\beta_j) 
\end{align*}
which reveals that the conditional likelihood $L(\+\beta_k \cond \+y,  \+\beta_{-k})$ has the form of a logistic regression on class indicators $\IBy_{ik}$.

The form of \eqref{eqn:conditional_category_prob_held_and_holmes} and the success of \pga with standard (binary) logistic regression suggests that we should construct an augmented model for Bayesian multi-class logistic regression  by taking, for $i=1,...,N$ and $k=1,...K-1$,
\[  \omega_{ik} \cond \+\beta_k \indsim \PG(1,  \+x_i^T \+\beta_k - C_{ik}),  \quad \] 
a slight tweak on the construction for standard (binary) logistic regression  \eqref{eqn:appendix_bayesian_logreg_model_with_pga},  where we had $\omega_{i} \cond \+\beta \indsim \PG(1,  \+x_i^T \+\beta)$.

Following \eqref{eqn:gist_of_proof_of_ccs_for_bayes_binom_reg_with_pga},  but using the conditional likelihood and altered construction for the \polyaGamma auxiliary variables\footnote{Note that,  for this example,  we are unnecessarily restricting to the case of multiclass logistic regression ($n_i \equiv 1$).  The same argument that we make here would of course also hold for multinomial regression,  which is a generalization.} we find 
{\tiny
\begin{align*} 
p(&\+\beta _k \cond \+\omega,  \+y,  \+\beta _{-k} )  \\
&\propto    \bigg[   \ds\prod_{i=1}^N \frac{e^{ (\+x_i^T \+\beta - C_{ik}) \IBy_{ik}}}{ \big( 1+e^{\+x_i^T \+\beta - C_{ik}} \big)}  \bigg]  \\ 
&\quad\quad\quad\quad\quad  \times \bigg[   \cosh \big(\df{\+x_i^T \+\beta - C_{ik}}{2} \big) e^{-\half (\+x_i^T \+\beta - C_{ik})^2 \omega_{ik}}  h(\omega_{ik}) \bigg]  p(\+\beta_k) \\
&\propto   \bigg[   \ds\prod_{i=1}^N \frac{e^{ (\+x_i^T \+\beta - C_{ik}) \IBy_{ik}}}{ \cancel{ \big(1+e^{\+x_i^T \+\beta - C_{ik}} \big) }} \bigg]   \\
& \quad\quad\quad\quad\quad \times \bigg[   \df{\cancel{ \big( 1+e^{\+x_i^T \+\beta - C_{ik}} \big) }}{ \cancel{2} e^{\half ( \+x_i^T \+\beta - C_{ik})}} e^{-\half (\+x_i^T \+\beta - C_{ik})^2 \omega_{ik}}  \cancel{h(\omega_{ik})} \bigg]  p(\+\beta_k) \\
&\propto   p(\+\beta_k)  \exp \bigg\{  \ds\sum_{i=1}^N  (\IBy_{ik} - \df{1}{2}) (\+x_i^T \+\beta_k - C_{ik}) - \frac{\omega_{ik}}{2} (\+x_i^T \+\beta_k - C_{ik})^2  \bigg\}  \labelit \label{eqn:complete_conditional_for_cc_for_beta_for_one_category_for_a_multinomial_logit_with_pga}
\end{align*}
}
Continuing to parallel the argument of Section \ref{sec:derivation_complete_conditionals_bayes_binom_reg_with_pga},  using Eq.~(\ref{eqn:complete_conditional_for_cc_for_beta_for_one_category_for_a_multinomial_logit_with_pga}) instead of Eq.~(\ref{eqn:gist_of_proof_of_ccs_for_bayes_binom_reg_with_pga}),  we find that the complete conditionals are given by 
\begin{subequations}
\begin{align}
\+\beta_k \cond \+\omega_k,  \+y & \sim \N (\+\mu_k, \+\Sigma_k) \label{eqn:complete_conditional_for_beta_k_under_multi_logit_regression_with_pga}\\
\omega_{ik} \cond \+\beta_k & \sim \PG(1, \+x_i^T \+\beta_k - C_{ik}) \label{eqn:complete_conditional_for_omega_ik_under_multi_logit_regression_with_pga}\\
\intertext{where}
\+\Sigma_k &= \bp{\+\Sigma_0^{-1} +  \+X^T \+\Omega_k \+X }^{-1}  \nonumber \\
\+\mu_{\+\omega} &= \+\Sigma_k \bp{   \+\Sigma_0^{-1} \+\mu_0 + \+X^T \+\Omega_k \+z_k } \nonumber 
\intertext{for}
\+\Omega_k & = \text{diag}\big(\omega_{1k}, \hdots , \omega_{Nk} \big) \nonumber \\ 
\+z_k  &= 
\begin{bmatrix}
\frac{\IBy_{1k} - 1/2}{\omega_{1k}} + C_{1k} \\ 
\vdots \\
\frac{\IBy_{Nk} - 1/2}{\omega_{Nk}}  + C_{Nk} 
 \end{bmatrix}  \nonumber 
\end{align}  
\end{subequations}

A valid Gibbs sampler is obtained by iteratively sampling from Eqs.~\eqref{eqn:complete_conditional_for_beta_k_under_multi_logit_regression_with_pga} and~\eqref{eqn:complete_conditional_for_omega_ik_under_multi_logit_regression_with_pga}.

\subsection{Lack of closed-form CAVI for Bayesian multi-logit regression and \pga}  \label{sec:blocker_to_cavi_for_multiclass_reg_with_canonical_link}

Here we demonstrate the lack of closed-form CAVI for Bayesian multi-logit regression under \pga (as reported in row 3 of Table 1).  To do so, we focus on the complete conditionals for $\omega_{ik}$.  Namely,  if we would like to perform coordinate ascent variational inference (CAVI),  we parallel the argument of \eqref{eqn:optimal_variational_parameter_for_bayesian_logreg_with_pga} \red{TODO: Fix; once CFB-Logit is in the appendix we can refer to it here.},  but with our current multi-class situation whereby we work with the complete conditional for $\omega_{ik}$ as $\PG(1,  c_{ik})$,  where  $c_{ik} = \+x_i^T \+\beta_k - C_{ik}$.  This differs slightly from the standard (binary) logistic regression case,   where the complete conditional for $\omega_{i}$ was $\PG(1,  c_{i})$,  where  $c_{i} = \+x_i^T \+\beta$.   In the current case,  we find by paralleling  \red{TODO: Fix}\eqref{eqn:optimal_variational_parameter_for_bayesian_logreg_with_pga} that we eventually need
{ \scriptsize 
\begin{align*}
\E_{q_{-\+\omega}} [c_{ik}]^2 &=  \bigg( \+x_i^T \E_{q_{\+\beta_k}} [\+\beta_k] -  \E_{q_{-\omega_{ik}}} [C_{ik }] \bigg)^2 \\
&=  \bigg( \+x_i^T \E_{q_{\+\beta_k}} [\+\beta_k] -  \E_{q_{\+\beta_{-k}}} \big[ \log \sum_{j \neq k} \exp (\+x_i^T \+\beta_j ) \big] \bigg)^2 
\labelit \label{eqn:blocker_to_cavi_in_multiclass_logistic_regression}
\end{align*}
}and while the first expectation in equation \eqref{eqn:blocker_to_cavi_in_multiclass_logistic_regression} is straightforward and parallels what we computed in binary logistic regression,  the expected log-sum-exp is distinct to the multiclass case,  and has no closed form.  Indeed, the expected log-sum-exp is a notorious blocker to closed-form CAVI. Indeed, precisely this fact motivated Delta Variational Inference  \cite{braun2010variational} \cite{wang2013variational}. For an enumeration of many bounds to this expression,  see \cite{depraetere2017comparison}.     

Thus,  if one seeks variational inference with closed-form updates,  \pga solves the problem for Bayesian binomial regression,  but not for Bayesian multi-class logistic regression (or more generally Bayesian multinomial regression),  at least not when using the standard canonical (multi-logit) link.

\subsection{Stick-breaking multi-class logistic regression} \label{sec:multinomial_reg_with_stick_breaking_link}

The stick-breaking construction of the multi-class logistic regression regression \cite{linderman2015dependent} is useful for the purpose of exploiting \pga for efficient inference.   First, the density of a \textit{categorical} distribution over $K$ categories with  parameter $\+\pi=(\pi_1, ..., \pi_k) \in \Delta_{K-1}$ can be represented in a stick-breaking manner as a product of $K-1$ Bernoullis. The density can be expressed as 
\begin{align}
\prod_{k=1}^{K-1} \wt{\pi_k}^{\IBy_{ik}} (1 - \wt{\pi_k})^{1-\IBy_{ik}}
\label{eqn:stick_breaking_categorical_density}	
\end{align}
where $\wt{\pi_k} := \df{\pi_k}{1 - \sum_{j<k} \pi_j}$ is the Bernoulli parameter and $\IBy_{ik} =1$ if the $i$th observation is the $k$th category, and $\IBy_{ik} =0$ otherwise.   

Stick-breaking multi-class logistic regression uses Eq.~\ref{eqn:stick_breaking_categorical_density} to construct a multi-class logistic regression over $K$ categories as a product of $K-1$ logistic regressions
\[ \prod_{k=1}^{K-1} \bigg(  \df{e^{\+x_i^T \+\beta^\stickBreaking_k}}{1+e^{\+x_i^T \+\beta^\stickBreaking_k}}\bigg)^{\IBy_{ik}} \bigg(  \df{1}{1+e^{\+x_i^T \+\beta^\stickBreaking_k}}\bigg)^{1-\IBy_{ik}}  \]

The multinomial parameter $\+\pi_i$ has explicit form given by
\begin{align*}
 \pi_{ik}  =   \df{e^{\+x_i^T \+\beta^\stickBreaking_k}}{1+e^{\+x_i^T \+\beta^\stickBreaking_k}}  \ds\prod_{j<k}  \df{1}{1+e^{\+x_i^T \+\beta^\stickBreaking_j}}  ,  \quad k=1,...,K
\labelit \label{eqn:stick_breaking_multinomial_parameter_explicit_form_for_components}
\end{align*}
where $\+\beta_K \equiv 0$.
\paragraph{Label asymmetry.} The stick-breaking formulation induces a label asymmetry, which can complicate prior-setting and reduce representational capacity \cite{zhang2017permuted}.   For example,  consider the case of multiclass logistic regression (so $n_i =1$).   In standard multinomial regression,  we have 
\begin{align}
p(\IBy_{ik} = 1 \cond \+\beta^\multiLogit) =  \df{e^{\+x_i^T \+\beta^\multiLogit_k}}{1+ \sum_{k=1}^{K-1} e^{\+x_i^T \+\beta^\multiLogit_k}} 
\label{eq:category_probability_under_standard_multinomial_regression}
\end{align}
whereas in stick-breaking multinomial regression,   we have (via Eq.~\ref{eqn:stick_breaking_multinomial_parameter_explicit_form_for_components})
\begin{align*}
 P(\IBy_{ik} = 1  \cond \+\beta^\stickBreaking)  =   \df{e^{\+x_i^T \+\beta^\stickBreaking_k}}{1+e^{\+x_i^T \+\beta^\stickBreaking_k}}  \ds\prod_{j<k}  \df{1}{1+e^{\+x_i^T \+\beta^\stickBreaking_j}}  
 \end{align*}
which differs from \eqref{eq:category_probability_under_standard_multinomial_regression} in that it clearly imposes fewer geometric constraints on the classification decision boundaries for smaller $k$. 
 For instance,  $p_{i1}$ can be larger than 50 \% if $\+x_i^T \+\beta_1 >0$,  whereas $p_{i2}$ can be larger than 50 \% only if $\+x_i^T \+\beta_1 <0$ and $\+x_i^T \+\beta_2 >0$.  Because of label asymmetry, predictive performance can be sensitive to how the $K$ different categories are ordered.  The geometric constraints implied by any given ordering of the labels can cause the model to struggle to learn the true decision boundaries.  

\section{Supplemental information for experiments} \label{sec:supplemental_info_for_experiments}

Open-source python code for reproducing experiments can be found at
\url{\codeURL}.

\subsection{Data simulations} \label{sec:simulating_data}
\subsubsection{Dataset generation} \label{sec:data_generation}

We generate simulated datasets from a categorial distribution with a softmax (multi-logit) inverse link function and given specifications ($N$ samples, $K$ categories, $M$ covariates).  For a given context ($N,K,M$), we may generate $D$ different datasets by setting the random seed to a different value.   

First, we generate covariate matrices $\+X \in \R^{N \times M}$ such that all entries are drawn i.i.d from $\N(0,1)$.  We use $\+x_i$ to refer to the $i$th row of $\+X$ for $i=1,\hdots,N$.

Next, we draw regression weights $\+B \in \R^{(M + 1) \times K}$  in a way that allows us to control category predictability.  We describe how to sample entries $\beta_{mk}$ for $m=0,1,\dots,M$ and $k=1,\hdots,K$.  We begin by generating intercepts for each category by sampling $\beta_{0k} \iid \N(0, \sigma^2_{\text{int}})$.  For covariates $m =1, \hdots, M$, we draw $\beta_{mk} \indsim \N(0,\sigma^2_{mk})$, where
\[ \sigma^2_{mk} = 
\begin{cases}
\sigma^2_{\text{high}}, & \text{ if $k$ = $\ceil{m /S}$}\\
\sigma^2_{\text{low}}, & \text{ otherwise }\\
\end{cases}
\]
for $\sigma^2_{\text{high}} > \sigma^2_{\text{low}}$ and $S:=\floor{M/K}$.  The motivation is as follows: We partition the $M$ covariates into $K+1$ covariate groups.  Covariate groups $k=1,...,K$ each have $S$ members that are potentially predictive of the $k$th category. Such covariates have regression entries $\beta_{mk} \sim \N(0,\sigma^2_{\text{high}})$; the relatively high variance $\sigma^2_{\text{high}} > \sigma^2_{\text{low}}$ allows the regression coefficients to escape the mean of zero.    As the value of $\sigma^2_{\text{high}}$ increases, the overall predictability of the categories given the covariates increases.   Note that there may be an additional $(k=0)$th group, with $M \bmod K$ members, which is not predictive of a specific category.  

Finally, we generate categorical observations by associating the covariates and regression weights via the softmax (multi-logit) inverse link function.  

Overall, our data generating mechanism is 
\begin{align*}
x_{im} &\iid \N(0,1), \quad\quad \tiny{i=1,\hdots, N, \; m=1,\hdots, M} \\
\beta_{mk} & \indsim N(0, \sigma^2_{mk}), \quad \tiny{m=0,\hdots, M, k=1,\hdots, K}\\
& \text{ where } \sigma^2_{mk}  =
\begin{cases}
\sigma^2_{\text{int}}, & \text{ if $m=0$}\\
\sigma^2_{\text{high}}, & \text{ if $m \geq 1$ and $k$ = $\ceil{m /S}$}\\
\sigma^2_{\text{low}}, & \text{ otherwise }\\
\end{cases} \\
y_i \cond \+x_i, \+B &\sim \text{Softmax}(\+B^T \widedot{\+x}_i)
\end{align*}
for observations $i=1,\hdots,N$, covariates $m=1,\hdots,M$ and categories $k=1,\hdots,K$, and where $\widedot{\+x}_i = (1,\+x_i^T)^T$ are the covariates prepended with a  value of 1 to correspond to the intercept term.  

Unless otherwise specified, we fix $\sigma^2_{\text{low}} = 0.001$ and $\sigma^2_{\text{int}} = 0.25$. We vary $\sigma^2_{\text{high}}$ throughout the experiments to control predictability.
    
\subsubsection{Metrics} \label{sec:metrics}
 
 To estimate the predictability of categories, we estimate the mean covariate-conditional category entropy for each dataset:
{\small 
\begin{align*}
\E_X \H[Y \cond X] \approx - \ds\sum_{i=1}^{N} \sum_{k=1}^K  p(y_i=k \cond \+x_i, \+B)\log p(y_i=k \cond \+x_i, \+B)
\labelit \label{eqn:mean_covariate_conditional_category_entropy}
\end{align*}}
where $p$ refers to the category probabilities and $\+B$ refers to the known regression weights from the true data generating process (softmax). 

The mean holdout log-likelihood for the $r$th prediction method  is computed by:

\begin{align}
\df{1}{N_{\text{test}}} \ds\sum_{i=1}^{N_\text{test}} \log   p_r(y_i=k \cond \+x_i, \+B_r) 
\label{eqn:mean_holdout_log_likelihood}
\end{align}

where the category probability formula $p_r$ and point estimate for $\+B_r$ are determined by the values of the corresponding columns for the $r$th row of Table \ref{tab:summary_of_modeling_strategies}.\footnote{Two of the modeling strategies - namely \texttt{Softmax (via MLE)} and \texttt{Baserate frequency} - can produce predictive probabilities of exact or numerical zero, e.g. when a category is observed in the test set that was never observed in the training set.  A single such instance will drive the log-likelihood metric to $-\infty$ regardless of the log likelihoods for any other sample.  To handle this issue, we renormalize these models to produce a minimum predictive probability of $\epsilon := 10^{-10}$ for each category.}


\subsection{Bayesian model averaging experiment: Supplemental information}

\subsubsection{Methodology} \label{sec:methodology_BMA_supplemental}

\paragraph{Data generation.}
We generated multiple datasets from a categorial distribution with a softmax (multi-logit) using the technique described in Sec.~\ref{sec:data_generation}.  In particular, we randomly generated 16 datasets by taking the number of categories to be $K \in \set{3, 10}$, the number of covariates to be a multiple of the number of categories via $M =aK$ for $a \in \set{1,2}$, the number of samples to be a multiplier on the number of parameters via $N=bP$ for $b \in \set{10,20, 40, 80, 160}$ (where recall that the number of parameters is given by $P=K(M+1)$ due the presence of an intercept), and  $\sigma^2_{\text{high}} \in \set{0.1, 4.0}$ to control the predictability of the categorical observations.

\paragraph{Training.} For each dataset, we used 80\% of the data for model training and held out the remaining 20\% for evaluation.  We applied our IB-CAVI inference technique with the logit link (so $H$ was taken as the standard logistic cdf).  For each dataset, we ran IB-CAVI until the surrogate lower bound $\ELBO_\IB$ had a mean value (across samples and categories) that dropped by 0.1 or less on consecutive iterations. 

\paragraph{Predictive likelihoods.} Recall that we have partitioned each dataset into training data $\+y^\text{train} \in \set{1,...,K}^{N_{\text{train}}}$ and hold-out test data $\+y^{\text{test}} \in \set{1,...,K}^{N_{\text{test}}}$. After training the model on $\+y^\text{train}$,  we consider three different predictive likelihoods for test set observations $y^{\text{test}}_i$, $i \in 1,..., N_\text{test}$. In particular, we can compute $p_\CBC(y^{\text{test}}_i \cond \+y_{\text{train}})$,  $p_\CBM(y^{\text{test}}_i \cond \+y_{\text{train}})$, and $p_\BMA(y^{\text{test}}_i \cond \+y_{\text{train}})$.   The former two quantities are estimated by substituting IB-CAVI's variational posterior expectation into the relevant model's category probability formulae, Eqs.~\ref{eqn:CBM_category_probabilities} and \ref{eqn:CBC_category_probabilities}.  The latter quantity is computed from the former two quantities via Eq.~\eqref{eqn:bayesian_model_average} using the method of Sec.~\ref{sec:determining_a_good_target_for_an_IB_approximation}.   

\paragraph{Discrepancy from true category probabilities.} Since the data is simulated, we have access to the ``ground truth" predictive likelihood $p_\text{true}(y^{\text{test}}_i \cond \+y_{\text{train}})$ for each test set sample, obtained by substituting the true regression weights $\+B_{\text{true}}$ into the softmax likelihood.  We can therefore evaluate the performance of our three estimated predictive likelihoods by computing the discrepancy between each approximation and this ground truth:  
\[ d_i : = D_\text{KL} \big[p_\text{true}(y^{\text{test}}_i \cond \+y_{\text{train}}) \parallel p_\mathcal{M}(y^{\text{test}}_i \cond \+y_{\text{train}})\big] \]
where $\mathcal{M} \in \set{\CBC, \CBM, \BMA}$, $ D_\text{KL}$ is the Kullback-Leibler divergence, and $i=1,...,N_\text{test}$.   Our performance measure for each estimated predictive likelihood is then the mean discrepancy across the test set, i.e. $\frac{1}{N_\text{test}} \sum_{i=1}^{N_\text{test}} d_i$.  

\subsubsection{Results} \label{sec:results_BMA_supplemental}

\begin{figure*}[htp!]
\begin{tabular}{cc}
\centering
\includegraphics[width=.5\textwidth]{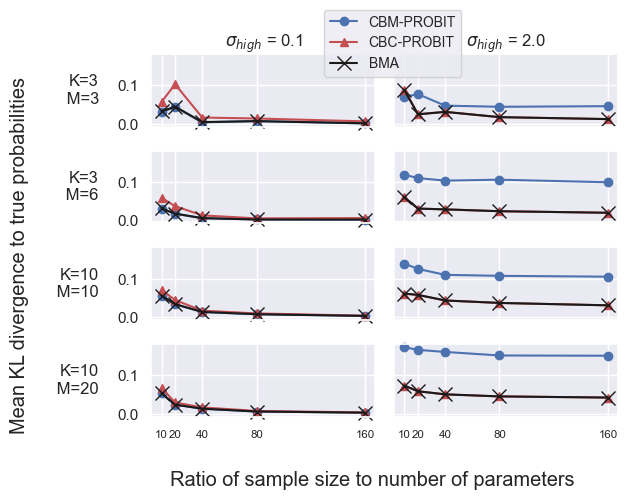} &
\includegraphics[width=.5\textwidth]{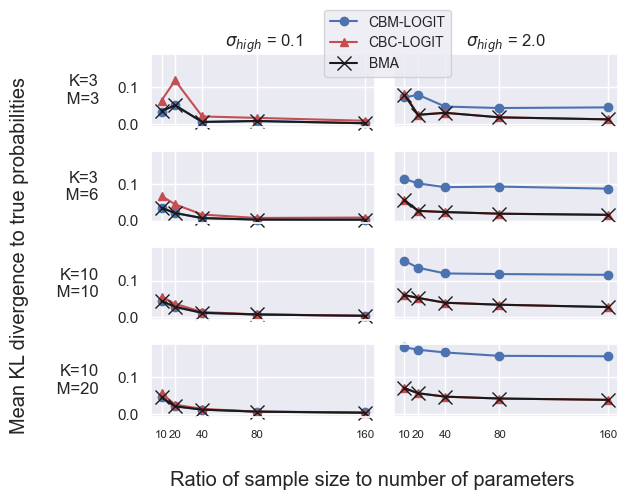} 
\end{tabular}
\caption{\textbf{Determining a good target for the \IB~approximation.}  Each point corresponds to a simulated dataset with some number of categories $K$ and covariates  $M$.   Plotted are the mean KL divergences on hold-out test data to the true categorical probabilities  for predictions of the \CBM and \CBC models (both estimated with IB-CAVI), as well as a Bayesian model averaging (BMA) of the two \IB~ approximation targets.  The left subplot use \CB~Probit models, whereas the right plot uses \CB~Logit models.  Within each subplot, the level of predictability of the categorical responses is \textit{weak} for datasets in the left column ($\sigma_\text{high} = 0.1$) and \textit{strong} for those in the right column ($\sigma_\text{high} = 2.0$). }
\label{fig:BMA_results_extended}
\end{figure*}.

Figure~\ref{fig:BMA_results_extended} provides an expanded version of Figure~\ref {fig:BMA_results_logit_ordered}.
   Some patterns of interest:

\begin{itemize}
\item As the number of categories and covariates and predictability ($K,M,\sigma^2_{\text{high}}$) are fixed, the error in the IB approximation decreases as the number of samples $N$ increases (as a multiple $b \in \set{10,20,40,80,120}$ on the number of parameters).  
\item As the predictability of the categorical response ($\sigma^2_{\text{high}}$) increases, the \CBC~model becomes better than \CBM~at serving as a target of the approximation. (To see this, compare the left column to the right column in Fig.~\ref{fig:BMA_results_logit_ordered}.) Since the predictability of the dataset may not be known in advance, this fact might seem to create a difficult model selection problem. Luckily, the Bayesian model   averaging (BMA) tracks the relative appropriateness of each model change by toggling the weight on the \CBC~model $w_\CBC$.  \red{TODO: Fix the sigma-high vs ssq-high thing.}
\item The relative advantage of the \CBC~model over the \CBM~model also seems to increase as the number of parameters $P=K(M+1)$ increases.  (To see this, compare the top rows to the bottom rows in Fig.~\ref{fig:BMA_results_extended}.) 
\end{itemize}

Table~\ref{tab:BMA_weights} provides more detailed information about the results shown in Fig.~\ref{fig:BMA_results_extended}.

\begin{table*}[htp!]
\caption{\textit{Additional results from applying the approximate Bayesian Model Averaging technique of Sec.~\ref{sec:determining_a_good_target_for_an_IB_approximation} to simulated datasets.}  This table provides more detailed information about the analysis depicted in Fig.~\ref{fig:BMA_results_logit_ordered}.  $w_\CBC$ gives the weight that the technique assigns to the \CBC~model.}
\label{tab:BMA_weights}
\centering 
\begin{tabular}{rrrr|r|rrr}
\toprule
 N &  K &  M &  $\sigma_\text{high}$ &  $w_\CBC$ & \multicolumn{3}{c}{Mean KL divergence to true probabilities from:} \\
    & & & & & \CBM & \CBC & \BMA \\
\midrule
  120 &  3 &  3 &     0.100 &       0.021 &                                0.016 &                                0.017 &                                0.015 \\
  120 &  3 &  3 &     2.000 &       0.861 &                                0.067 &                                0.046 &                                0.040 \\
  240 &  3 &  3 &     0.100 &       0.005 &                                0.031 &                                0.065 &                                0.031 \\
  240 &  3 &  3 &     2.000 &       0.999 &                                0.043 &                                0.027 &                                0.026 \\
  480 &  3 &  3 &     0.100 &       0.002 &                                0.007 &                                0.021 &                                0.007 \\
  480 &  3 &  3 &     2.000 &       1.000 &                                0.023 &                                0.018 &                                0.018 \\
  960 &  3 &  3 &     0.100 &       0.000 &                                0.002 &                                0.018 &                                0.002 \\
  960 &  3 &  3 &     2.000 &       1.000 &                                0.029 &                                0.020 &                                0.020 \\
 1920 &  3 &  3 &     0.100 &       0.000 &                                0.002 &                                0.008 &                                0.002 \\
 1920 &  3 &  3 &     2.000 &       1.000 &                                0.030 &                                0.018 &                                0.018 \\
  210 &  3 &  6 &     0.100 &       0.001 &                                0.030 &                                0.069 &                                0.030 \\
  210 &  3 &  6 &     2.000 &       1.000 &                                0.116 &                                0.093 &                                0.093 \\
  420 &  3 &  6 &     0.100 &       0.001 &                                0.016 &                                0.030 &                                0.016 \\
  420 &  3 &  6 &     2.000 &       1.000 &                                0.090 &                                0.030 &                                0.030 \\
  840 &  3 &  6 &     0.100 &       0.000 &                                0.009 &                                0.033 &                                0.009 \\
  840 &  3 &  6 &     2.000 &       1.000 &                                0.071 &                                0.023 &                                0.023 \\
 1680 &  3 &  6 &     0.100 &       0.000 &                                0.007 &                                0.025 &                                0.007 \\
 1680 &  3 &  6 &     2.000 &       1.000 &                                0.077 &                                0.018 &                                0.018 \\
 3360 &  3 &  6 &     0.100 &       0.000 &                                0.001 &                                0.017 &                                0.001 \\
 3360 &  3 &  6 &     2.000 &       1.000 &                                0.076 &                                0.014 &                                0.014 \\
 1100 & 10 & 10 &     0.100 &       0.002 &                                0.042 &                                0.054 &                                0.042 \\
 1100 & 10 & 10 &     2.000 &       1.000 &                                0.114 &                                0.057 &                                0.057 \\
 2200 & 10 & 10 &     0.100 &       0.008 &                                0.028 &                                0.034 &                                0.028 \\
 2200 & 10 & 10 &     2.000 &       1.000 &                                0.098 &                                0.052 &                                0.052 \\
 4400 & 10 & 10 &     0.100 &       0.023 &                                0.011 &                                0.013 &                                0.011 \\
 4400 & 10 & 10 &     2.000 &       1.000 &                                0.079 &                                0.017 &                                0.017 \\
 8800 & 10 & 10 &     0.100 &       0.925 &                                0.008 &                                0.008 &                                0.008 \\
 8800 & 10 & 10 &     2.000 &       1.000 &                                0.074 &                                0.017 &                                0.017 \\
17600 & 10 & 10 &     0.100 &       1.000 &                                0.004 &                                0.004 &                                0.004 \\
17600 & 10 & 10 &     2.000 &       1.000 &                                0.073 &                                0.017 &                                0.017 \\
 2100 & 10 & 20 &     0.100 &       0.005 &                                0.041 &                                0.051 &                                0.041 \\
 2100 & 10 & 20 &     2.000 &       1.000 &                                0.135 &                                0.061 &                                0.061 \\
 4200 & 10 & 20 &     0.100 &       0.000 &                                0.021 &                                0.026 &                                0.021 \\
 4200 & 10 & 20 &     2.000 &       1.000 &                                0.128 &                                0.053 &                                0.053 \\
 8400 & 10 & 20 &     0.100 &       0.006 &                                0.012 &                                0.015 &                                0.012 \\
 8400 & 10 & 20 &     2.000 &       1.000 &                                0.111 &                                0.028 &                                0.028 \\
16800 & 10 & 20 &     0.100 &       1.000 &                                0.006 &                                0.007 &                                0.007 \\
16800 & 10 & 20 &     2.000 &       1.000 &                                0.109 &                                0.023 &                                0.023 \\
33600 & 10 & 20 &     0.100 &       1.000 &                                0.004 &                                0.004 &                                0.004 \\
33600 & 10 & 20 &     2.000 &       1.000 &                                0.103 &                                0.022 &                                0.022 \\
\bottomrule
\end{tabular}
\label{tab:approximation_error_when_using_IB_CAVI_and_BMA}
\end{table*}

\subsection{Variational Bayes vs. Maximum Likelihood: Supplemental information} \label{sec:vb_vs_mle_supplemental}
\subsubsection{Methodology} \label{sec:data_generation_for_cavi_vs_mle}

We generate data using the method described in Sec.~\ref{sec:data_generation}.  For this experiment, we generate $D=10$ datasets per simulation context, which is a particular choice of $K=3, M=2K, N \in \set{1, 100} *P \text{ where $P=K(M+1)$ }, \sigma_{\text{high}} \in \set{0.01, 0.5, 1, 2, 5, 10, 20, 50, 100}, \sigma_{\text{low}}=0.01, \sigma_{\text{int}} =1.0$. The choice of $M=2K$ could be imagined as the number of covariates under a light (order 2) autoregressive structure.  For each dataset, we used 80\% of the data for model training and held out the remaining 20\% for evaluation.  

\subsubsection{Modeling Strategies} \label{sec:modeling_strategies}
 We compare a number of different modeling strategies:  
\begin{enumerate}
\item \textit{Data generating process}: We take the known regression coefficients $\+B_{\text{true}}$ and plug it into the softmax (multi-logit) categorical probability function.  
\item \textit{Softmax (via MLE)}: We estimate the MLE, $\+B_{\text{MLE}}$, for and softmax a.k.a. multi-logit model.  The optimization was computed using automatic differentiation in \texttt{jax} (and default convergence parameters).  We can interpret the results of the optimization as an (approximate) MLE due to the convexity of the multi-logit function. The solver used was BFGS, which is the only solver that \texttt{jax} supports\footnote{as of documentation revision 1182e7aa}.  We can make predictions on new samples by plugging in $\+B_{\text{MLE}}$ to the multi-logit categorical probability function. 
 \item \textit{\CB~(via IB-CAVI)}: We compute CAVI for  \CB~models (specifically, the \CBC-Probit and \CBM-Probit) with a $\N(\+0,\+I)$ prior using the variational technique with independent binary approximation  described in the main body of the text.   CAVI was run until the The variational posterior mean $\E_q[\+B]$ was used as a point estimate for $\+B$, and then substituted into the category probability formula for either the \CBC-Probit or \CBM-Probit.  
 \item \textit{Baserate frequency}: We use the raw frequencies of each category in the training set and use those as the predicted category probabilities for test set data, regardless of the value of the covariates, i.e.
\begin{align*}
p(y_i = k \cond \+x_i) = f_k 
\labelit \label{eqn:baserate_frequency}
\end{align*}
where $f_k$ is the frequency with which the $k$th category was observed in the training set. 
 \end{enumerate}
 
 The differences between the modeling strategies are summarized in Table \ref{tab:summary_of_modeling_strategies}.

\begin{table*}[htp!]
\caption{Summary of various modeling strategies for categorical data as used in the simulations experiment.}
\label{tab:summary_of_modeling_strategies}
\centering 
 \begin{tabular}{c|cc}
\textbf{Modeling Strategy} &  \textbf{Model} & \textbf{Inference for $\+B$} \\
\hline 
Data generating process & Softmax & $\+B$ is known \\
Softmax (via MLE) & Softmax & MLE on softmax  model \\
\CBC-Probit (via IB-CAVI) & \CBC & Variational posterior mean for the IB model, $\E_q[\+B]$ \\ 
\CBM-Probit  (via IB-CAVI) & \CBM & Variational posterior mean for the IB model, $\E_q[\+B]$ \\ 
Baserate frequency & Equation \eqref{eqn:baserate_frequency} &  N/A \\
 \end{tabular}
\end{table*}

\paragraph{Training.}  The MLE and IB-CAVI were both   initialized at  the zero matrix.  For each dataset, we ran IB-CAVI until the surrogate lower bound $\ELBO_\IB$ had a mean value (across samples and categories) that dropped by 0.1 or less on consecutive iterations. 

\red{TODO: We probably now need to move these modeling strategies up in the appendix, because they are probably relevant to the BMA experiment, which now comes first.}

\subsection{Holdout performance over time: Supplemental information} \label{sec:holdout_performance_over_time_supplemental}

The purpose of this experiment is to compare IB-CAVI against other methods for Bayesian inference with categorical GLMs. In particular, we compare the performance on holdout data as a function of training time.

\subsubsection{Datasets} \label{sec:datasets_for_perfromance_over_time}

\paragraph{Simulated datasets.} We construct simulated datasets using the method described in Sec.~\ref{sec:data_generation} for given specifications ($N$ samples, $K$ categories, $M$ covariates).  We set $\sigma^2_{\text{high}}=2.0$.

\paragraph{Real datasets.} We investigate the following real datasets:
\begin{enumerate}
\item 	The \textit{Detergent Purchase} dataset  \cite{imai2005bayesian}, which has 2,657 observations, 6 covariates, and 6 categorical responses.  Each record represents the purchase of a laundry detergent by a household in Sioux Falls, South Dakota.  The prediction goal is to identify which of the 6 laundry detergents was purchased given the prices of all 6 detergents. This data is available via the GPL-3 license at \url{https://github.com/kosukeimai/MNP/blob/master/data/detergent.txt.gz}. For the original paper using this dataset, see \cite{chintagunta1998empirical}.

\item The \textit{Anuran Frog Calls} dataset \citep{anuranFrogDataset2017}, which has 7,195 observations, 22 covariates, and 10 categorical responses. Each record represents extracted audio features (mel-frequency cepstrum coefficients (MFCCs) from a recording of a frog making some natural noises. The prediction goal is to identify which of 10 species the frog belongs to. This data is available from the UC-Irvine Machine Learning Repository via an open-access CC-BY license. For an in-depth paper using this data, see \citet{anuranFrogPaper2016}.

\item The \textit{Glass Identification} dataset \citep{glassDataset1987}, which has 214 samples, 9 covariates, and 6 response categories that were observed. Each record represents observable properties of a physical sample of glass, with the prediction goal being to identify which of six types of glass the sample represents. A seventh possible category is noted in the data description but never observed. This data is available from the UC-Irvine Machine Learning Repository under an open-access CC-BY license. For the original paper using this dataset, see \citet{glassPaper1987}.

\item A \textit{Single-User Process Start} dataset, which has 17,724 observations, 1,553 covariates, and 1,553 categorical responses. The dataset is constructed from the Comprehensive, Multi-Source Cybersecurity Events Dataset \citep{kent2015comprehensive} using the methods of Sec.~\ref{sec:intrusion_detection_experiment_additional_info}. Each record contain one process start from one user account \texttt{U293@DOM1} along with the identity and timing of the $W=5$ immediately preceding process starts. The prediction goal is to identify the next process start.  The data is open-access with all copyrights waived, and the preprocessing used is available at \url{\codeURL}.
\end{enumerate}

For all real datasets, we z-transformed all covariates, as the range of some variables is very small (e.g. consider the \texttt{RI} variable in the \textit{glass identification} dataset, which only varies from 1.51 to 1.52).  This lets us use independent $\N(0,1)$ priors on the regression weights for each covariate-category combination. No missing data occurred in any of the datasets.

\subsubsection{Modeling Strategies} \label{sec:modeling_strategies_performance_over_time}
Here we describe the various modeling strategies we used for Bayesian categorical regression modeling of the provided datasets. For motivation on which methods to include vs. exclude in the experiment, see the discussion of Sec.~\ref{sec:supplemental_info_on_alternative_inference_methods}. 
  
\begin{enumerate}
 \item \textit{\CB-Probit and \CB-Logit (via IB-CAVI)}: We compute CAVI for \CB-Probit and \CB-Logit models  with a $\N(\+0,\+I)$ prior using the variational technique with independent binary approximation  described in the main body of the text.  
\item \textit{Softmax regression (via automatic differentiation variational inference (ADVI)).}  The gradient updates for softmax regression (whose parameters have support of unconstrained reals) are described in Sec.~\ref{sec:ADVI}.  We implement these updates in \texttt{jax}, and optimize using  Algorithm 1 of \cite{kucukelbir2017automatic}. We follow the recommendations of that paper to guide the optimization details: one Monte Carlo sample per update, and adaptive step-size sequences with  varying learning rates but all other hyper-parameters kept at their recommended defaults.

 \item \textit{Softmax regression (via the No U-Turn Sampler (NUTS))}  We sample from the posterior of softmax regression using the No U-Turn Sampler (NUTS) \cite{hoffman2014no} as implemented in the Python package \texttt{numpyro}.  \red{MCH: Should I say anything more about the defaults?}
 \item \textit{Softmax regression (via Gibbs after \pga)} Here we model the data with softmax regression (more specifically the identified version of it  which is obtained by setting $\+\beta_K \equiv 0$; this is often called multi-logit regession), but using the Gibbs sampler which is available after \pga.  The complete conditionals for the Gibbs sampler are given in Sec.~\ref{sec:gibbs_sampling_for_multi_logit_regression_with_pga}. 
 \end{enumerate}

\subsubsection{General experimental methodology}

For training, we used 80\% of the data for model training and held out the remaining 20\% for evaluation for all datasets except glass identification.  Due to the small size of the glass identification dataset, we instead used a 90\%/10\% split.  All methods were initialized to have their matrix of regression weights be $\+B = \+0$.   Each inference method was run for a preset number of iterations (ADVI, IB-CAVI) or samples (NUTS, Gibbs) in an attempt to make the running time for each method similar.  The performance of NUTS is dependent upon the number of tuning samples, which was set to be 25-33\% as large as the number of samples retained afterwards. 

For prediction on holdout test data, the posterior mean (for NUTS and Gibbs) or variational posterior mean (for ADVI and IB-CAVI) was used as a point estimate $\widehat{\+B}$ for $\+B$. This value $\widehat{\+B}$ was then substituted into the appropriate category probability formula -- softmax, multi-logit (i.e. identified softmax), \CB-Probit, or \CB-Logit.  For a given \CB~link function (probit or logit), the \CB~variant (\CBM~or \CBC) was chosen that yielded the largest training likelihood.  This strategy provides a cheap heuristic approximation to BMA, as most datasets have sufficiently many observations that the BMA weights tend to be very close to 0.0 or 1.0. 

For performance metrics, we used mean holdout log-likelihood and mean predictive accuracy.  The mean holdout log-likelihood was  computed in the standard way (Eq.~\eqref{eqn:mean_holdout_log_likelihood}).  For the accuracy metric, the category with the largest probability was considered to be the predicted category.  If a test set observation had $C$ categories predicted with the same probability, then the model was given credit for $1/C$ rather than $1$ correct response.   For simulated data, we also computed these performance metrics under random guessing and when using the true model (i.e. softmax regression, using $\+B_{\text{true}}$.)

\subsubsection{Results}

The primary results were given in Sec.~\ref{sec:holdout_perf_over_time}.  Supplemental results are provided in Fig.~\ref{fig:holdout_perf_over_time_supplemental}.

\begin{figure*}[h]
\centering
\begin{tabular}{c c c l }    
	Glass 
    & Frog Calls 
    & Simulated 
\\
	{\scriptsize $K=6, M=9, N =214$}
    & {\scriptsize $K=10, M=22, N=7195$ } 

    &  {\scriptsize $K{=}100, M{=}200, N{=}20000$}  
\\ 
    \includegraphics[height=3cm]{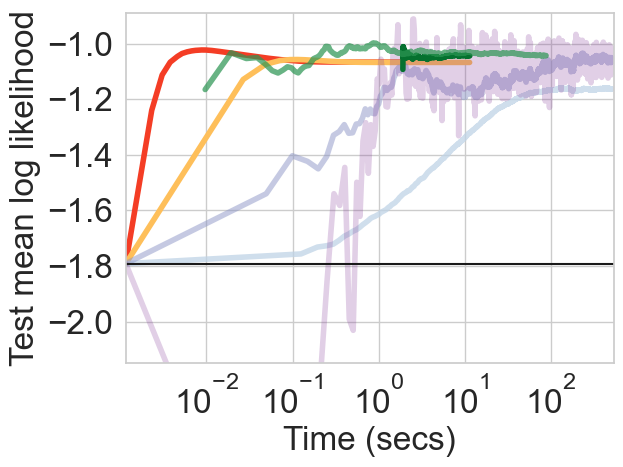} 
&
    \includegraphics[height=3cm]{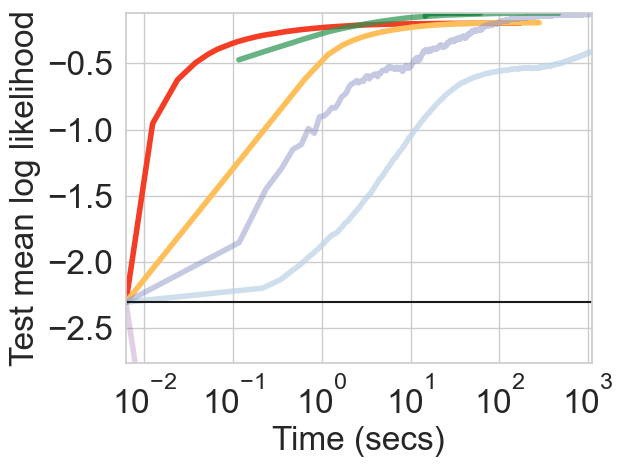}
& \includegraphics[height=3cm]{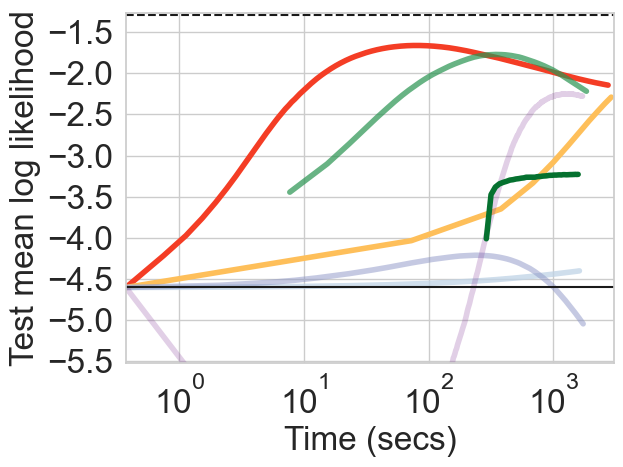}
& \includegraphics[height=2.5cm]{images/perf_over_time/legend_only_show_CB_logit=True.png}
\\
    \includegraphics[height=3cm]{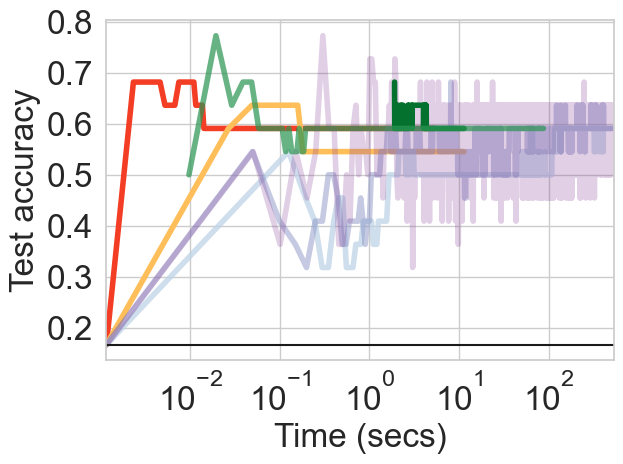}
   & \includegraphics[height=3cm]{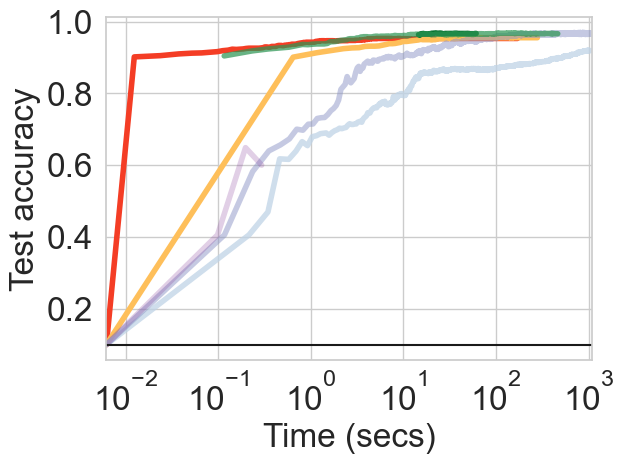}
&
    \includegraphics[height=3cm]{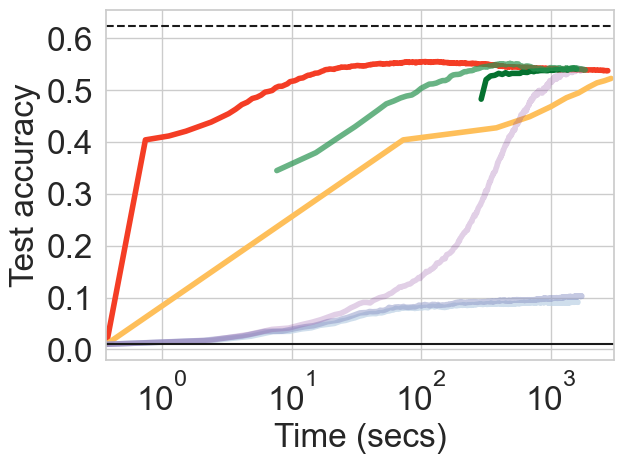}
\end{tabular}
    \caption{Supplemental comparisons of holdout log likelihood (top) and accuracy (bottom) over training time on real  and simulated datasets with $K$ categories, $M$ covariates, and $N$ instances.  For ADVI, we try the learning rates $\mathcal{L}:=(0.01, 0.1, 1.0, 10, 100)$ recommended by \cite{kucukelbir2017automatic}, adjusted to $10^{-1} \mathcal{L}$ in the largest dataset to reduce divergence. If a line is absent for ADVI, the method diverged.  Note that for IB-CAVI, the parallelism over $K$ could be exploited to yet further reduce training time.}
\label{fig:holdout_perf_over_time_supplemental}
\end{figure*}

\subsection{The impact of the \IB-approximation on posterior over category probabilities} 


In this section, we directly investigate the quality of the posterior over category probabilities that is learned by \IB-CAVI. By \textit{posterior over category probabilities}, we refer to the categorical likelihoods $p(y = k | \+B)$ obtained by drawing the regression weights from the  approximate posterior density over weights $q(\+B \cond \+y_{1:N})$, where $\+y_{1:N}$ is the training data.   While a direct analysis of $q(\+B \cond \+y_{1:N})$ is possible, this is an intermediate quantity less relevant to applications (see Sec.~\ref{sec:problem_formulation}) and may be confounded by identifiability issues.

 Thus, we compare \IB-CAVI's posterior over category probabilities against that learned by other methods that do not make an \IB-approximation.  We would like to obtain a concrete visualization of  how the \IB-approximation impacts the bias and variance of this posterior over category probabilities.  Of particular interest is the comparison to the NUTS sampler, which can be taken as the gold standard.

\subsubsection{Methodology}

\paragraph{Dataset.} We construct a simulated dataset using the method described in Sec.~\ref{sec:data_generation} with $N$=1000 samples, $K$=4 categories, and $M$=8 covariates.  We set $\sigma^2_{\text{high}}=4.0$.

\paragraph{Methods.} We train a \CB-Probit model with \IB-CAVI (Algorithm~\ref{alg:ib_cavi_for_cb_probit}) until the drop in the mean ELBO (with the mean taken across the number of samples $N$ and categories $K$) was less than 0.01  across consecutive iterations.  Bayesian model averaging (BMA; Sec.~\ref{sec:determining_a_good_target_for_an_IB_approximation}) reveals that the weight on the \CBC~model, $\pi_\CBC$, was very close to 1.0; thus, the predictions of the \CB-Probit model with BMA is virtually identical to the predictions of the \CBC-Probit model.  For this reason, our baseline black-box inference methods use the \CBC-Probit (rather than \CBM-Probit, or some mixture).  In this case, the baseline methods used were Automatic Differentiation Variational Inference (ADVI) or the No-U-Turn Sampler (NUTS) (see Sec.~\ref{sec:modeling_strategies_performance_over_time}).

\subsubsection{Results}

Fig.~\ref{fig:violin_plots_posterior_over_category_probabilities} gives the posterior over category probabilities for the first 9 training set observations as approximated by IB-CAVI, ADVI, and NUTS. 

\subsubsection{Discussion}

\IB-CAVI delivers posteriors over category probabilities that are reasonably good approximations to those obtained by ADVI and NUTS.  However, the procedure does appear to reduce variance and introduce some bias.  For applications where fidelity to the true posterior is critical, one could use \IB-CAVI for warm-starting.  That is, one could use \IB-CAVI's quickly learned approximate posterior to initialize a more expensive procedure that delivers greater fidelity.  For example, one might use \IB-CAVI to initialize NUTS, which is computationally expensive but asymptotically exact.   

\begin{figure*}
\centering
\begin{tabular}{c c c}
& 
\includegraphics[height=2.5cm]{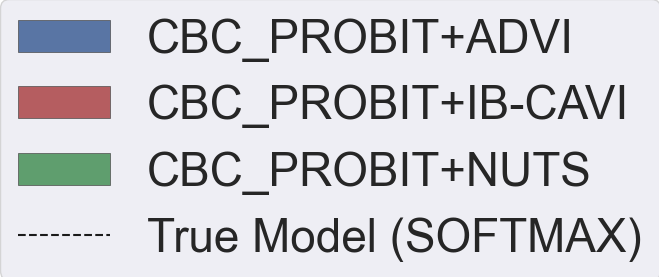}
& \\
  \includegraphics[height=4.5cm]{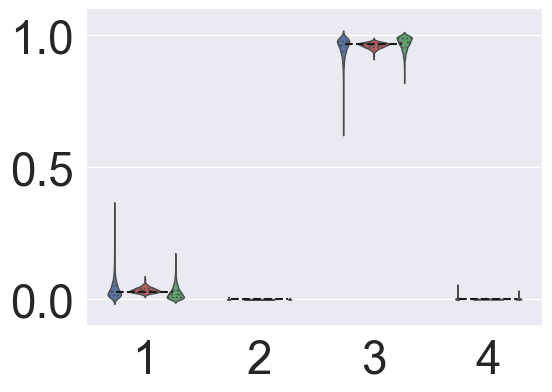} 
&
    \includegraphics[height=4.5cm]{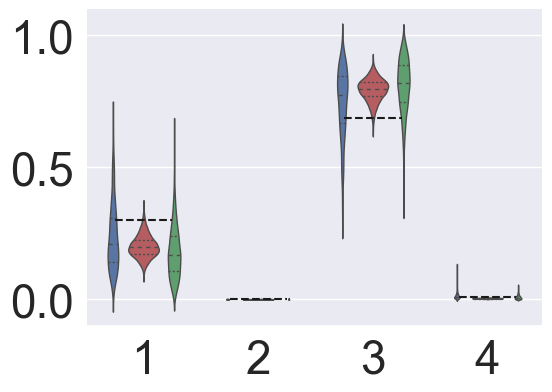}
& \includegraphics[height=4.5cm]{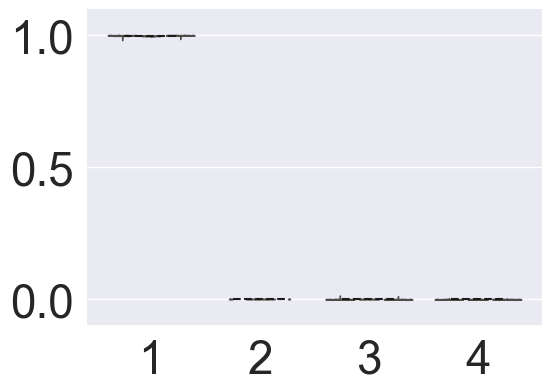}
\\
    \includegraphics[height=4.5cm]{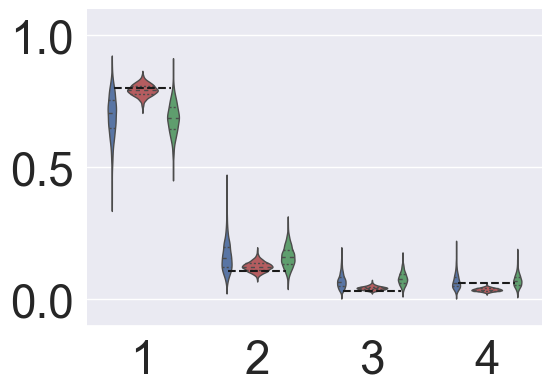}
   & \includegraphics[height=4.5cm]{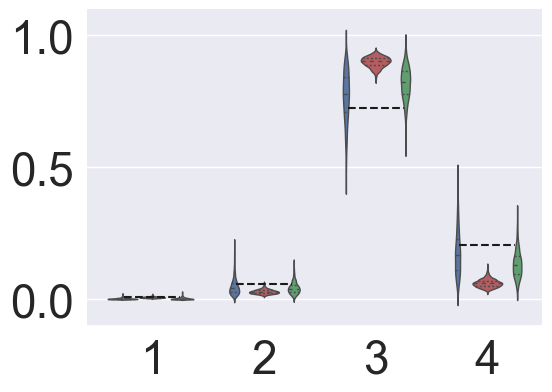}
& \includegraphics[height=4.5cm]{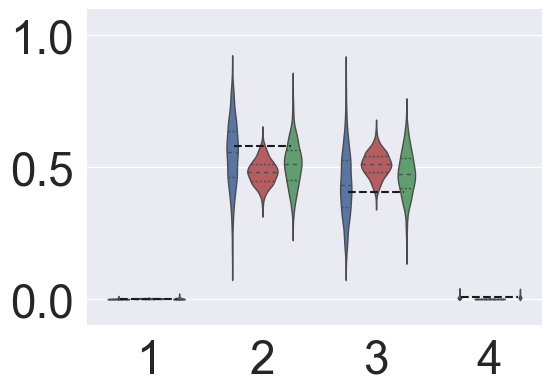} \\
\includegraphics[height=4.5cm]{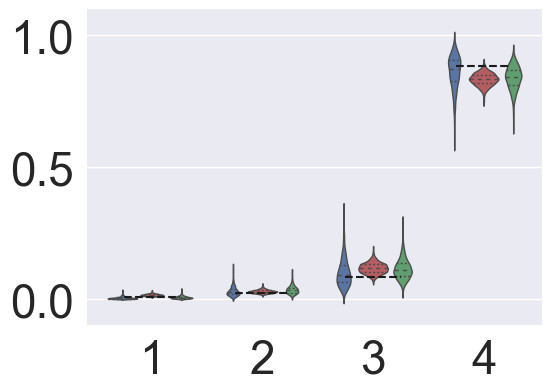}
   & \includegraphics[height=4.5cm]{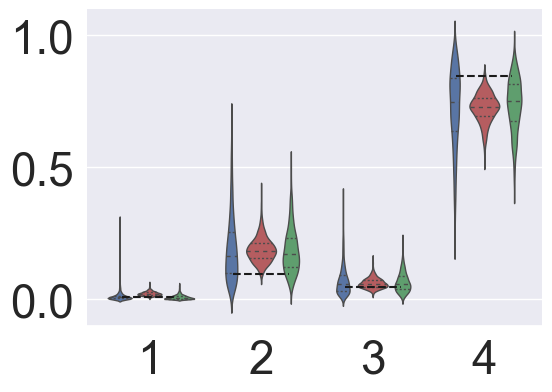}
& \includegraphics[height=4.5cm]{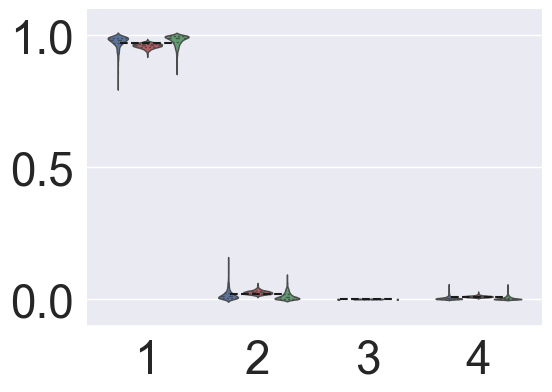} 
\end{tabular}
    \caption{Violin plots showing the approximate posteriors over category probabilities (y-axis) across four possible categories (x-axis).  The approximate posteriors over category probabilities are given by three approximate Bayesian inference methods applied to simulated softmax regression data. Each subplot represents the  approximate posterior predictive distribution over categories given the covariates for a single training set observation.   
}
\label{fig:violin_plots_posterior_over_category_probabilities}
\end{figure*}

\subsection{Intrusion detection experiment: Supplemental information} \label{sec:intrusion_detection_experiment_additional_info}

\subsubsection{Data}  \label{sec:cyber_data}

The raw multi-source cyber-security behavioral event data \cite{kent2015comprehensive} contains 58 consecutive days of computer usage behavior from within Los Alamos National Laboratory's corporate, internal computer network.   Behavior is represented in terms of events from five modalities (process starts,  network activity,  etc.) For earlier work on intrusion detection with this dataset, see \cite{turcotte2016poisson}.

 We restrict our analysis to the raw process data, which represents process start and stop events collected from individual Windows-based desktop computers and servers.  In addition, we  restrict our attention to process starts from human users on active directory domain accounts. This restriction requires three pre-processing steps:

\begin{enumerate}
\item Use only the process starts (discard the process ends).
\item Restrict to human users. Users with names beginning with a 'U' are human accounts while those beginning with a 'C' are computer accounts (managed by computers not actual people).  Correspondingly, users which start with 'C'  seemed to have less variation across users in process starts. 
\item Restrict to active directory domain accounts. Domains  starting with `C` are local accounts typically used on individual workstations.  Domains starting with `DOM` can  be used to authenticate on local machines, but they are also the standard way of authenticating for network resources (email, databases, servers, etc.).  There is a lot more user overlap for `DOM` resources, making them less predictable.\footnote{This information was provided by personal communication with Aaron Scott Pope, the current contact from Los Alamos National Laboratories provided by \url{https://csr.lanl.gov/data/cyber1/}.} 	
\end{enumerate}

\subsubsection{Featurization} \label{sec:featurization_for_intrusion_detection}
    
 During an exploratory data analysis, we notice:
 \begin{enumerate}
 \item Regularities in process start subsequences can have minor permutations.  For instance, see user U1788, who deterministically cycles through P111-P296-P298-P299, until breaking out of this pattern for the last 10 or so process starts.  The same processes are used, but in a different pattern.  
 \item Multiple processes can be launched simultaneously (up to the single second resolution with which time is reported), and the order in which simulataneous processes are listed is not invariant. (Note this provides a partial explanation for item (1).)
 \end{enumerate}

 To accommodate these features of the data, we do not use a strict autoregressive featurization, but a softer version which should be more tolerant of noise.  In particular, we choose a lookback window of $W$ processes, and then featurize each of these $W$ processes with $\exp(-\Delta t/\tau)$ seconds, where $\tau$ is a temperature parameter and $\Delta t$ refers to how long ago (in seconds) the process was launched.   

    The window size $W$ could perhaps be justified by plotting the distribution on the number of simultaneous
    process launches, and saying that the window size is the whatever-th percentile of that distribution.
    
\subsubsection{Methodology for experiment}
Here we describe the methodology used for the experiment discussed in Sec~\ref{sec:intrusion_detection_experiment}. We selected $U$=32 users from the database who had moderately many process starts. The number of processes started per user, $N_u$, over the course of the 58 days of data collection ranged from 17,678 to 19,261.\footnote{The target number $N_u$ serving as an  inclusion criterion was chosen out of convenience: the Python package in its currently implementation can handle $N_u \approx 20,000$ without a memory error, but cannot handle the largest value of $N_u$ in the dataset, due to memory constraints. No attempt was made to model the largest $N_u$, because the experiment as is seems sufficient to prove the point.  Further scalability could be obtained by improving the implementation (in terms of handling of sparsity and/or further exploiting the fact that the algorithm is embarassingly parallel across categories), or by incorporating memoization \cite{hughes2013memoized} or  stochastic variational inference \cite{hoffman2013stochastic} strategies within the IB-CAVI framework.}

We learn each user's process start behavior by training a separate model for each user. We use the featurization strategy of Section \ref{sec:featurization_for_intrusion_detection}, somewhat arbitrarily choosing the window size to be $W=5$ and the temperature to be $\tau=60$ seconds.   We take the number of categories to be $K=1,553$, the number of unique processes in the entire dataset. 

We take the first 80\% of the process start events to be training data, and the remainder to be hold-out test data.  We use IB-CAVI to approximately learn the \CBC-Probit model. We ran inference for 100 iterations.  Each iteration required approximately 5 to 20 seconds of computation time.

\subsection{Glass identification: Supplemental analysis}

Here we provide further analysis of the glass identification dataset that was also analyzed in the holdout performance over time experiment (Sec.~\ref{sec:holdout_performance_over_time_supplemental}).  Here, following \cite{johndrow2013diagonal}, we perform 10-fold cross validation, randomly splitting the dataset 10 times into a training set and test set, where each split put 90\% of the original dataset into the training set.  Thus, each data split had $N_\text{train} = 192$ training samples, and $N_\text{test} = 22$ test samples.  

We z-transformed all variables, as the range of some variables is very small (e.g. consider the \texttt{RI} variable, which only varies from 1.51 to 1.52).  This lets us use independent $\N(0,1)$ priors on the regression weights for each covariate-category combination.

\subsubsection{Methodology}
We applied two different Bayesian inference methods : MCMC sampling and variational inference.  For MCMC sampling, we applied the implementation of the No U-Turn Sampler (NUTS) \cite{hoffman2014no} given in the \texttt{numpyro} library. We  obtained 10,000 total samples (3,000 burn-in samples).  \red{TODO: Add better description and cite for numpyro, as well as NUTS.}  For variational inference, we applied IB-CAVI, and concluded convergence when the drop in the mean ELBO (with the mean taken across the number of samples $N$ and categories $K$) was less than 0.005 across consecutive iterations.   For both inference methods, we initialized the regression weights $\+B$ to the zero matrix. 

\subsubsection{Results}

Table~\ref{tab:glass_identification_results} shows the results. We find that IB-CAVI gives results that are close to those obtained by NUTS, but between 44 and 1,110 times faster.  We also note from the NUTS results that the \CB~models perform competitively with the softmax model, a much more familiar categorical GLM.  

\begin{table*}[htp!]
\caption{\textit{Glass identification results}. The (geometric) mean holdout likelihood is given by $\exp \big( \frac{1}{F N_{\text{test}}}\sum_{f=1}^F \sum_{n=1}^{N_{\text{test}}} \log p(y_n^{\text{test}} \cond \+B^*) \big)$, where $F$ is the number of cross-validation folds and $\+B^*$ is the posterior expectation from IB-CAVI or NUTS. It represents the typical probability score that the fitted model assigns to categorical outcomes in the test set.  Computation time is measured in seconds. Note that accuracy will always be identical for \CBC~and \CBM~models with the same \IB~base model when fit with \IB-CAVI, as guaranteed by Prop.~\ref{prop:IB_plus_OHT_and_IB_plus_NSP_make_the_same_predictions}.}
\label{tab:glass_identification_results}
\centering
\begin{tabular}{lr|rr|rr|rr|rr}
\toprule
Model & Softmax & \multicolumn{2}{c|}{\CBC-Logit} & \multicolumn{2}{c|}{\CBM-Logit} & \multicolumn{2}{c|}{\CBC-Probit} & \multicolumn{2}{c}{\CBM-Probit} \\
Inference &     NUTS &       NUTS & IB-CAVI &       NUTS & IB-CAVI &        NUTS & IB-CAVI &        NUTS & IB-CAVI \\
\midrule
Mean likelihood  &    0.38 &      0.38 &    0.36 &      0.38 &    0.36 &       0.34 &    0.35 &       0.41 &    0.37 \\
Accuracy         &    0.64 &      0.65 &    0.64 &      0.64 &    0.64 &       0.65 &    0.65 &       0.64 &    0.65 \\
Computation time &   20.17 &     26.07 &    0.30 &     15.41 &    0.30 &     333.04 &    0.35 &      59.98 &    0.35 \\
\bottomrule
\end{tabular}
\end{table*}

\end{document}

%% file: related_work.tex
\subsection{Related work on IB models.}

Our work is inspired by the \emph{diagonal orthant} (\DO) models proposed by~\citet{johndrow2013diagonal}. What we call the \CBC~likelihood is equivalent to the marginal likelihood of the DO model (integrating away auxiliary variables). \citeauthor{johndrow2013diagonal} further proposed using independent binary regressions (as we do) to perform scalable Bayesian computation for categorical data.
\citeauthor{johndrow2013diagonal} argued for IB approximation based on a claimed identification equivalence of point estimated weights between IB and DO models.

In a recent non-archival workshop paper \citep{wojnowicz2021easy}, we clarified that IB should be viewed as a separate, surrogate model (see also Sec.~\ref{sec:impossibility_of_exact_inference_on_a_CB_model_via_inference_on_an_IB_model}).
This paper extends that early line of work, offering a more coherent view of surrogate bounds, expanding to include many possible cdfs (not just probit) for IB approximations, and introducing our BMA approach to effective predictions.  We also simplify inference, as neither the auxiliary variables in \citeauthor{johndrow2013diagonal}'s DO model (nor the auxiliary variables in \citep{wojnowicz2021easy}'s SDO model) are needed to relate \IB~models to categorical models.

In summary, building on the IB approximation first suggested by \citeauthor{johndrow2013diagonal}, we contribute the following advances:
(1) We clarify that doing inference on a relevant categorical model via an \IB~model requires an \textit{approximation}. 
(2) We \textit{justify} this approximation via surrogate likelihood bounds.
(3) We expand the class of categorical models suitable for \IB~approximation, showing that \textit{both} \CBC~and \CBM~models should be included to obtain high-quality predictions (see Sec.~\ref{sec:experiment_on_averaging_approximation_targets_for_enhanced_performance}). 
(4) We focus on optimization approaches to posterior estimation, which may be more scalable than \citeauthor{johndrow2013diagonal}'s Gibbs sampling.  

\textbf{Similarity to one-vs-rest classification.}
At a high-level, our IB approximation is similar to a common generic heuristic for building multi-class classifiers known as a \emph{one-versus-rest} ensemble (also called \emph{one-versus-all}).
One-vs-rest schemes fit $K$ separate binary classifiers to distinguish each class from all others, and then make a one-of-$K$ prediction by taking the class corresponding to the classifier with largest predicted score or probability.
There is empirical evidence that one-vs-rest schemes can deliver accuracies on par with one-of-K classifiers for non-linear kernel methods~\citep{rifkinDefenseOneVsAllClassification2004}.
Due to simplicity and computational speed, widely-used software packages support this scheme~\citep{pedregosaScikitlearnMachineLearning2011} and efforts to classify 10,000 image classes have called one-vs-rest ``the only affordable option'' \citep{dengWhatDoesClassifying2010}.
However, to our knowledge there has not yet been statistical justification for one-vs-rest schemes in terms of a true multi-class likelihood, leading to concerns about coherency~\citep{murphySec14SVMs2012}.
Our likelihood bound arguments justify \emph{probabilistic} one-vs-rest approximations of CB models.

%% file: inputs/fig_combined_likelihood_and_accuracy.tex
\begin{figure}[!ht]
\centering
\includegraphics[width=.45\textwidth]{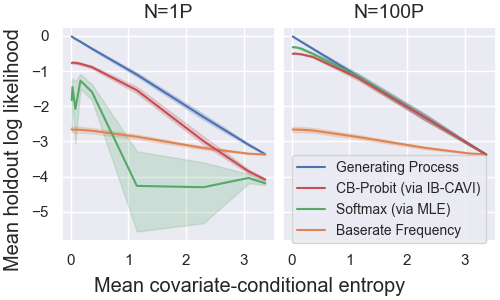}\caption{\textbf{Performance of IB-CAVI vs. softmax MLE given few vs. many observations-per-parameter.}  Plotted is the mean holdout log likelihood as a function of mean covariate-conditional category entropy \eqref{eqn:mean_covariate_conditional_category_entropy}, which quantifies predictability. Data was simulated from a softmax regression ($K$=30, $M$=60), comparing regimes where the number of observations is modest ($N=P$, \emph{left}) vs.  abundant ($N=100P$, \emph{right}) relative to the number of parameters ($P=K(M+1)$).  The \CB-Probit was estimated via IB-CAVI, with Bayesian model averaging of the \CBC-Probit and \CBM-Probit.   Error bars are 95\% confidence intervals for the expectation (over $D=10$ replicated datasets) of the mean test-set log likelihood.  The confidence intervals were determined by bootstrapping.  \red{TODO: Be careful about M vs M+1.} \red{DISCUSS: Should I reorder legend?} \red{TODO: Consider mentioning how unlikely it is to be in the range where entropy < .5} }
\label{fig:performance_simulations_combined_likelihood_and_accuracy}
\end{figure}

%% file: acks.tex
This research was sponsored by the U.S. Army DEVCOM Soldier Center, and was accomplished under Cooperative Agreement Number W911QY-19-2-0003. The views and conclusions contained in this document are those of the authors and should not be interpreted as representing the official policies, either expressed or implied, of the U.S. Army DEVCOM Soldier Center, or the U.S. Government. The U. S. Government is authorized to reproduce and distribute reprints for Government purposes notwithstanding any copyright notation hereon.

We also acknowledge support from the U.S. National Science Foundation under award HDR-1934553 for the Tufts T-TRIPODS Institute. 
SA is supported by NSF CCF 1553075, NSF DRL 1931978, NSF EEC 1937057, and AFOSR FA9550-18-1-0465.
ELM is supported by NSF grants 1934553, 1935555, 1931978, and 1937057.
MCH is supported by NSF IIS-1908617.